\documentclass{article}

\usepackage{microtype}
\usepackage{graphicx}
\usepackage{subcaption}
\usepackage{booktabs} 

\usepackage{hyperref}

\usepackage[preprint]{icml2026}

\usepackage{amsmath}
\usepackage{amssymb}
\usepackage{mathtools}
\usepackage{amsthm}
\usepackage{threeparttable}

\usepackage[capitalize,noabbrev]{cleveref}

\theoremstyle{plain}
\newtheorem{theorem}{Theorem}[section]

\newtheorem{lemma}[theorem]{Lemma}

\theoremstyle{definition}
\newtheorem{definition}[theorem]{Definition}

\theoremstyle{remark}
\newtheorem{remark}[theorem]{Remark}

\usepackage[textsize=tiny]{todonotes}

\usepackage{addition}

\icmltitlerunning{Unlearning Offline Stochastic Multi-Armed Bandits}

\begin{document}

\twocolumn[
  \icmltitle{Unlearning Offline Stochastic Multi-Armed Bandits}



  \icmlsetsymbol{equal}{*}

  \begin{icmlauthorlist}
    \icmlauthor{Zichun Ye}{equal,SJTU}
    \icmlauthor{Runqi Wang}{equal,SJTU}
    \icmlauthor{Xuchuang Wang}{UMass}
    \icmlauthor{Xutong Liu}{UW}
    \icmlauthor{Shuai Li}{SJTU}
    \icmlauthor{Mohammad Hajiesmaili}{UMass}
  \end{icmlauthorlist}

  \icmlaffiliation{SJTU}{Shanghai Jiao Tong University, Shanghai, China}
  \icmlaffiliation{UW}{University of Washington, Tacoma, United States}
  \icmlaffiliation{UMass}{University of Massachusetts Amherst, Massachusetts, United States}

  \icmlcorrespondingauthor{Xutong Liu}{xutongl@uw.edu}
  \icmlcorrespondingauthor{Shuai Li}{shuaili8@sjtu.edu.cn}

  \icmlkeywords{Machine Learning, ICML}

  \vskip 0.3in
]



\printAffiliationsAndNotice{\icmlEqualContribution}

\begin{abstract}
Machine unlearning aims to unlearn data points from a learned model, offering a principled way to process data-deletion requests and mitigate privacy risks without full retraining. Prior work has mainly studied unsupervised / supervised machine unlearning, leaving unlearning for sequential decision-making systems far less understood. We initiate the first study of a foundational sequential decision-making problem: offline stochastic multi-armed bandits (MAB). We formalize the privacy constraint for offline MAB and measure utility by the post-unlearning decision quality. We conduct a systematic study of both single- and multi-source unlearning scenarios under two data-generation models, the fixed-sample model and the distribution model. For these settings, our algorithmic design is built on two canonical base algorithms: Gaussian mechanism and rollback, and we propose adaptive algorithms that switch between them according to the data regime and privacy constraint. We further introduce a mixing procedure that elucidates the rationale behind these baselines. We provide performance guarantees across the above settings and establish lower bounds under both dataset models. Experiments validate the predicted tradeoffs and demonstrate the effectiveness of the proposed methods.

\end{abstract}

\section{Introduction}
Modern machine learning systems are increasingly trained on third-party and user-generated data, from images to texts used in recommendation or advertising \cite{CBCJ19,TPWB23}. When a user requests that certain records be deleted (unlearned) from the offline dataset, it is no longer enough to merely remove the raw data: trained models may still retain information about the deleted records, and this information can be extracted through inference attacks \cite{SSSS17,CLEK19}. Motivated by the demand for unlearning, a growing line of work \cite{GGVZ19,GGHV20,GJNR21,MPSR22,ZZZC22,YJCC25} studied machine unlearning—updating a trained model so that the influence of a specified subset of training data is removed, ideally without retraining from scratch. 

To formalize when data removal is sufficient, \citet{GGVZ19} introduced a probabilistic notion of unlearning, inspired by differential privacy \cite{DR14}. Building on this foundation, a line of work studied supervised machine unlearning \cite{GGHV20,ISCZ21} and unsupervised machine unlearning  \cite{NRSM21,GJNR21,UMRR21}. Most of these results focused on the empirical risk minimization (ERM) setting, aiming to minimize the training loss on the retained dataset after deleting the requested samples, and several approaches are memory intensive \cite{BCCC21,GJNR21}. To address these limitations, \citet{SAKS21} introduced an unlearning definition that incorporated both test-performance guarantees and explicit storage constraints. Building on this framework, \cite{SW22} studied online unlearning request and \cite{LLQR23} investigated unlearning for minimax models.

While substantial progress has been made in supervised / unsupervised machine unlearning, the landscape is far less understood in sequential decision-making paradigm such as multi-armed bandits (MAB) or reinforcement learning (RL). Recent work \cite{YZZ25} has begun exploring reinforcement unlearning and highlights that RL agents may memorize sensitive properties of environments. However, existing formulations primarily target environment-level unlearning and do not address the data-sample deletion guarantees central to machine unlearning. Moreover, they do not quantitatively formalize a privacy-style constraint, and therefore provide no sharp characterization of the resulting privacy-utility tradeoff. As a result, unlearning for decision-making still remains largely unexplored.

\begin{table*}[t]
\centering
\begin{threeparttable}
\caption{Theoretical results for the unlearning problem under various settings.}
\label{tab:theory_results}
\renewcommand{\arraystretch}{1.5}
\setlength{\tabcolsep}{6pt}

\begin{tabularx}{\textwidth}{Xcc}
\toprule
\textbf{Model\textsuperscript{$\star$}} &
\textbf{Unlearning Type\textsuperscript{$\circ$}} &
\textbf{Results\textsuperscript{$\dagger$}} \\
\midrule
\multirow{2}{3.2cm}{\\ Fixed-sample \\ model ($\+M_f$) }
& Single
& \makecell[l]{Upper bound:
$
\begin{cases}
    O\tp{\max\set{\sqrt{\frac{\ln \tp{N m} }{N(a_0)} } + \frac{k \gamma}{N(a_0)}, \sqrt{\frac{\ln \tp{N m} }{N_*}} }}, \gamma < \gamma_0 \\ 
    O\tp{\max\set{\sqrt{\frac{\ln \tp{N m} }{N(a_0) - k}}, \sqrt{\frac{\ln \tp{N m} }{N_*}} }}, \gamma \ge \gamma_0 
\end{cases} 
$ \\
Lower bound: \,\; 
$\Omega\tp{e^{-\eps}\sqrt{\frac{1}{N(a_0)-k}}}$ } \\
\cmidrule(lr){2-3}

& Multi
& \makecell[l]{Upper bound:
$ 
\begin{cases}
O\tp{\max\set{\sqrt{\frac{\ln \tp{N m} }{N_{\min}} } + \frac{k_{\max} \gamma}{N_{\min}}, \sqrt{\frac{\ln \tp{N m} }{N_*}} }}, \gamma < \gamma_0' \\ 
O\tp{\max\set{\sqrt{\frac{\ln \tp{N m} }{N_{\min} - k_{\max}}}, \sqrt{\frac{\ln \tp{N m} }{N_*}} }}, \gamma \ge \gamma_0'
\end{cases}$ }  \\
\midrule

\multirow{2}{3.2cm}{\\ Distribution \\ model ($\+M_d$)}
& \multirow{2}{1cm}{\\ Single}
& $C^*\in [2,\infty):$ \makecell[l]{Upper bound:
$O\tp{\min\set{ \sqrt{\frac{C^*\ln \tp{N m} }{\max\set{1, N - 2kC^*}}}, \sqrt{\frac{C^*\ln \tp{N m} }{N}} + \frac{ k \gamma C^*}{N} } }$    \\
Lower bound:
$\Omega\tp{e^{-\eps}\sqrt{\frac{C^*}{N-k}}}$ } \\
\cmidrule(lr){3-3}

& 
& $C^*\in (1,2):$ \makecell[l]{Upper bound:
$
O\tp{e^{-\frac{(N-k)}{2}\ln \frac{C^*}{8(C^*-1)} + \frac{(N+k)}{2}\ln \frac{2}{C^*} }}$ \\
Lower bound:
$\Omega\tp{(2-C^*)e^{-\eps-(N-k)(2-C^*)\ln\tp{\frac{2}{C^*-1}}} }$} \\
\bottomrule
\end{tabularx}
    \begin{tablenotes}[flushleft]
    \footnotesize
    \item[\(\star\)] $\+A$ is the set of base arms, $\abs{\+A} = m$. $N(a)$ denotes the sample counts of arm $a$ in the offline dataset $D$, $\abs{D} = N$. For $(\eps,\delta)$-unlearning, denote $\gamma = \frac{\sqrt{2\ln \frac{1.25}{\delta}}}{\eps}$. $\+M_f$ has the assumption that $N(a^*) \ge N_*$, while $\+M_d$ assumes that the behavior policy satisfies $d(a^*) \ge \frac{1}{C^*}$. More details are given in \Cref{sec:prelim}.
    \item[\(\circ\)] Single-source unlearning means only one arm (i.e. $a_0$) makes an unlearning request $U$. Let $\abs{U} = k$. Multi-source unlearning means at least two arms make the request, denote them as $\set{a_{u_i}}_{i=1}^\ell$. And the unlearning request $U_i$ from $a_{u_i}$ satisfies $\abs{U_i} = k_i$.
    \item[\(\dagger\)] For single-source, $\gamma_0= \frac{4}{3}\sqrt{\frac{\pi \ln (Nm)}{N(a_0)-k}}$. For multi-source, $N_{\min} = \min_{i\in [\ell]} N(a_{u_i}), k_{\max} = \max_{i\in[\ell]} k_i, \gamma_0' = \frac{4}{3}\sqrt{\frac{\pi \ln (Nm)}{N_{\min}-k_{\max}}}$. 
    \end{tablenotes}
\end{threeparttable}
\end{table*}

In this work, we take a first step by studying data-level unlearning for more fundamental decision-making primitive of offline stochastic multi-armed bandits. Offline MAB learning studies how to select the best arm purely from an offline dataset, without any additional exploration. The objective is to minimize the expected sub-optimality, which is the gap of mean reward between the best arm $a^*$ and the output arm $\hat a$. Unequal arm coverage is a central challenge in offline learning and directly motivates the pessimism principle. Concretely, pessimism discounts poorly supported arms through lower confidence bound (LCB) estimates, and has been widely adopted with minimax-optimal guarantees \cite{RZMJ21,LMS22}.

Since conventional machine unlearning methods are not directly applicable due to fundamental differences, our goal is to address two questions: (1) How can we design unlearning algorithms, and what is the trade-off between sub-optimality and privacy? (2) Under the same privacy constraint, which unlearning mechanism yields better sub-optimality? 

\paragraph{Our contributions.}

\textit{Problem formulation.} 
We initiate the study of $(\varepsilon,\delta)$-unlearning for offline stochastic MAB. Our definition follows the general framework of \citet{SAKS21,SW22,LLQR23}, which captures both performance on unseen test data and storage constraints. We evaluate utility via post-unlearning sub-optimality, and to systematically characterize how unlearning affects learning performance, we study two standard offline dataset models: the fixed-sample model and the distribution model. We also distinguish single- and multi-source unlearning based on how many arms the unlearning request $U$ comes from.

\textit{Algorithms and theoretical results.} 
We propose two unlearning base algorithms, the Gaussian mechanism and rollback, and develop unlearning algorithms that combine them across the above settings. Leveraging the learned model (arm), we detect whether the request $U$ comes from $a^*$, and choose between the two base algorithms according to the privacy constraint and problem parameters. Our analysis draws on pessimism principles from offline bandit learning. The results yield a simple guideline: when $U$ comes from $a^*$ and the privacy constraint is relatively loose, the Gaussian mechanism yields better performance; otherwise, rollback is preferable. We also introduce a mixing procedure to motivate this baseline design. We complement our upper bounds with lower bounds for single-source unlearning under both dataset models. When $\eps$ is small, our bounds are near-tight up to logarithmic factors in several regimes, including under $\+M_f$ when $N_* \ge N(a_0)-k$ and under $\+M_d$. Theoretical results are summarized in \Cref{tab:theory_results}.

\textit{Extensive experiments.} 
We conduct extensive experiments on synthetic offline MAB datasets to validate our theoretical findings and evaluate empirical performance. For both dataset models, our proposed algorithms consistently achieve competitive or superior performance compared to several baselines, and effectively select between Gaussian mechanism and rollback depending on the data regime and privacy constraint. These results demonstrate the robustness of our approach across different settings, and highlight the benefits of adaptivity in offline bandit unlearning.

\subsection{Related Work}

\paragraph{Machine unlearning.}
Machine unlearning was first proposed in \citet{CY15}, and \citet{GGVZ19} subsequently introduced a DP-inspired notion of unlearning that has served as the foundation for much of the later literature. Existing unlearning algorithms draw on a range of techniques, including perturbed gradient descent \citep{NRSM21,UMRR21}, projected residual updates \citep{ISCZ21}, and Newton-style updates \citep{GGHV20,SAKS21,SW22,LLQR23}. The first two approaches are typically analyzed through training loss and can incur substantial storage overhead, whereas Newton-style updates could provide test loss performance and storage-efficient implementations. Despite this progress, decision-making systems have received far less attention. Our work builds on the Newton-style updates, which applies a Newton step on the model parameters that largely removes the influence of the deleted samples. This method directly motivates the Gaussian mechanism and rollback components in our algorithm design.

\paragraph{Learning for offline stochastic multi-armed bandits.}

Offline (or batch) learning has attracted substantial recent attention in settings where additional exploration is costly or impossible. A growing body of work studies offline RL \citep{FMP19,KRNJ20,YTYE20,KZTL20,GSG21}, while parallel lines develop theory for offline bandits, including pessimism-based analyses for offline MAB and contextual MAB \citep{RZMJ21}, similar methods for offline linear contextual bandits \citep{LMS22}, and more recently offline combinatorial MAB \citep{LDZW25}. Across these problems, performance hinges on coverage of the arm space, and different data-collection models yield different assumptions and guarantees. In offline MAB, two standard models are the fixed-sample model, where per-arm counts are fixed in advance \citep{XWMD2021,LMS22,LTND22,WBHJ22}, and the distribution model, where arms are drawn i.i.d. from a behavior policy \cite{RZMJ21,ZZ24,LDZW25}. However, unlearning has not been systematically studied in offline bandits. And we fill this gap by analyzing offline bandit unlearning under both models and characterizing the resulting privacy-utility tradeoff.

\paragraph{Differential privacy.}
A closely related line of work \cite{CZZY20,RZLS20,ZCHL20,TWZW22,ZZ24} studies MAB under differential privacy (DP), especially local differential privacy (LDP), where users do not trust the server and thus each user-side curator applies a DP mechanism to the raw reward (or context / arm) before it is revealed to the learner. This privacy model could be applied to offline learning \cite{ZZ24}, where one may privatize each data point and then run an offline learner on the privatized dataset. Importantly, DP/LDP can be seen as a sufficient condition for unlearning \cite{GGHV20,SAKS21,LLQR23}. The downside is that achieving DP/LDP typically requires adding substantial noise, which may degrade decision quality. This motivates our central viewpoint: when unlearning from an offline dataset, we can do better than perturbing every data point by leveraging the structure of the requested deletions to achieve lower sub-optimality under the same privacy constraint.

\section{Preliminaries} \label{sec:prelim}
In this section, we introduce the model of offline stochastic MAB, including its learning and unlearning problems.

\subsection{Learning for Offline Stochastic MAB} 
\paragraph{Offline dataset.}
We focus on offline stochastic MAB. The environment has a set of base arms $\+A$ denoted by $[m] = \set{1,\cdots,m}$. Pulling arm $a\in \+A$ yields a random reward $r\in [0,1]$ drawn from a fixed but unknown distribution $R(a)$ with mean $\mu(a) = \E[r\sim R(a)]{r}$. An offline dataset $D=\set{(a_i,r_i)}_{i=1}^N$ of size $N$ is generated from these distributions, where the reward component
always follows $r_i\sim R(a_i)$ and the arm component $\set{a_i}_{i=1}^N$ depends on the dataset model. Let $a^* = \argmax_{a\in\+A}\mu(a)$ denote the best arm. We let $N(a)$ be the sample counts of arm $a$ in $D$ and $\+Z$ be the data point space. Therefore $D\in \+Z^N$.

To comprehensively study unlearning, we consider two standard offline dataset models which decide the per-arm sample counts in $D$: the fixed-sample model $\+M_f$ and the distribution model $\+M_d$. Under $\+M_f$ \cite{LMS22,LTND22,WBHJ22}, the sample-count vector $\vN\in \bb N^m$ is fixed in advance with $\norm{\vN}_1 = N$. The only randomness in $D(\vN)$ comes from the rewards. Under $\+M_d$ \cite{RZMJ21,ZZ24,LDZW25}, each arm $a_i$ is sampled i.i.d. from a behavior policy $d \in \Delta(\+A)$, and then $r_i\sim R(a_i)$. Equivalently, each data point is sampled i.i.d. from the joint distribution $\+D \defeq d\otimes R$, i.e., $(a_i,r_i)\sim \+D$ and $D\sim \+D^N$. 

\paragraph{Data coverage condition.}
Under the fixed-sample model $\+M_f$, to guarantee the learning algorithm’s performance, we require that for every dataset $D(\vN)$, $a^*$ satisfies $N(a^*) \ge N_*$ for some constant $N_*$. Under the distribution model $\+M_d$, following \citet{RZMJ21}, we assume a finite concentrability coefficient $C^* > 1$ such that $d(a^*) \ge \frac{1}{C^*}$. The parameter $C^*$ captures deviation between the behavior policy $d$ and the distribution induced by the optimal algorithm. It also serves to guarantee the performance via the Chernoff bound on $N(a^*)$.

\paragraph{Performance Metrics.}
 A learning algorithm $\pi$ trains on $D$ and outputs an arm $\hat a\in \+A$, i.e., $\pi: \+Z^N \to \+A$. The goal of the learning algorithm is to minimize the sub-optimality 
\[
    \+M_f: \inf_{\pi}\sup_{D(\vN)}\E{\mu(a^*) - \mu(\hat a)},
\]
\[
    \+M_d: \inf_{\pi}\sup_{D\sim \+D^N}\E{\mu(a^*) - \mu(\hat a)},
\]
where the expectation is taken over the generation of $D$ and output of $\pi$.
We take the LCB algorithm of \citet{RZMJ21} as our default learning algorithm. This choice is motivated by its minimax-optimal guarantees under both $\+M_f$ and $\+M_d$ when $C^* \ge 2$. Details of LCB algorithm are provided in \Cref{sec:algo}. For the regime $C^* \in(1,2)$ under $\+M_d$, imitation learning leads to a sharper guarantee and we present this as an extension in \Cref{sec:algo_extension}.

\subsection{Unlearning for Offline Stochastic MAB}

Given an offline dataset $D$, an unlearning request specifies a subset $U\subseteq D$ of size $\abs{U} = k$ to be deleted. In addition to the learned decision $\pi(D)\in\+A$, the unlearning procedure is allowed to access a stored statistic $T(D)\in\+T$, which summarizes $D$ but does not contain the full dataset and is intended to facilitate efficient unlearning. An unlearning algorithm $\pi'$ then takes $(U,\pi(D),T(D))$ as input and outputs an updated arm $\hat a'\in\+A$, i.e., $\pi': \+Z^k \times \+A \times \+T \rightarrow \+A$. Let $\emptyset$ be the empty request. We follow the definition of $(\eps,\delta)$-unlearning in \citet{SAKS21,LLQR23}.
\begin{definition}[($\eps, \delta$)-unlearning ($(\eps,\delta)$-UL) \cite{SAKS21,LLQR23}] \label{def:unlearning}
A privacy-preserving data-removal mechanism $\pi': \+Z^k \times \+A \times \+T \rightarrow \+A$ performs $(\eps,\delta)$-UL for a learning algorithm $\pi : \+Z^N \to \mathcal{A}$, if for all $D \subseteq \+Z^N, U\subseteq D, A \subseteq \+A$:
\begin{align*}
	    &\phantomeq \Pr{\pi'(U,\pi(D),T(D)) \in A} \\
        &\leq e^\eps\cdot  \Pr{\pi'(\emptyset, \pi(D\setminus U), T(D\setminus U)) \in A} + \delta,
\end{align*}
and 
\begin{align*}
        &\phantomeq \Pr{\pi'(\emptyset, \pi(D\setminus U), T(D\setminus U)) \in A} \\ 
        &\leq e^\eps\cdot \Pr{\pi'(U,\pi(D),T(D)) \in A} + \delta.
\end{align*}
\end{definition}
\Cref{def:unlearning} characterizes the indistinguishability between the output distribution of (1) the model trained on $D$ and then unlearned with $U$ and (2) the model trained on $D\setminus U$ and then unlearned with an empty set. 
We emphasize that $\emptyset$ in (2) is a fictitious request therefore additional information about the realistic request $U$ will also be stored in $T(D\setminus U)$. 

The $(\eps,\delta)$-UL requirement acts as a constraint in our optimization problem. Our objective is to minimize the expected sub-optimality between $\hat a'$ and $a^*$, subject to $\pi', \pi$ satisfying $(\eps,\delta)$-UL. More formally, under $\+M_f$, \, $SubOpt(\vN,N_*,k,\eps,\delta) =$
\begin{align*}
    \inf_{\pi',\pi:(\eps,\delta)-UL}\sup_{D(\vN), U\subseteq D} \E{\mu(a^*) - \mu(\hat a')},
\end{align*}
while under $\+M_d$, \, $SubOpt(N,C^*,k,\eps,\delta) =$
\begin{align*}
    \inf_{\pi',\pi:(\eps,\delta)-UL}\sup_{D\sim \+D^N, U\subseteq D} \E{\mu(a^*) - \mu(\hat a')}.
\end{align*}

The expectation is taken over the generation of $D$ and output of $\pi$ and $\pi'$. Throughout this paper, we assume that $U$ is selected independently of the reward in each data point. Consequently, the empirical mean reward of each arm computed from $D\setminus U$ remains an unbiased estimator of its true mean. We call the request single-source if all samples in $U$ come from the same arm, and multi-source if $U$ contains samples from multiple arms.

\section{Our Results} \label{sec:algo}
Challenges of unlearning in offline bandits involve a privacy-utility trade-off and must also account for storage and computation. Unlike classical machine unlearning, the effect of deletions depends on the source of $U$ (i.e., which arm the deleted samples come from). Finally, the algorithm design and analysis must adapt to the data-generation models.

In this section, we first present the learning algorithm and its sub-optimality guarantee, and then introduce two unlearning primitives, the Gaussian mechanism and rollback. Next, in
\Cref{sec:algo_single_ub}, we propose an adaptive algorithm for single-source unlearning under both $\+M_f$ and $\+M_d$ and establish upper bounds. In \Cref{sec:algo_single_lb}, we prove lower bounds for single-source unlearning under both models. In \Cref{sec:algo_extension}, we introduce a mixing procedure to motivate the two base algorithms and present an improved learning-unlearning pair for $\+M_d$ when $C^*\in(1,2)$. Finally, we extend to multi-source unlearning under $\+M_f$ and establish an upper bound.

\paragraph{The Lower Confidence Bound algorithm.} 
Denote $\hat \mu(a), b(a)$ as the empirical mean and the penalty term of arm $a$ respectively. \Cref{alg:LCB} calculates the LCB: $\hat \mu(a) - b(a)$ for every arm based on $D$ and outputs the arm with the largest LCB. And we have an upper bound as follows.

\begin{lemma}[Theorem 1 \cite{RZMJ21}]\label{lem:ub_LCB}
    For any offline dataset $D = \set{(a_i,r_i)}_{i=1}^N$, $\abs{\+A} = m$, the arm $\hat a$ returned by \Cref{alg:LCB} follows that 
    $\+M_f: \E{\mu(a^*) - \mu(\hat a) } \leq \sqrt{\frac{2\ln \frac{ m}{\tau}}{N_*}} + \tau$, and $\+M_d: \E{\mu(a^*) - \mu(\hat a) } \leq 2\sqrt{\frac{C^*\ln \frac{ m}{\tau}}{N}} + 2\tau$.
\end{lemma}

\begin{algorithm}[H]
\caption{LCB algorithm \cite{RZMJ21}}
\label{alg:LCB}
{\bfseries Input:} Dataset: $D=\set{(a_i,r_i)}^N_{i=1}$, a confidence level: $\tau \in (0,1)$  
\begin{algorithmic}[1]
    \STATE Set $N(a) \gets \sum_{i=1}^N \1{a_i=a}$ for all $a\in \+A$
    \FOR{$a\in \+A$}
        \IF{$N(a) = 0$}
            \STATE Set the empirical mean reward $\hat\mu(a) \gets 0$
            \STATE Set the penalty $b(a) \gets 1$
        \ELSE
            \STATE Compute the empirical mean reward  $\hat \mu(a) \gets \frac{\sum_{i=1}^N r_i\1{a_i = a} }{N(a)}$
            \STATE Compute the penalty $b(a) \gets \sqrt{\frac{\ln \frac{ m}{\tau}}{2N(a)}}$
        \ENDIF
    \ENDFOR \\
\end{algorithmic}
{\bfseries Output:} $\hat a= \argmax_{a} \hat \mu(a) - b(a)$ 
\end{algorithm}

\paragraph{Two unlearning base algorithms.} 
Our algorithms are built around two canonical extremes derived from Newton-style updates: Gaussian mechanism and rollback. The Gaussian mechanism adds calibrated Gaussian noise to the learner's output (e.g., the LCB of each arm) so that the released result satisfies the desired $(\eps,\delta)$-UL constraint. In contrast, rollback efficiently removes the request $U$ from the stored statistics and recomputes the learner's output on the retained dataset $D\setminus U$, which is equivalent to retraining on $D\setminus U$ for offline bandit learner since the learner depends only on the LCB. Our adaptive procedures compare these two base algorithms under the same privacy budget and select the one that yields smaller sub-optimality.

\paragraph{Notations.}
For the rest part of \Cref{sec:algo}, we denote $\wh \vmu$, $\wh \vN$ as the empirical mean vector and the empirical sample-count vector computed from $D$ respectively. Denote $\vf: \+Z^* \rightarrow \bb R^{m}$ as the function which calculates the LCB of each arm and arranges them in a vector based on a certain dataset. For any vector $\vnu$, let $\vnu_{(i)}$ be the $i$-th dimension of $\vnu$. Therefore $\wh \vf_{(i)}(D)$ is the LCB of the $i$-th arm on dataset $D$. Let $\gamma = \frac{\sqrt{2\ln \frac{1.25}{\delta}}}{\eps}$, $D' = D\setminus U$. Due to notational convenience and space constraint, we only provide unlearning algorithms when unlearned with the realistic non-empty set $U$. For algorithms when unlearned with the empty set, we place them in \Cref{app:algo}. They are included only as references in proving $(\eps,\delta)$-UL.

\subsection{Unlearning Algorithms and Upper Bounds}
\label{sec:algo_single_ub}
In this section, we consider the single-source unlearning under both models, where $U$ comes from arm $a_0$ and $\abs{U} = k$. We propose \Cref{alg:single} below. The algorithm takes as input the learned arm $\hat a$ from \Cref{alg:LCB}, the unlearning request $U$, additional statistics $\wh \vf(D)$, $\wh \vmu$, $\wh \vN$ and a confidence level $\tau$. Under $\+M_f$, $\wh \vN=\vN$ is fixed, whereas under $\+M_d$ it is sampled from behavior policy $d$. The key design principle is to trade off Gaussian mechanism and rollback. Intuitively, when $a_0$ coincides with the best arm and the privacy constraint is relatively loose, adding appropriately calibrated Gaussian noise preserves performance. In all other cases, rollback is preferable. Concretely, \Cref{alg:single} adds Gaussian noise if $\hat a = a_0$ and $\gamma < \gamma_0$; otherwise, it performs rollback. Detailed proofs are given in \Cref{app:algo_single_ub}.

\begin{algorithm}[H]
\caption{Single-source unlearning algorithm}
\label{alg:single}
{\bfseries Input:} Output of $\pi(D)$: $\hat a$, unlearning request: $U$ from $a_0$, $\abs{U}=k$, additional statistics $T(D)$: $\wh \vf(D)$, $\wh \vmu$, $\wh \vN$, 
a confidence level: $\tau \in (0,1)$  
\begin{algorithmic}[1]
    \STATE Set $\wh \vf(D') \gets \wh \vf(D), \gamma_0 \gets \frac{4}{3}\sqrt{\frac{\pi \ln \frac{m}{\tau}}{N(a_0)-k}}$
    \IF{$\tp{\hat a = a_0}$ and $\tp{\gamma < \gamma_0 }$ } 
        \STATE Set $\Delta_{\vf} \gets \frac{3k}{2N(a_0)}$, $\sigma \gets \Delta_{\vf} \gamma$ 
        \STATE Sample $\nu \in \bb R$ from $\+N(0,\sigma^2)$   
        \STATE Set $\wh \vf_{(a_0)}(D') \gets \wh \vf_{(a_0)}(D') + \nu$ \textcolor{gray}{\COMMENT{Gaussian Mechanism}}
    \ELSE 
        \STATE Set $N'(a_0) \gets N(a_0) - k$  
        \STATE Compute the new empirical mean reward  $\hat \mu'(a_0) \gets \frac{\hat\mu(a_0)\cdot N(a_0) - \sum_{(a_0,r_i) \in U} r_i}{N'(a_0)}$
        \STATE Compute the new penalty $b'(a_0) \gets \sqrt{\frac{\ln \frac{ m}{\tau}}{2N'(a_0)}}$ 
        \STATE Set $\wh \vf_{(a_0)}(D') \gets \hat \mu'(a_0) - b'(a_0)$  \textcolor{gray}{\COMMENT{Rollback}}
    \ENDIF    
\end{algorithmic}
{\bfseries Output:} $\hat a' = \argmax_{a} \wh \vf_{(a)}(D')$ 
\end{algorithm}

\begin{theorem} \label{thm:unlearn_fixed_single}
    \Cref{alg:single} and \Cref{alg:LCB} are $(\eps,\delta)$-unlearning under both $\+M_f$ and $\+M_d$.
\end{theorem}

To more finely characterize how the sample counts of the unlearned arm $a_0$ and the optimal arm $a^*$ in $D$ affect sub-optimality, we first study $\+M_f$, where the per-arm sample counts are fixed. We provide an upper bound of $SubOpt(N,N_*,k,\eps,\delta)$.

\begin{theorem}[$\+M_f$] \label{thm:fixed_single_ub}
Consider any offline dataset $D(\vN)$, any unlearning request $U$ from $a_0$. $N(a^*)\ge N_*, \abs{U} = k$. When $k \leq  N(a_0) - \sqrt{\frac{N(a_0) \ln \frac{m}{\tau}}{2}}$, by setting $\tau = \frac{1}{N}$, the arm $\hat a'$ returned by \Cref{alg:single} satisfies when $\gamma < \gamma_0 = \frac{4}{3}\sqrt{\frac{\pi \ln (Nm)}{N(a_0)-k}}$, \, $SubOpt(\vN,N_*,k,\eps,\delta) =$
\begin{align*}
    O\tp{\max\set{\sqrt{\frac{\ln \tp{N m} }{N(a_0)} } + \frac{k \gamma}{N(a_0)}, \sqrt{\frac{\ln \tp{N m} }{N_*}} }}. 
\end{align*}  
while when $\gamma \ge \gamma_0$, \, $SubOpt(\vN,N_*,k,\eps,\delta) =$
 \begin{align*}
    O\tp{\max\set{\sqrt{\frac{\ln \tp{N m} }{N(a_0) - k}}, \sqrt{\frac{\ln \tp{N m} }{N_*}} }}. 
\end{align*}    
\end{theorem}

\begin{remark}
Our proof reveals that when $a_0 = a^*$, the performance of \Cref{alg:single} is guaranteed by the learning algorithm (\Cref{alg:LCB}), corresponding to the second term in the upper bound. When $a_0 \neq a^*$, the performance of \Cref{alg:single} is the minimum of the Gaussian mechanism and rollback guarantees, with the dominating term switching as the privacy constraint $\gamma$ varies; this trade-off is captured by the first term in the upper bound. The threshold $\gamma_0$ is defined by comparing the two corresponding bounds, namely the value at which they coincide, and thus serves as the switching point between the two regimes.
\end{remark}

We further consider $\+M_d$, where the per-arm sample counts can vary every time. Despite this additional randomness, \Cref{alg:single} continues to achieve good performance as under $\+M_f$. The resulting bounds depend on the $N$, $k$, and the coverage parameter $C^*$, rather than on a fixed realization of the sample-count vector. Again we give the upper bound of $SubOpt(N,C^*,k,\eps,\delta)$. 

\begin{theorem}[$\+M_d$] \label{thm:dist_single_ub}
Consider any offline dataset $D\sim \+D^N$, where the behavior policy satisfies $d(a^*) \ge \frac{1}{C^*}$ for some $C^* > 1$, any unlearning request $U$ from $a_0$. $\abs{U} = k$. When $N > 8C^*\ln N$ and $k \leq  N(a_0) - \sqrt{\frac{N(a_0) \ln \frac{m}{\tau}}{2}}$, by setting $\tau = \frac{1}{N}$, the arm $\hat a'$ returned by \Cref{alg:single} satisfies 
\begin{align*}
    &SubOpt(N,C^*,k,\eps,\delta) = O\Bigg( \min\Bigg\{  \\
    &\sqrt{\frac{C^*\ln \tp{N m} }{\max\set{1, N - 2kC^*}}}, \sqrt{\frac{C^*\ln \tp{N m} }{N}} + \frac{ k \gamma C^*}{N}  \Bigg \} \Bigg).
\end{align*}   
\end{theorem}

\begin{remark}
Under $\+M_d$, similarly, when $a_0=a^*$, the performance is governed by \Cref{alg:LCB}, corresponding to the second term in the upper bound. When $a_0\neq a^*$, the performance is the better of the Gaussian mechanism and rollback guarantees, and the dominating term switches with the privacy constraint $\gamma$, which is captured by the first term.
\end{remark}

\subsection{Lower Bounds} \label{sec:algo_single_lb}
In this section, we present the lower bounds under both models for single-source unlearning. The proof idea is to utilize the definition of $(\eps,\delta)$-UL and apply Le Cam's method on two similar instances. Detailed proofs are given in \Cref{app:algo_single_lb}. 

\begin{theorem}[$\+M_f$] \label{thm:fixed_single_lb}
Let $\eps \ge 0$, $\delta \leq \frac{1}{8\sqrt{e}}$. For any unlearning request $U$ from $a_0$, the lower bound of sub-optimality satisfies $SubOpt(\vN,N_*,k,\eps,\delta) = \Omega\tp{e^{-\eps}\sqrt{\frac{1}{N(a_0)-k}}}$.
\end{theorem}

\begin{remark}
If $N_* \ge N(a_0) - k$, our upper bound is dominated by the first term. In this regime, the lower bound matches the upper bound up to a logarithmic factor when $\eps$ is relatively small. 
\end{remark}

\begin{theorem}[$\+M_d$] \label{thm:dist_single_lb}
Let $\eps \ge 0$, $\delta \leq \frac{1}{8\sqrt{e}}$. For any single-source unlearning request $U$, when $C^* \in [2, \infty)$, the lower bound of sub-optimality satisfies
$SubOpt(N,C^*,k,\eps,\delta) = \Omega\tp{e^{-\eps}\sqrt{\frac{C^*}{N-k}}}$, while $C^*\in (1,2)$, $SubOpt(N,C^*,k,\eps,0)=$ $\Omega\tp{(2-C^*)e^{-\eps-(N-k)(2-C^*)\ln\tp{\frac{2}{C^*-1}}} }$.
\end{theorem}

\begin{remark}
When $C^*\in [2,\infty)$, the lower bound matches the upper bound up to a logarithmic factor when $\eps$ is relatively small. When $C^*\in (1,2)$, the lower bound exhibits an $e^{-N}$ dependence on $N$ , which is much smaller than the generic upper bound. This motivates us to draw on ideas from imitation learning to design new learning-unlearning pairs in \Cref{sec:algo_extension}.
\end{remark}

\subsection{Extensions}
\label{sec:algo_extension}
\paragraph{3.3.1 The Mixing Algorithm.}
To justify the rationale behind our baselines, we introduce a mixing procedure under $\+M_f$ as a control baseline (as an example, we choose $\+M_f$, we could obtain the same conclusion under $\+M_d$). The mixing procedure interpolates between pure Gaussian perturbation and pure rollback. It introduces an additional parameter $k' \in [0,k]$ that uniformly selects a subset $U'\subseteq U$ to be handled via rollback first; it then adds Gaussian noise calibrated to the sensitivity of the remaining requests $U\setminus U'$. We show that, under the same privacy constraint, this mixing procedure cannot outperform the better of the two extremes. The algorithm and the corresponding analysis are deferred to \Cref{app:algo_fixed_single_mixing}. In \Cref{sec:exp}, we also include experiments that corroborate this conclusion and use the mixing procedure as an additional baseline.

\paragraph{3.3.2 Improved Results for $C^* \in (1,2)$.}
When $C^* \in (1,2)$ is known, an imitation learning rule can yield a faster dependence on $N$ in the offline MAB setting by simply returning the most frequently selected arm in the dataset $D$ \cite{RZMJ21}. Motivated by this improvement, we design \Cref{alg:dist_single_le2} to enhance the performance under this regime, which is $(0,0)$-UL with imitation learning. We have the following upper bound. The pseudo-code and corresponding analysis are deferred to \Cref{app:algo_dist_single_le2}. Then we have the following upper bound.

\begin{theorem}[$\+M_d$] \label{thm:dist_single_le2_ub}
Consider any offline dataset $D\sim \+D^N$, where $d(a^*) \ge \frac{1}{C^*}$ for some $C^* \in (1,2)$, any unlearning request $U$ from $a_0$. $\abs{U} = k$. When $k \leq min\set{\frac{3N}{4}, \frac{(2-C^*)N}{C^*}}$, the arm $\hat a'$ returned by \Cref{alg:dist_single_le2} satisfies $SubOpt(N,C^*,k,\eps,0) = O\tp{e^{-\frac{(N-k)}{2}\ln \frac{C^*}{8(C^*-1)} + \frac{(N+k)}{2}\ln \frac{2}{C^*} }}$.
\end{theorem}

\begin{remark}
By switching the learning algorithm to an imitation learning rule and using rollback for unlearning, \Cref{thm:dist_single_le2_ub} yields a sharper upper bound with an $e^{-N}$ dependence on $N$ when $C^*\rightarrow 1$. This term matches the lower bound in \Cref{thm:dist_single_lb} when $C^* \in(1,2)$. 
\end{remark}

\paragraph{3.3.3 Multi-Source Unlearning Algorithm Under $\+M_f$.}
For multi-source unlearning under $\+M_f$, the unlearning request $U$ may involve multiple arms $a_{u_1},\ldots,a_{u_\ell}$ and can be denoted as $U=\bigcup_{i=1}^{\ell} U_i$, where $U_i$ denotes the subset of deleted samples associated with arm 
$a_{u_i}$ for $1\le i\le \ell$, and $\lvert U_i\rvert = k_i$. We further define
$N_{\min} \;=\; \min_{i\in[\ell]} N(a_{u_i}), k_{\max} \;=\; \max_{i\in[\ell]} k_i$. We design \Cref{alg:fixed_multi}, which is $(\eps,\delta)$-UL with \Cref{alg:LCB}. Its pseudo-code and corresponding analysis are deferred to \Cref{app:algo_fixed_multi}. Accordingly, the form of sub-optimality becomes $SubOpt(\vN,N_*,\{k_i\}_{i=1}^{\ell},\varepsilon,\delta)$. We now present the corresponding upper bound.

\begin{theorem}[$\+M_f$] \label{thm:fixed_multi_ub}
Consider any offline dataset $D(\vN)$, any unlearning request $U = \bigcup_{i=1}^{\ell} U_i$ where $U_i$ is selected from the data points of $a_{u_i}$. $N(a^*) \ge N_*$, $\abs{U_i} = k_i, \forall i\in[\ell]$. When $k_i \leq  N(a_{u_i}) - \sqrt{\frac{N(a_{u_i})\ln\frac{m}{\tau}}{2}}, \forall i\in [\ell]$, $k_{\max} < N_{\min}$, by setting $\tau = \frac{1}{N}$, the arm $\hat a'$ returned by \Cref{alg:fixed_multi} satisfies when $\gamma < \gamma_0' = \frac{4}{3}\sqrt{\frac{\pi \ln \tp{Nm} }{N_{\min} - k_{\max}}}$, \, $SubOpt(\vN,N_*,\set{k_i}_{i=1}^\ell,\eps,\delta) = $
\begin{align*}
    O\tp{\max\set{\sqrt{\frac{\ln \tp{Nm}}{N_{\min}}} +  \frac{k_{\max} \gamma}{N_{\min}}, \sqrt{\frac{\ln \tp{N m} }{N_*}} }}. 
\end{align*}
while when $\gamma \ge \gamma_0$, \, $SubOpt(\vN,N_*,\set{k_i}_{i=1}^\ell,\eps,\delta) =$
\begin{align*}
    O\tp{\max\set{\sqrt{\frac{\ln \tp{N m} }{N_{\min} - k_{\max}} }, \sqrt{\frac{\ln \tp{N m} }{N_*}} }}. 
\end{align*} 
\end{theorem}  

\begin{remark}
Our bound is consistent with the single-source case under $\+M_f$ as a special instance: when $\ell=1$, we have $N_{\min}=N(a_{u_1})=N(a_0)$ and $k_{\max}=k_1=k$, and the multi-source guarantee reduces to the corresponding single-source bound. Future work includes designing multi-source unlearning algorithms and deriving upper bounds under $\+M_d$, as well as establishing matching lower bounds for the multi-source setting.
\end{remark}

\begin{figure*}[t]
    \centering
    \begin{subfigure}{0.46\linewidth}
        \centering
        \includegraphics[width=\linewidth]{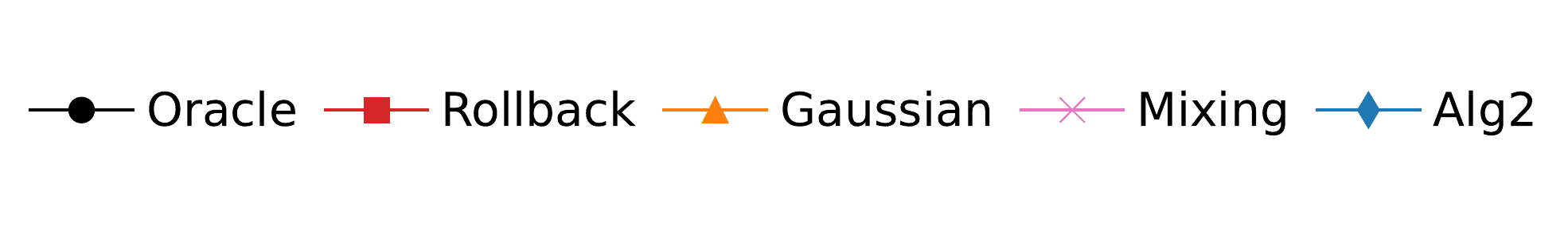}
    \end{subfigure}
    \begin{subfigure}{0.46\linewidth}
        \centering
        \includegraphics[width=\linewidth]{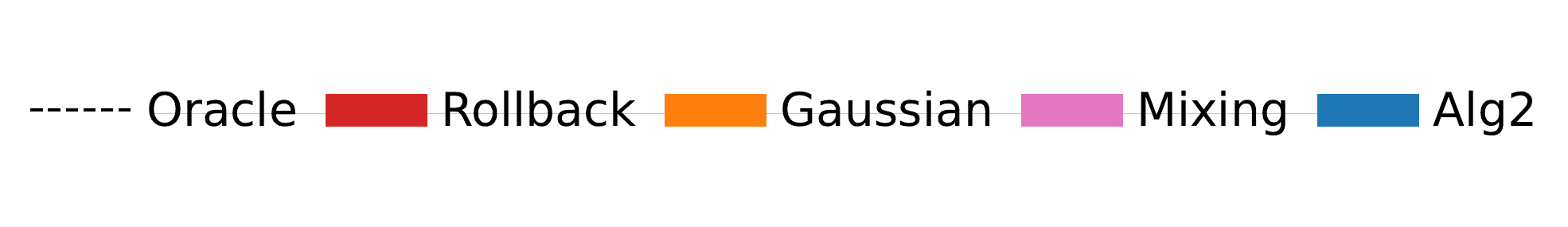}
    \end{subfigure}
    \hfill
    \begin{subfigure}{0.22\linewidth}
        \centering
        \includegraphics[width=\linewidth]{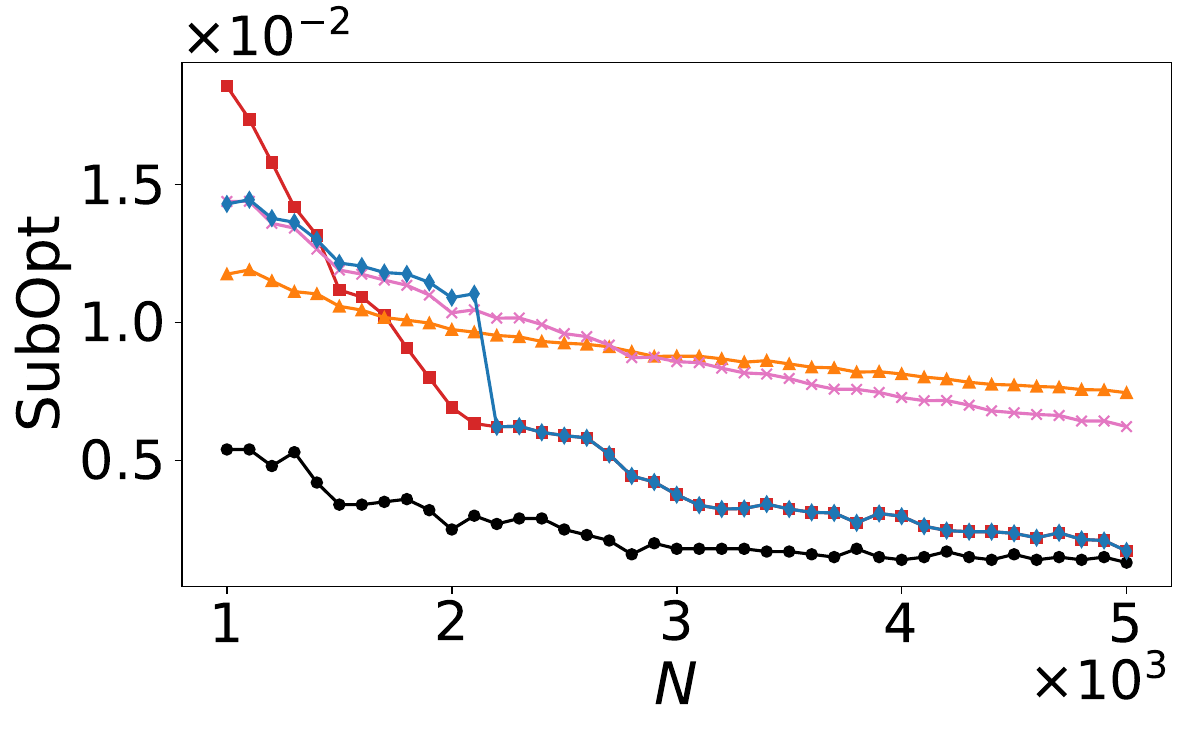}
        \caption{$\+M_f$, hard case}
    \end{subfigure}
    \begin{subfigure}{0.22\linewidth}
        \centering
        \includegraphics[width=\linewidth]{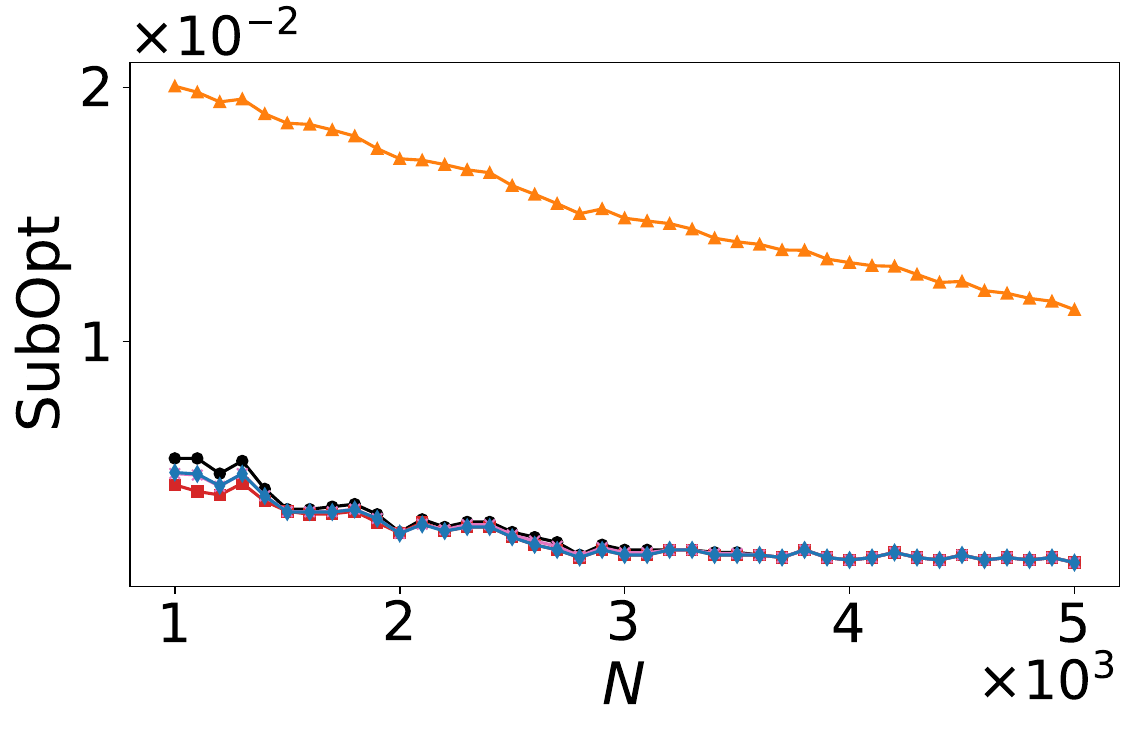}
        \caption{$\+M_f$, easy case}
    \end{subfigure}
    \begin{subfigure}{0.22\linewidth}
        \centering
        \includegraphics[width=\linewidth]{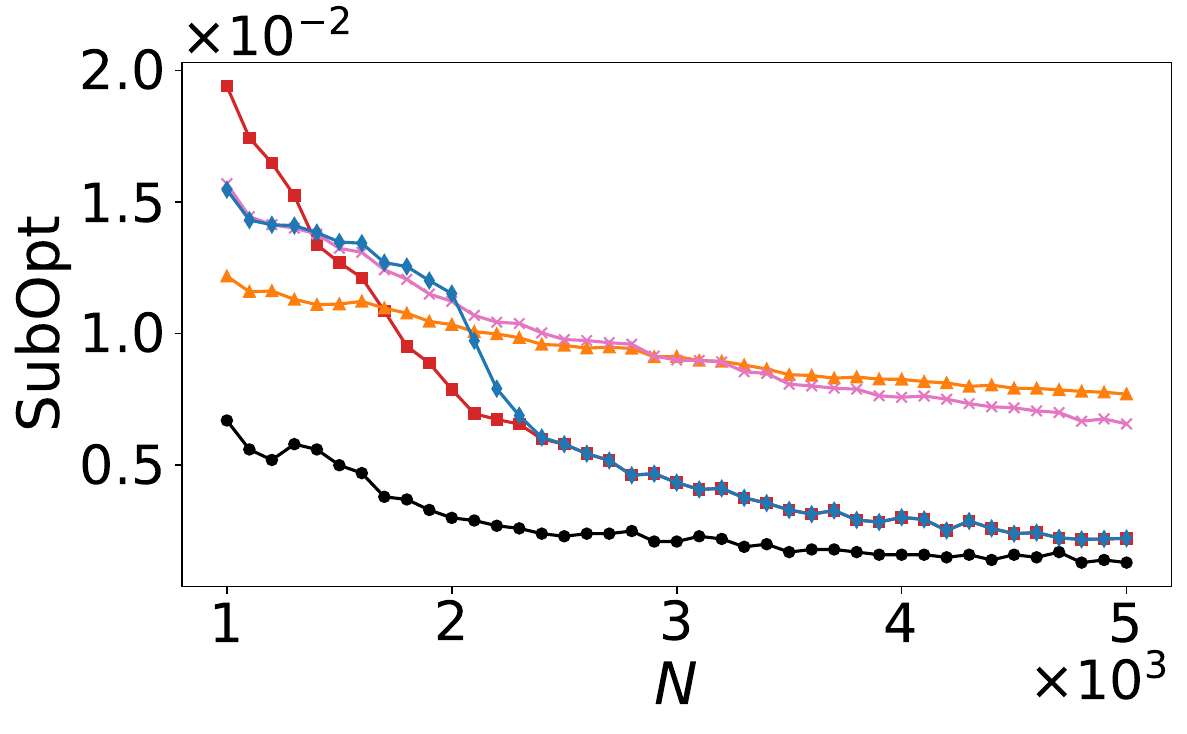}
        \caption{$\+M_d$, hard case}
    \end{subfigure}
    \begin{subfigure}{0.22\linewidth}
        \centering
        \includegraphics[width=\linewidth]{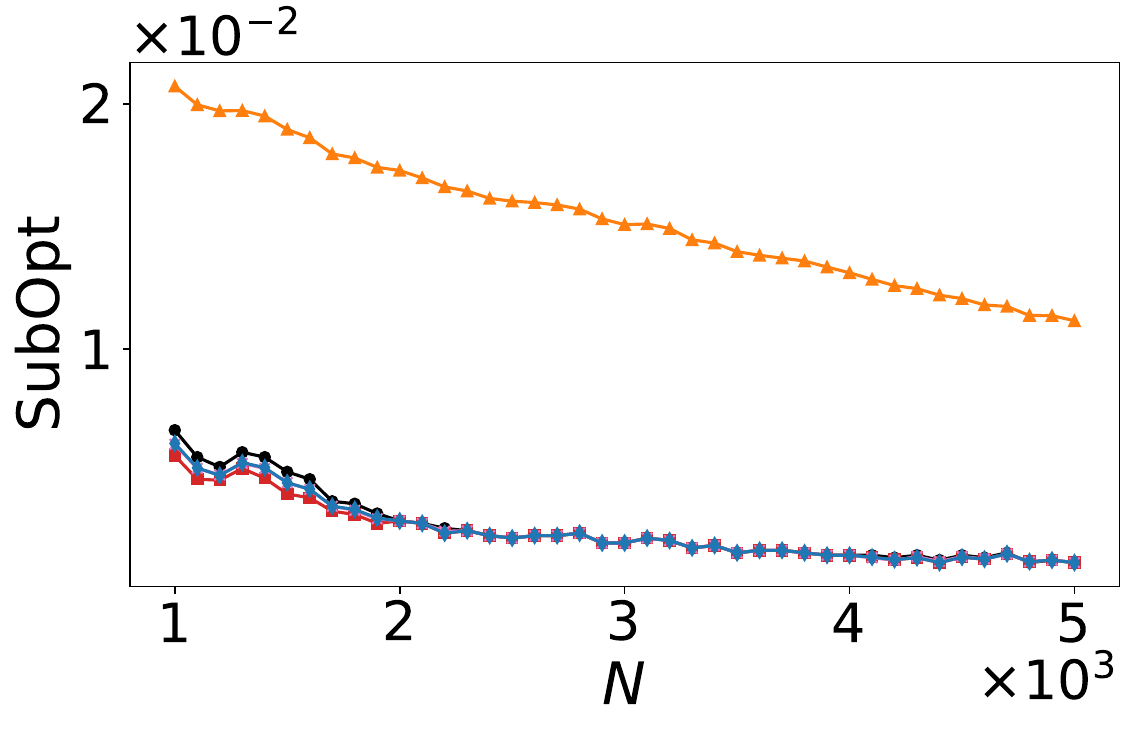}
        \caption{$\+M_d$, easy case}
    \end{subfigure}
    \hfill
    \begin{subfigure}{0.22\linewidth}
        \centering
        \includegraphics[width=\linewidth]{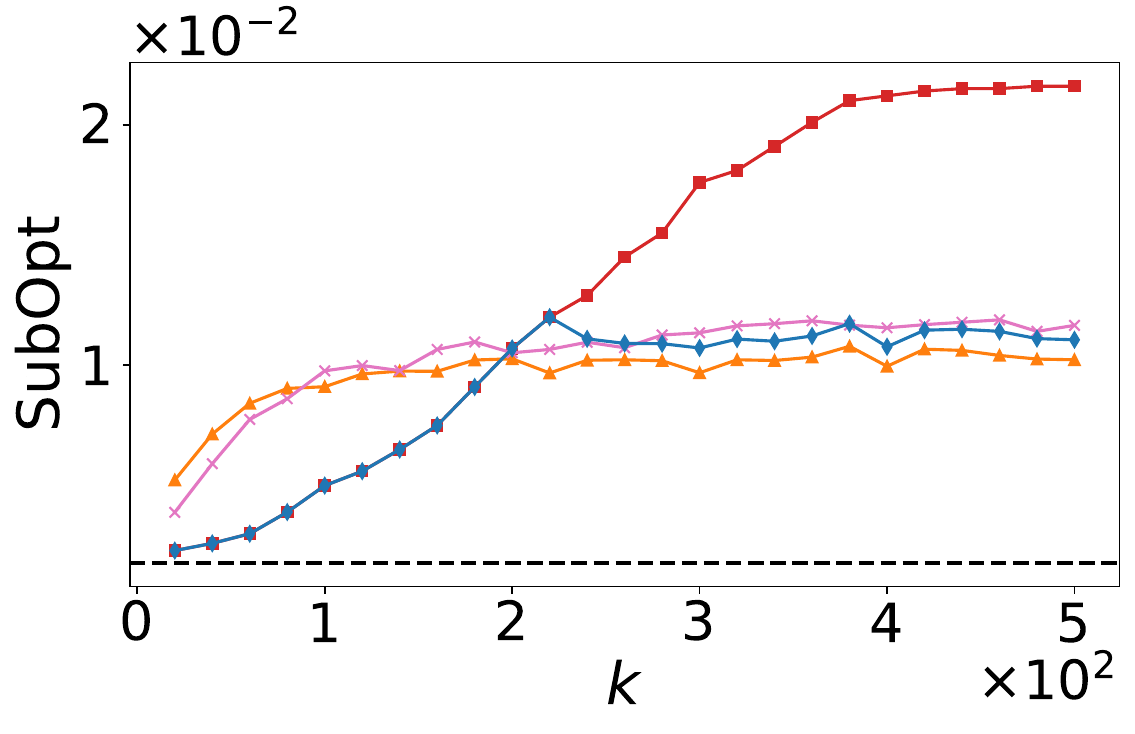}
        \caption{$\+M_f$, hard case}
    \end{subfigure}
    \begin{subfigure}{0.22\linewidth}
        \centering
        \includegraphics[width=\linewidth]{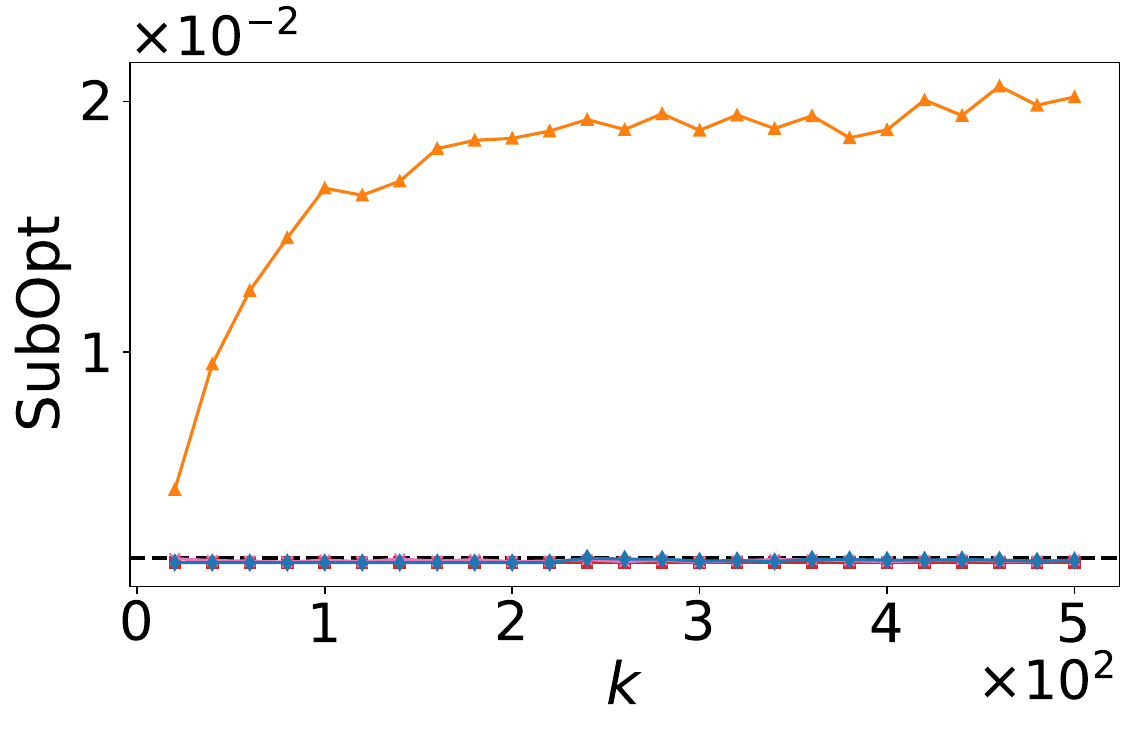}
        \caption{$\+M_f$, easy case}
    \end{subfigure}
    \begin{subfigure}{0.22\linewidth}
        \centering
        \includegraphics[width=\linewidth]{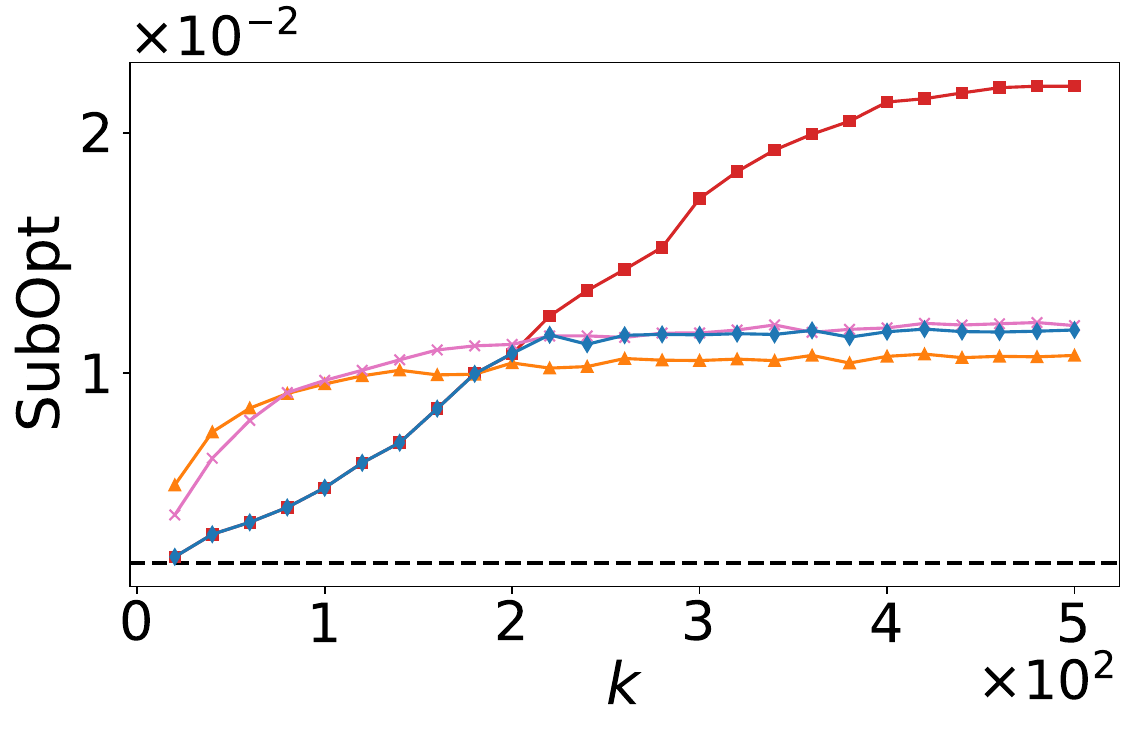}
        \caption{$\+M_d$, hard case}
    \end{subfigure}
    \begin{subfigure}{0.22\linewidth}
        \centering
        \includegraphics[width=\linewidth]{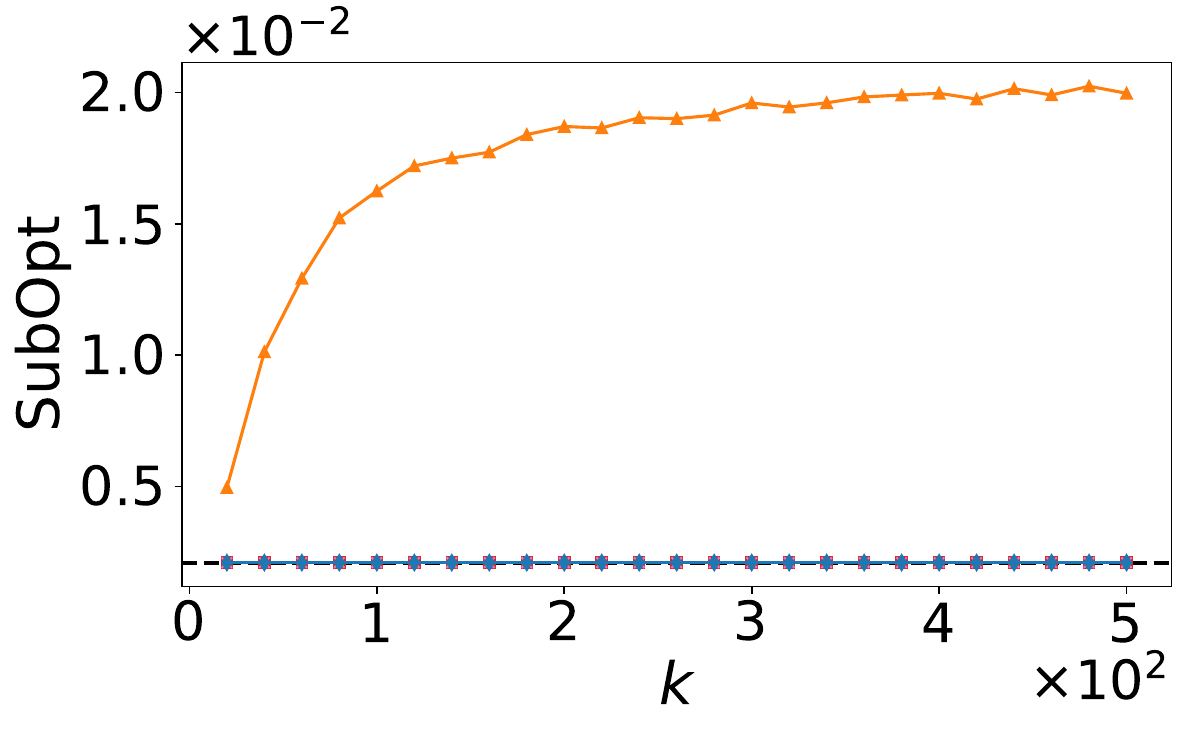}
        \caption{$\+M_d$, easy case}
    \end{subfigure}
    \hfill
    \begin{subfigure}{0.22\linewidth}
        \centering
        \includegraphics[width=\linewidth]{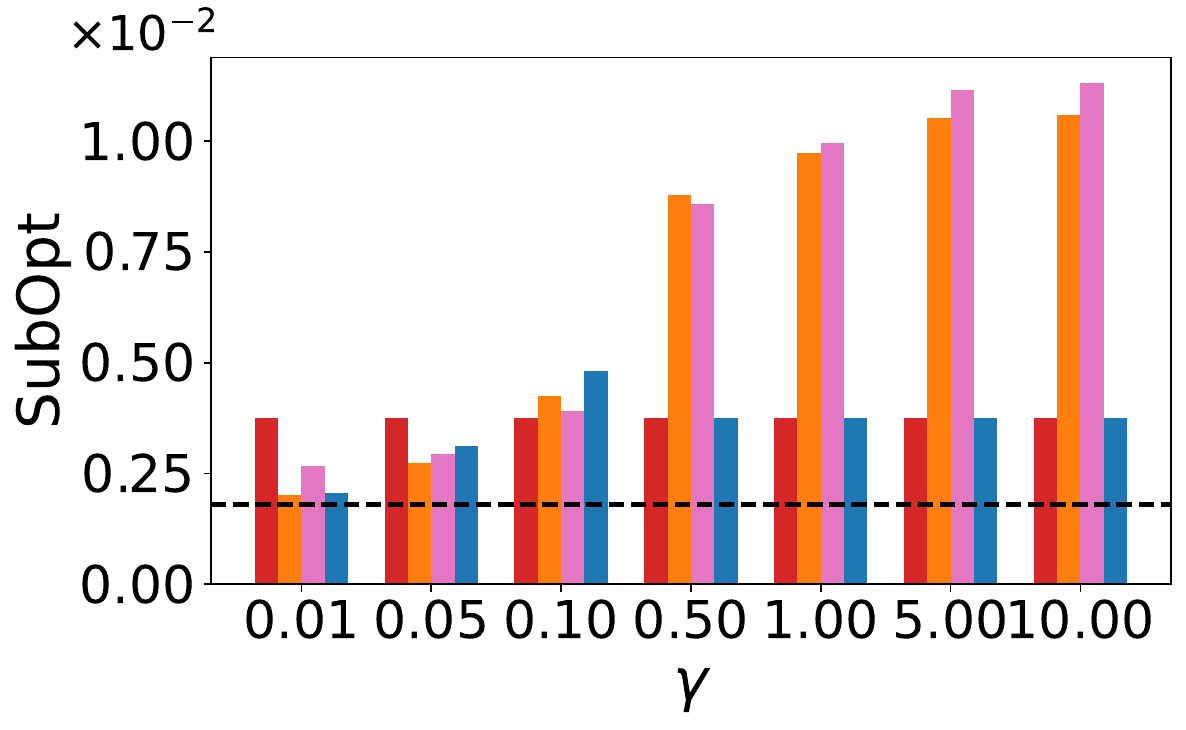}
        \caption{$\+M_f$, hard case}
    \end{subfigure}
    \begin{subfigure}{0.22\linewidth}
        \centering
        \includegraphics[width=\linewidth]{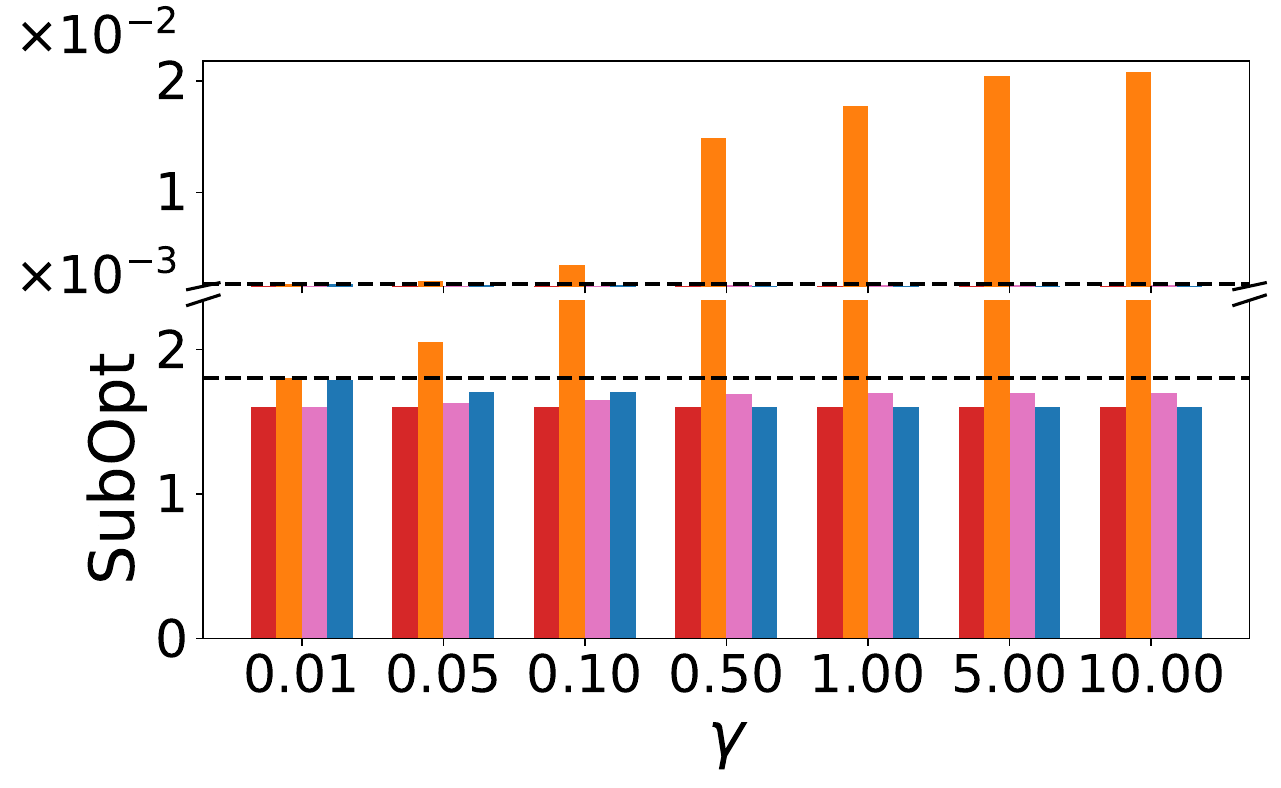}
        \caption{$\+M_f$, easy case}
    \end{subfigure}
    \begin{subfigure}{0.22\linewidth}
        \centering
        \includegraphics[width=\linewidth]{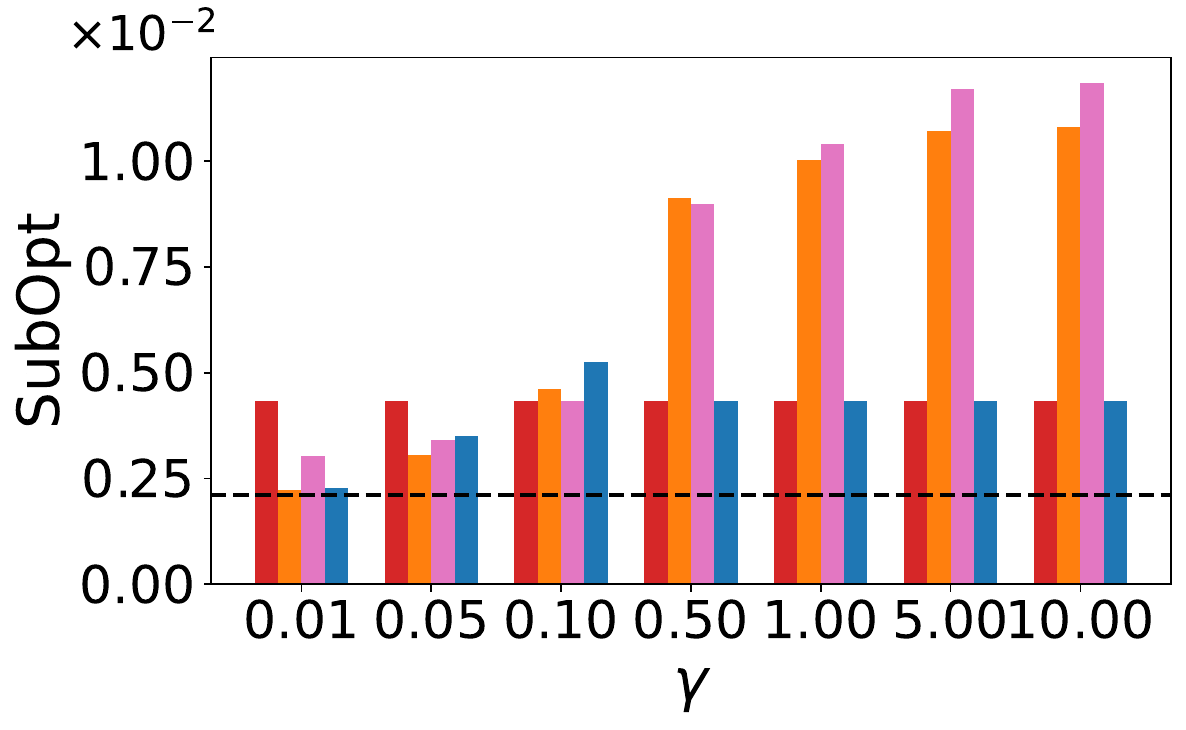}
        \caption{$\+M_d$, hard case}
    \end{subfigure}
    \begin{subfigure}{0.22\linewidth}
        \centering
        \includegraphics[width=\linewidth]{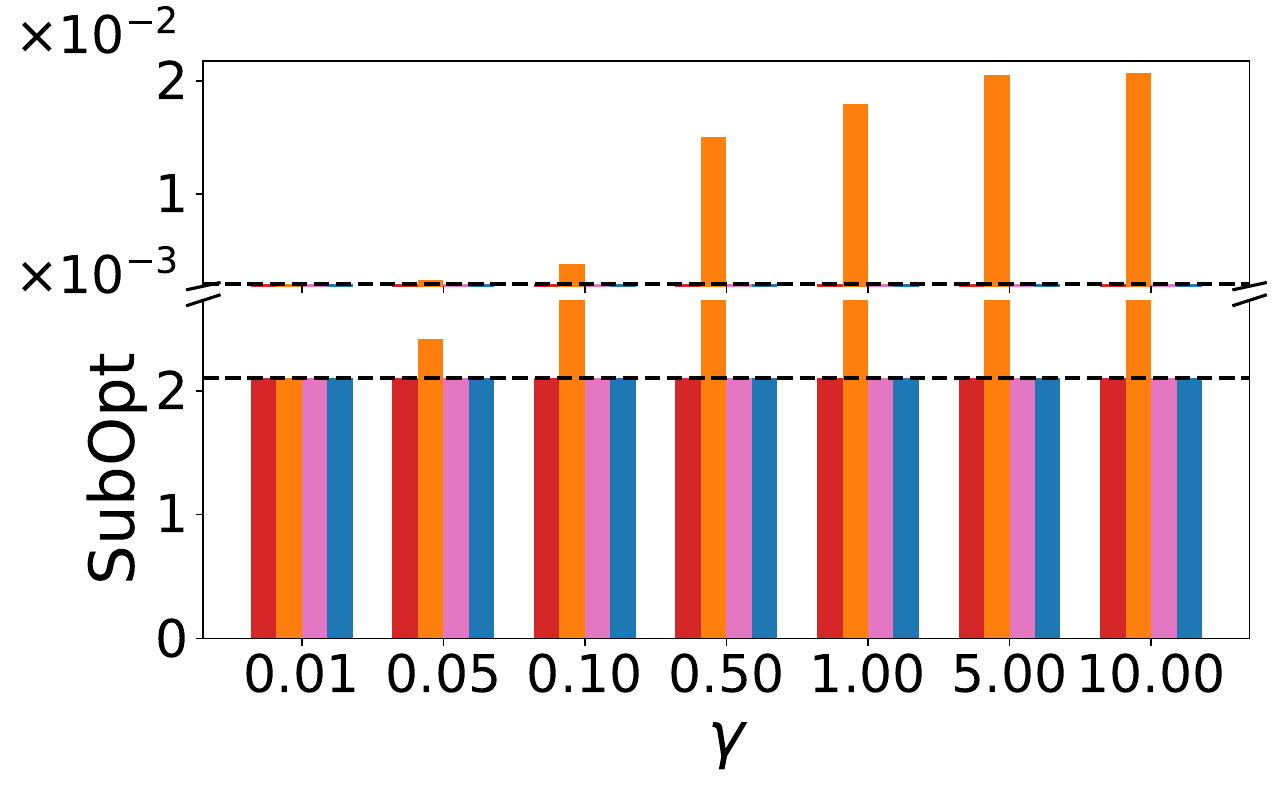}
        \caption{$\+M_d$, easy case}
    \end{subfigure}
    \caption{Comparison of \Cref{alg:single} with baselines under $\+M_f$ (uniform) and $\+M_d$ ($C^*=5$). }
    \label{fig:alg2.pdf}
\end{figure*}

\section{Experiments} \label{sec:exp}
In this section, we report experimental results on synthetic datasets. We compare \Cref{alg:single} with several baselines, including an Oracle, the Gaussian mechanism, rollback and the mixing algorithm under both $\+M_f$ and $\+M_d$. We also compare \Cref{alg:dist_single_le2} with \Cref{alg:single} when $C^* \in (1,2)$ under $\+M_d$. The Oracle corresponds to \Cref{alg:LCB} executed on the original dataset without deletion, while all other methods operate on datasets subject to deletion. The mixing algorithm here refers to the $k'=\frac{k}{2}$ case in \Cref{alg:fixed_single_mixing}. Performance is evaluated in terms of the sub-optimality. In practice, we set $\gamma_0^{\text{exp}} = 1.3\,\gamma_0$ in all experiments to improve performance of algorithms. The theoretical guarantees remain unchanged with $\gamma_0^{\text{exp}}$. Additional experimental details for experimental settings and experimental results of \Cref{alg:fixed_single_mixing} and \Cref{alg:fixed_multi} are deferred to \Cref{append:exp}.

\subsection{Experiments of the Fixed-Sample Model}\label{sec:exp_fixed}
We generate 200 independent synthetic offline bandit datasets using a round-robin behavior policy over 5 arms with Bernoulli rewards \(\vmu=(0.10,0.08,0.06,0.04,0.02)\). To efficiently evaluate deletions, we adopt a prefix-sharing strategy (details in \Cref{append:exp_fixed}). In both settings, each prefix undergoes 5 independent deletion requests. Both the hard case (\(a_0 = a^*\)) and the easy case (\(a_0\) is the second-best suboptimal arm) are considered, and results are averaged over 10 runs per configuration. We vary one of $N$, $k$, or $\gamma$
while fixing the remaining two parameters to $(N, k, \gamma) = (3000, 80, 0.5)$. Due to multiple sources of randomness, the variance across runs is heterogeneous. For clarity, we therefore report mean performance across runs. 

\Cref{fig:alg2.pdf} shows that \Cref{alg:single} consistently demonstrates robust and competitive performance across all settings under $\+M_f$, yielding low sub-optimality across a wide range of $N$, $k$, and $\gamma$. In the hard case, the algorithm adaptively balances rollback and Gaussian mechanism: for large $\frac{k}{N}$, it favors adding Gaussian noise, while for small $\frac{k}{N}$, it reliably selects rollback. Quantitatively, when varying $N$, \Cref{alg:single} achieves 35-77\% and 39-73\% reduction in sub-optimality relative to the Gaussian mechanism and mixing algorithm for $N \geq 2200$, and 4-23\% reduction relative to rollback for $N \leq 1300$. In the easy case, the algorithm remains highly stable, consistently achieving at least a 76\% improvement over Gaussian and staying within 11\% of the best method. Similar trends are observed when varying $k$ or $\gamma$, highlighting its overall robustness. Overall, \Cref{alg:single} consistently achieves low sub-optimality and ranks among the best or second-best methods across most regimes, while the baselines exhibit highly variable sub-optimality and perform well only in isolated settings.

\subsection{Experiments of the Distribution Model}\label{sec:exp_distribution}
When $C^*\in [2,\infty)$, we generate 200 independent synthetic offline datasets using a stochastic behavior policy that selects $a^*$ with probability $1/C^*, C^*=5$ and distributes the remaining probability uniformly over suboptimal arms. Other configurations are the same as those in \Cref{sec:exp_fixed}. \Cref{fig:alg2.pdf} shows that \Cref{alg:single} also maintains robust and competitive performance under $\+M_d$. In the hard case, it adaptively switches between Gaussian mechanism and rollback based on $\frac{k}{N}$, reliably stabilizing performance. When varying $N$, \Cref{alg:single} achieves 21-72\% and 24-68\% reduction in sub-optimality relative to the Gaussian mechanism and mixing algorithm for $N \geq 2200$, and 7-20\% reduction relative to rollback for $N \leq 1300$. In the easy case, the algorithm consistently attains at least a 70\% improvement over Gaussian and remains either the best-performing method or within 10\% of it. Similar trends are observed when varying $k$ or $\gamma$. Overall, \Cref{alg:single} consistently maintains low sub-optimality across regimes, while the baselines suffer from pronounced variability in sub-optimality.

The setup for $C^* \in (1,2)$ is identical to that for $C^* \in [2, \infty)$ and only the range of $C^*$ differs. We focus on the hard case, as the sample counts of $a^*$ dominates the dataset in this regime. As shown in \Cref{tab:dist_single_le2}, when $C^* = 1.3$, \Cref{alg:dist_single_le2} achieves consistently competitive performance across different deletion scenarios, yielding up to more than $93\%$ reduction in sub-optimality compared to \Cref{alg:single}.

\begin{table}[H]
    \centering
    \caption{Comparison of \Cref{alg:dist_single_le2} and \Cref{alg:single} under distribution model ($C^*=1.3$).}
    \label{tab:dist_single_le2}
    \resizebox{\linewidth}{!}{%
    \begin{tabular}{lcccc}
        \toprule
        \multirow{2}{*}{Algorithm} & \multicolumn{4}{c}{Sub-optimality} \\ 
        \cmidrule(lr){2-5}
                  & $N=900$ & $N=1000$ & $N=1100$ & $N \geq 1200$ \\
        \midrule
        \Cref{alg:dist_single_le2} & 0.0008 & 0.0000 & 0.0000 & 0.0000 \\
        \Cref{alg:single}           & 0.0120 & 0.0115 & 0.0110 & 0.0000 \\
        \bottomrule
    \end{tabular}%
    }
\end{table}

\section{Conclusion and Future Work}
In this paper, we introduce $(\varepsilon,\delta)$-unlearning for offline stochastic MAB. We design adaptive algorithms for single-source unlearning under both the fixed-sample and distribution models, and provide upper and lower bounds in both models. We further develop extensions including a mixing baseline, an alternative learning--unlearning pair under the distribution model when $C^*\in (1,2)$, and multi-source unlearning under the fixed-sample model. Experiments validate the predicted trade-offs and the effectiveness of our methods. Several promising directions for future work emerge from our results. These include tightening the remaining gaps between upper and lower bounds, extending multi-source unlearning to the distribution model (both algorithm design and theory), and generalizing our framework to contextual/linear bandits or broader offline RL settings.

\section*{Impact Statement}

This paper presents work whose goal is to advance the field of Machine
Learning. There are many potential societal consequences of our work, none
which we feel must be specifically highlighted here.

\bibliography{example_paper}
\bibliographystyle{icml2026}

\newpage

\appendix
\onecolumn

\section{Missing Proofs in \Cref{sec:algo}} \label{app:algo}

\subsection{Algorithms and Upper Bounds for Single-Source Unlearning} \label{app:algo_single_ub}

Remember that the realistic unlearning request $U$ is from $a_0$, $\abs{U} = k$, $D' = D\setminus U$ and $N(a_0)$ is the sample counts of $a_0$ in $D$. We present the single-source unlearning algorithm when the unlearning request is $\emptyset$.

\begin{algorithm}[H]
\caption{Single-source unlearning algorithm with input $\emptyset$}
{\bfseries Input:} Output of $\pi(D')$, additional statistics $T(D')$: $\pi(D) = \hat a$, $a_0$, $k$, $N(a_0)$, $\wh \vf(D')$, a confidence level: $\tau \in (0,1)$  
\begin{algorithmic}[1]
    \STATE $\gamma_0 \gets \frac{4}{3}\sqrt{\frac{\pi \ln \frac{m}{\tau}}{N(a_0)-k}}$
    \IF{$\tp{\hat a = a_0}$ and $\tp{\gamma < \gamma_0 }$ } 
        \STATE Set $\Delta_{\vf} \gets \frac{3k}{2N(a_0)}$, $\sigma \gets \Delta_{\vf}\gamma$ 
        \STATE Sample $\nu \in \bb R$ from $\+N(0,\sigma^2)$   
        \STATE Set $\wh \vf_{(a_0)}(D') \gets \wh \vf_{(a_0)}(D') + \nu$ 
    \ENDIF    
\end{algorithmic}
{\bfseries Output:} $\hat a' = \argmax_{a} \wh \vf_{(a)}(D')$ 
\end{algorithm}

Note that the above algorithm adds the Gaussian noise only if \Cref{alg:single} adds the Gaussian noise, and these noises are sampled from the same distribution. Next we prove $(\eps,\delta)$-UL for \Cref{alg:single} and \Cref{alg:LCB}. The first step is to calculate $\Delta_{\vf}$, which is the $\ell_2$ sensitivity of $\vf$.

\begin{lemma} \label{lem:LCB_sensi}
    For any dataset $D$, any unlearning request $U$ from $a_0$, when $k \leq  N(a_0) - \sqrt{\frac{N(a_0) \ln \frac{m}{\tau}}{2}}$, we have $\Delta_{\vf} \leq \frac{3k}{2N(a_0)}$.
\end{lemma}

\begin{proof}
For any dataset $D$, after the unlearning of $U\subseteq D$, $\vf_{(a_0)}(D')$ becomes $\hat\mu'(a_0) - b'(a_0)$ where $N'(a_0)= N(a_0)-k$, $\hat\mu'(a_0) = \frac{\hat\mu(a_0)\cdot N(a_0) - \sum_{(a_0,r_i)\in U} r_i}{N'(a_0)}$ and $b'(a_0) = \sqrt{\frac{\ln \frac{m}{\tau}}{2N'(a_0)}}$. For $a\neq a_0$, $\vf_{(a)}(D') = \vf_{(a)}(D)$. Now we calculate the sensitivity when the dataset changes from $D$ to $D\setminus U = D'$:

\begin{align*}
    \Delta_{\vf} 
    &= \max_{D, D'} \norm{\vf(D)-\vf(D')}_2 \\
    &= \max_{D, D'}
    \abs{\hat\mu(a_0) - \frac{\hat\mu(a_0)\cdot N(a_0) - \sum_{(a_0,r_i)\in U} r_i}{N(a_0)-k} + \sqrt{\ln \frac{m}{\tau}} \tp{\sqrt{\frac{1}{2(N(a_0) - k)}} - \sqrt{\frac{1}{2N(a_0)}}} } \\
    &\leq \max_{D, D'}
    \abs{\frac{ N(a_0)\cdot \sum_{(a_0,r_i)\in U} r_i - k \sum_{(a_0,r_i)\in D} r_i }{N(a_0)\tp{N(a_0)-k}} + \frac{k \sqrt{\ln \frac{m}{\tau}}}{2\tp{N(a_0)-k}\sqrt{2N(a_0)} } } \\
    &\overset{(a)}{=} \abs{\frac{k}{N(a_0)} +   \frac{k \sqrt{\ln \frac{m}{\tau}}}{2\tp{N(a_0)-k}\sqrt{2N(a_0)} } } \overset{(b)}{\leq} \frac{3k}{2N(a_0)},
\end{align*}
where $(a)$ comes from choosing $r_i = 1, \forall (a_0,r_i)\in U$ and $r_j = 0, \forall (a_0,r_j)\in D\setminus U$, $(b)$ comes from $k \leq N(a_0) - \sqrt{\frac{N(a_0)\ln\frac{m}{\tau}}{2}}$.
\end{proof}

\begin{theorem}[\Cref{thm:unlearn_fixed_single} restated]
    \Cref{alg:single} and \Cref{alg:LCB} are $(\eps,\delta)$-unlearning.
\end{theorem}

\begin{proof}
From \Cref{lem:LCB_sensi}, we know that 
\[
    \abs{\vf(D\setminus U) - \vf(D)} \leq \Delta_{\vf} \leq \frac{3k}{2N(a_0)}.
\]
Let \(\wh \vf_1, \wh \vf_1'\) denote the LCB vector at the input and output step when the inputs are \(\emptyset,\,\pi(D\setminus U),\,T(D\setminus U)\) respectively, we have \(\wh \vf_1 = \vf(D\setminus U)\). Similarly, let \(\wh \vf_2, \wh \vf_2'\) denote the corresponding vector when the inputs are \(U,\,\pi(D),\,T(D)\), then \(\wh \vf_2 = \vf(D)\). If the unlearning algorithm adds Gaussian noise, then by the proof of the Gaussian mechanism (Theorem~A.1) in \citet{DR14}, for any measurable set $R\subseteq \bb R^{m}$,
\[
    \Pr{\wh \vf_1' \in R} \leq e^\eps\cdot  \Pr{\wh \vf_2' \in R} + \delta, 
    \quad \Pr{\wh \vf_2' \in R} \leq e^\eps\cdot \Pr{\wh \vf_1' \in R} + \delta.
\]
If the unlearning algorithm instead executes rollback, then $\wh \vf_2' = \vf(D\setminus U) = \wh \vf_1 = \wh \vf_1'$, and hence $\Pr{\wh \vf_1' \in R} = \Pr{\wh \vf_2' \in R}$. Finally, viewing \(g(\vf(D)) := \argmax_{a} \vf_{(a)}(D)\) as a post-processing map and applying the post-processing lemma (Proposition~2.1) in \citet{DR14}, we conclude that \Cref{alg:single} and \Cref{alg:LCB} satisfy \((\eps,\delta)\)-unlearning.
\end{proof}

\begin{theorem} [\Cref{thm:fixed_single_ub} restated]
Consider any offline dataset $D(\vN)$, any unlearning request $U$ from $a_0$. $N(a^*)\ge N_*, \abs{U} = k$. When $k \leq  N(a_0) - \sqrt{\frac{N(a_0) \ln \frac{m}{\tau}}{2}}$, by setting $\tau = \frac{1}{N}$, the arm $\hat a'$ returned by \Cref{alg:single} satisfies when $\gamma < \gamma_0 = \frac{4}{3}\sqrt{\frac{\pi \ln (Nm)}{N(a_0)-k}}$,
\begin{align*}
    &\phantomeq SubOpt(\vN,N_*,k,\eps,\delta) = O\tp{\max\set{\sqrt{\frac{\ln \tp{N m} }{N(a_0)} } + \frac{k \gamma}{N(a_0)}, \sqrt{\frac{\ln \tp{N m} }{N_*}} }}. 
\end{align*}  
while when $\gamma \ge \gamma_0$,
\begin{align*}
    &\phantomeq SubOpt(\vN,N_*,k,\eps,\delta) = O\tp{\max\set{\sqrt{\frac{\ln \tp{N m} }{N(a_0) - k}}, \sqrt{\frac{\ln \tp{N m} }{N_*}} }}. 
\end{align*}    
\end{theorem}

\begin{proof}
Note that through the proof, $a_0$ has been fixed whereas $a^*$ depends on the reward distribution underlying $D(\vN)$. When $\gamma < \gamma_0$, we first consider the case where $a_0 = a^*$. Denote $\Eps_1$ as the event that $\abs{\hat \mu(a) - \mu(a)} \leq \sqrt{\frac{\ln \frac{ m}{\tau}}{2N(a)}}$ for all arms $a$ on the dataset $D$. Applying Hoeffding's inequality for a fixed arm $a$, we have 
\[
    \Pr{\abs{\hat \mu(a) - \mu(a)} \leq \sqrt{\frac{\ln \frac{ m}{\tau}}{2N(a)}}} \ge 1 - \frac{2\tau}{ m},
\]
which implies that $\Pr{\Eps_1} \ge 1- 2\tau$ using the union bound. Furthermore, regardless of whether $a_0 = \hat a$ occurs or not, we could always let $\Eps_2$ be the event that
$\abs{\nu} \leq \sigma\sqrt{2\ln\frac{ m}{\tau}} = t$, according to the property of Gaussian $\+M_d$, we have 
\[
    \Pr{\abs{\nu} \leq t} \ge 1 - 2e^{-\frac{t^2}{2\sigma^2}} = 1- \frac{2\tau}{ m},
\]
then we obtain $\Pr{\Eps_2} \ge 1- \tau$ since $m \ge 2$. Therefore under the condition that $\Eps = \Eps_1 \cap \Eps_2$ occurs, for every $a\in \+A$ and $\nu$, we have 
\[
    \abs{\hat\mu(a) - \mu(a)} \leq b(a), \quad \abs{\nu} \leq \sigma\sqrt{2\ln\frac{ m}{\tau}}.  
\]
If the event $a_0 = \hat a$ occurs, \Cref{alg:single} will add Gaussian noise sampled from $\+N(0,\sigma^2 )$. In view of the definition of $\hat a'$, for $\hat a' \neq a^*$, we have 
\[
    \hat\mu(a^*) - b(a^*) + \nu \leq \hat\mu(\hat a') - b(\hat a') 
    \Rightarrow \mu(a^*) \leq \mu(\hat a') + 2b(a^*) - \nu,
\] 
and this occurs only if $\nu < 0$. Then we have
\begin{align*}
    \E{\mu(a^*) - \mu(\hat a') \mid a_0 = \hat a} 
    &\leq \sqrt{\frac{2\ln \frac{ m}{\tau}}{N(a^*)}} + \frac{1}{\Pr{\Eps_2}}\int_0^{\sigma\sqrt{2\ln\frac{ m}{\tau}}}\frac{x}{\sqrt{2\pi}\sigma}e^{-\frac{x^2}{2\sigma^2}}dx \\ 
    &\leq \sqrt{\frac{2\ln \frac{ m}{\tau}}{N(a_0)}} + \frac{\sigma}{\sqrt{2\pi}}.
\end{align*}
If the event $a_0 \neq \hat a$ occurs, \Cref{alg:single} will execute rollback on $a_0 = a^*$. Before the rollback, we know that $\wh \vf_{(\hat a)}(D)$ is the largest among all $a\in \+A$. Since $\wh \vf_{(a)}(D') = \wh \vf_{(a)}(D)$ for every $a\neq a^*$, $\hat a'$ is either $\hat a$ or $a^*$ after the rollback. Then the expected sub-optimality will not be worse than the learning algorithm, from \Cref{lem:ub_LCB}, we have
\[
    \E{\mu(a^*) - \mu(\hat a') \mid a_0 \neq \hat a} \leq \sqrt{\frac{2\ln \frac{ m}{\tau}}{N(a^*)}} + \tau = \sqrt{\frac{2\ln \frac{ m}{\tau}}{N(a_0)}} + \tau.
\]
In conclusion, the expected sub-optimality satisfies:
\begin{align*}
    \E{\mu(a^*) - \mu(\hat a') } 
     &\leq \E{ \mu(a^*) - \mu(\hat a') \mid \Eps} +  
     \Pr{\Eps^c}\cdot \E{\mu(a^*) - \mu(\hat a') \mid \Eps^c} \\ 
     &= \E{\tp{\mu(a^*) - \mu(\hat a')}\1{a_0 = \hat a} \mid \Eps} +  \E{\tp{\mu(a^*) - \mu(\hat a')}\1{a_0 \neq \hat a} \mid \Eps} + \sum_{j=1}^2 \Pr{\Eps_j^c} \\ 
     &\leq \Pr{a_0 = \hat a \mid \Eps}\cdot \tp{\sqrt{\frac{2 \ln \frac{ m}{\tau}}{N(a_0)}} + \frac{\sigma}{\sqrt{2\pi}} } + \Pr{a_0 \neq \hat a \mid \Eps}\cdot \tp{ \sqrt{\frac{2\ln \frac{ m}{\tau}}{N(a_0)}} + \tau } + 3\tau \\
     &\leq \sqrt{\frac{2\ln \frac{ m}{\tau}}{N(a_0)}} + \frac{\sigma}{\sqrt{2\pi}} + 3\tau \\ 
     &= \sqrt{\frac{2\ln \tp{Nm}}{N(a_0)}} + \frac{3k\gamma}{2\sqrt{2\pi} N(a_0)} + \frac{3}{N}. 
\end{align*}

The second case is $a_0 \neq a^*$. If the event $a_0 = \hat a$ occurs, then \Cref{alg:single} adds Gaussian noise only to $\wh \vf_{(a_0)}(D)$, while leaving $\wh \vf_{(a)}(D)$ unchanged for all $a\neq a_0$. If the noise realization $\nu \ge 0$, the selected arm remains $\hat a' = a_0$, and the resulting sub-optimality is controlled by the guarantee of the learning algorithm. If $\nu \leq 0$, the guarantee still holds, since all coordinates of $\wh \vf_{(a)}(D)$ for $a\neq a_0$ are unchanged. 

If instead the event $a_0\neq \hat a$ occurs, then \Cref{alg:single} executes rollback on arm $a_0$, and the expected sub-optimality cannot be worse than that of the learning algorithm. Therefore, when $a_0\neq a^*$, the expected sub-optimality satisfies
\begin{align*}
 \E{\mu(a^*) - \mu(\hat a') } 
     &\leq  \E{\mu(a^*) - \mu(\hat a) } \leq \sqrt{\frac{2\ln \tp{Nm}}{N(a^*)}} + \tau \leq \sqrt{\frac{2\ln \tp{Nm}}{N_*}} + \frac{1}{N}.  
\end{align*}
In conclusion, when $\gamma < \gamma_0$, the upper bound satisfies
\begin{align*}
 \E{\mu(a^*) - \mu(\hat a')} 
     &=  
     \begin{cases}
        O\tp{\sqrt{\frac{\ln \tp{N m} }{N(a_0)}} + \frac{k\gamma}{N(a_0)}}, & a_0 = a^*, \\
        O\tp{\sqrt{\frac{\ln \tp{N m} }{N_*}}}, & a_0 \neq a^*.
    \end{cases}
\end{align*}
Thus 
\[
SubOpt(\vN,N_*,k,\eps,\delta) = O\tp{\max\set{\sqrt{\frac{\ln \tp{N m} }{N(a_0)} } + \frac{k \gamma}{N(a_0)}, \sqrt{\frac{\ln \tp{N m} }{N_*}} }}.
\]

When $\gamma \ge \gamma_0$. \Cref{alg:single} executes rollback. When $a_0 \neq a^*$, we could similarly prove that the performance will be guaranteed by the learning algorithm,  
\begin{align*}
    \E{\mu(a^*) - \mu(\hat a') } \leq \sqrt{\frac{2\ln \tp{Nm}}{N_*}} + \frac{2}{N}.  
\end{align*}
When $a_0 = a^*$, we change the definition of $\Eps_2$ to be the event that $\abs{\hat \mu(a) - \mu(a)} \leq \sqrt{\frac{\ln \frac{ m}{\tau}}{2N(a)}}$ for every $a\in \+A$ and $\abs{\hat \mu'(a_0) - \mu(a_0)} \leq \sqrt{\frac{\ln \frac{m}{\tau}}{2(N(a_0) - k)}}$ for $a_0$ after the rollback of $k$ points. According to our assumption, $U$ is independent of the reward of $a_0$, thus $\hat \mu'(a_0)$ is unbiased with respect to $\mu(a_0)$. We could also applying Hoeffding's inequality to bound the probability 
\[
    \Pr{\abs{\hat \mu'(a_0) - \mu(a_0)} \leq \sqrt{\frac{\ln \frac{m}{\tau}}{2(N(a_0) - k)}}} \ge 1- \frac{2\tau}{ m} \ge 1-\tau,
\]
which implies that $\Pr{\Eps_1} \ge 1- 3\tau$ using the union bound. Therefore the upper bound satisfies
\begin{align*}
 \E{\mu(a^*) - \mu(\hat a')} 
     &\leq \E{ \mu(a^*) - \mu(\hat a') \mid \Eps} +  
     \Pr{\Eps^c}\cdot \E{\mu(a^*) - \mu(\hat a') \mid \Eps^c} \\ 
     &\leq \sqrt{\frac{2\ln \frac{ m}{\tau}}{N(a_0) - k}} + 4\tau = \sqrt{\frac{2\ln \tp{Nm}}{N(a_0) - k}} + \frac{4}{N}. 
\end{align*}
In conclusion, when $\gamma \ge \gamma_0$, the upper bound satisfies
\begin{align*}
 \E{\mu(a^*) - \mu(\hat a')} 
     &=  
     \begin{cases}
        O\tp{\sqrt{\frac{\ln \tp{N m} }{N(a_0) - k}}}, & a_0 = a^*, \\
        O\tp{\sqrt{\frac{\ln \tp{N m} }{N_*}}}, & a_0 \neq a^*.
    \end{cases}
\end{align*}
Then we obtain that 
\begin{align*}
    &\phantomeq SubOpt(\vN,N_*,k,\eps,\delta) = O\tp{\max\set{\sqrt{\frac{\ln \tp{N m} }{N(a_0) - k}}, \sqrt{\frac{\ln \tp{N m} }{N_*}} }}. 
\end{align*}  
\end{proof}

\paragraph{How to obtain $\gamma_0$}
The threshold $\gamma_0$ is optimized in the case $a_0 = a^*$, since the upper bound is fixed to be $O\tp{\sqrt{\frac{\ln \tp{N m} }{N(a^*)}}}$ when $a_0\neq a^*$. Let 
\[
    \sqrt{\frac{2\ln \frac{ m}{\tau}}{N(a_0)-k}} = \sqrt{\frac{2\ln \frac{ m}{\tau}}{N(a_0)}} + \frac{3k\gamma}{2\sqrt{2\pi} N(a_0)},
\]
and we could obtain that $\gamma = \frac{4\sqrt{\pi N(a_0)\ln \frac{m}{\tau} }}{3\sqrt{N(a_0)-k}\tp{\sqrt{N(a_0)-k} + \sqrt{N(a_0)}} } \approx \frac{4}{3}\sqrt{\frac{\pi \ln \frac{m}{\tau} }{N(a_0)-k}} = \gamma_0$.

We first give the lemma for binomial random variable $N(a^*)$.
\begin{lemma}[\textbf{Lemma 4} in \cite{RZMJ21}]
    \label{lem:binomial}
    With probability at least $1 - e^{-\frac{Nd(a^*)}{8}}$, one has
    \[
        N(a^*) \ge \frac{1}{2}Nd(a^*).
    \]
\end{lemma}

\begin{theorem}[\Cref{thm:dist_single_ub} restated]
Consider any offline dataset $D\sim \+D^N$, where the behavior policy satisfies $d(a^*) \ge \frac{1}{C^*}$ for some $C^* > 1$, any unlearning request $U$ from $a_0$. $\abs{U} = k$. When $N > 8C^*\ln N$ and $k \leq  N(a_0) - \sqrt{\frac{N(a_0) \ln \frac{m}{\tau}}{2}}$, by setting $\tau = \frac{1}{N}$, the arm $\hat a'$ returned by \Cref{alg:single} satisfies 
\begin{align*}
    SubOpt(N,C^*,k,\eps,\delta) = O\Bigg( \min\Bigg\{ \sqrt{\frac{C^*\ln \tp{N m} }{\max\set{1, N - 2kC^*}}}, \sqrt{\frac{C^*\ln \tp{N m} }{N}} + \frac{ k \gamma C^*}{N}  \Bigg \} \Bigg).
\end{align*}   
\end{theorem}

\begin{proof}
We first consider the case where $a_0 = a^*$. Denote $\Eps_1$ as the event that $N(a^*) \ge \frac{N}{2C^*}$. From \Cref{lem:binomial} we know that $\Pr{\Eps_1} \ge 1 - e^{-\frac{N}{8C^*}} \ge 1 - \tau$. Denote $\Eps_2$ as the event that $\abs{\hat \mu(a) - \mu(a)} \leq \sqrt{\frac{\ln \frac{ m}{\tau}}{2N(a)}}$ for all arms $a\in \+A$ and $\abs{\hat \mu'(a_0) - \mu(a_0)} \leq \sqrt{\frac{\ln \frac{m}{\tau}}{2(N(a_0) - k)}}$ for $a_0$ after the rollback of $k$ points. According to our assumption, $U$ is independent of the reward of $a_0$, thus $\hat \mu'(a_0)$ is unbiased with respect to $\mu(a_0)$. Applying Hoeffding's inequality for a fixed arm $a$, we have 
\[
    \Pr{\abs{\hat \mu(a) - \mu(a)} \leq \sqrt{\frac{\ln \frac{ m}{\tau}}{2N(a)}}} \ge 1 - \frac{2\tau}{ m},
\]
and
\[
    \Pr{\abs{\hat \mu'(a_0) - \mu(a_0)} \leq \sqrt{\frac{\ln \frac{m}{\tau}}{2(N(a_0) - k)}}} \ge 1- \frac{2\tau}{ m} \ge 1-\tau,
\]
which implies that $\Pr{\Eps_2} \ge 1- 3\tau$ using the union bound. Furthermore, regardless of the branch \Cref{alg:single} takes, we could always let $\Eps_3$ be the event that
$\abs{\nu} \leq \sigma\sqrt{2\ln\frac{ m}{\tau}} = t$, according to the property of Gaussian $\+M_d$, we have 
\[
    \Pr{\abs{\nu} \leq t} \ge 1 - 2e^{-\frac{t^2}{2\sigma^2}} = 1- \frac{2\tau}{ m},
\]
then we obtain $\Pr{\Eps_3} \ge 1- \tau$ since $m \ge 2$. Therefore under the condition that $\Eps = \Eps_1 \cap \Eps_2 \cap \Eps_3$ occurs, for every $a\in \+A$ and $\nu$, we have 
\[
    N(a_0) \ge \frac{N}{2C^*}, \quad \abs{\hat\mu(a) - \mu(a)} \leq b(a), \quad \abs{\hat \mu'(a_0) - \mu(a_0)} \leq \sqrt{\frac{\ln \frac{m}{\tau}}{2(N(a_0) - k)}}, \quad \abs{\nu} \leq \sigma\sqrt{2\ln\frac{ m}{\tau}}.  
\]
If the event $a_0 = \hat a$ and $\gamma < \gamma_0$ occurs, \Cref{alg:single} will add Gaussian noise sampled from $\+N(0,\sigma^2 )$. In view of the definition of $\hat a'$, for $\hat a' \neq a^*$, we have 
\[
    \hat\mu(a^*) - b(a^*) + \nu \leq \hat\mu(\hat a') - b(\hat a') 
    \Rightarrow \mu(a^*) \leq \mu(\hat a') + 2b(a^*) - \nu,
\] 
and this occurs only if $\nu < 0$. Then we have
\begin{align*}
    \E{\mu(a^*) - \mu(\hat a') \mid a_0=\hat a, \gamma< \gamma_0} 
    &\leq \sqrt{\frac{2\ln \frac{ m}{\tau}}{N(a^*)}} + \frac{1}{\Pr{\Eps_2}} \int_0^{\sigma\sqrt{2\ln\frac{ m}{\tau}}}\frac{x}{\sqrt{2\pi}\sigma}e^{-\frac{x^2}{2\sigma^2}}dx \\ 
    &\leq \sqrt{\frac{2\ln \frac{ m}{\tau}}{N(a_0)}} + \frac{\sigma}{\sqrt{2\pi}} = R_g.
\end{align*}
If either $a_0 \neq \hat a$ or $\gamma \ge \gamma_0$ occurs, \Cref{alg:single} executes rollback. Then the sub-optimality satisfies
\begin{align*}
    \E{\mu(a^*) - \mu(\hat a') \mid a_0=\hat a, \gamma \ge \gamma_0} 
    &\leq \sqrt{\frac{2\ln \frac{ m}{\tau}}{N(a_0) - k}} = R_r.
\end{align*}
According to the definition of $\gamma_0$, it is the threshold at which $R_g = R_r$; moreover, for $\gamma < \gamma_0$, we have $R_g < R_r$. Therefore, 
\[
    \E{\mu(a^*) - \mu(\hat a') \mid a_0=\hat a} = \1{R_g < R_r}\cdot R_g + \1{R_g \ge R_r}\cdot R_r = \min\set{R_r,R_g}. 
\]
If the event \(a_0 \neq \hat a\) occurs, then \Cref{alg:single} executes rollback on arm \(a_0=a^*\). Prior to rollback, \(\wh \vf_{(\hat a)}(D)\) is the largest coordinate among all \(a\in\mathcal A\). Moreover, rollback only modifies the coordinate corresponding to \(a^*\): for every \(a\neq a^*\), we have \(\wh \vf_{(a)}(D')=\wh \vf_{(a)}(D)\). Consequently, after rollback the selected arm \(\hat a'\) can only be either \(\hat a\) or \(a^*\). Therefore, the expected sub-optimality is no worse than that of the learning algorithm. By \Cref{lem:ub_LCB}, we have
\[
    \E{\mu(a^*) - \mu(\hat a') } \leq \sqrt{\frac{2\ln \frac{ m}{\tau}}{N(a^*)}} + \tau = \sqrt{\frac{2\ln \frac{ m}{\tau}}{N(a_0)}} + \tau.
\]
Therefore when $a_0=a^*$, the expected sub-optimality satisfies:
\begin{align*}
    \E{\mu(a^*) - \mu(\hat a') } 
     &\leq \E{ \mu(a^*) - \mu(\hat a') \mid \Eps} +  
     \Pr{\Eps^c}\cdot \E{\mu(a^*) - \mu(\hat a') \mid \Eps^c} \\ 
     &= \E{\tp{\mu(a^*) - \mu(\hat a')}\1{a_0 = \hat a} \mid \Eps} +  \E{\tp{\mu(a^*) - \mu(\hat a')}\1{a_0 \neq \hat a} \mid \Eps} + \sum_{j=1}^3 \Pr{\Eps_j^c} \\ 
     &\leq \Pr{a_0 = \hat a \mid \Eps}\cdot \min\set{R_r,R_g} + \Pr{a_0 \neq \hat a \mid \Eps}\cdot \tp{ \sqrt{\frac{2\ln \frac{ m}{\tau}}{N(a_0)}} + \tau } + 5\tau \\
     &\leq \min\set{R_r,R_g} + 5\tau \\ 
     &\leq \min\set{\sqrt{\frac{4C^*\ln \tp{N m}}{\max\set{1, N - 2kC^*}}}, \sqrt{\frac{4C^*\ln \tp{N m}}{N}} + \frac{3k\gamma C^*}{\sqrt{2\pi}N}} + \frac{5}{N}. 
\end{align*}

The second case is $a_0 \neq a^*$. If the event $a_0 = \hat a$ and $x=1$ occurs, \Cref{alg:single} will add Gaussian noise to $\wh \vf_{(a_0)}(D)$. However, this will not affect $\wh \vf_{(a)}(D)$, $a\neq a_0$. If $\nu \ge 0$, $\hat a'$ will still be $a_0$, the performance will be guaranteed by the learning algorithm. If $\nu \leq 0$, the performance will also be guaranteed by the learning algorithm since $\wh \vf_{(a)}(D)$ of every $a\neq a_0$ remains the same. 

If either $a_0 \neq \hat a$ or $x=0$ occurs, \Cref{alg:single} will execute rollback on $a_0$, and the expected sub-optimality will not be worse than that of the learning algorithm, thus in conclusion, when $a_0 \neq a^*$, the expected sub-optimality satisfies
\begin{align*}
 \E{\mu(a^*) - \mu(\hat a') } 
     &\leq  \E{\mu(a^*) - \mu(\hat a) } \leq 2\sqrt{\frac{C^*\ln \tp{Nm}}{N}} + \frac{2}{N}.  
\end{align*}
Therefore \Cref{alg:single} satisfies that 
\[
 \E{\mu(a^*) - \mu(\hat a')} =  
     \begin{cases}
        O\tp{\min\set{\sqrt{\frac{C^*\ln \tp{N m} }{N}} + \frac{ k \gamma C^*}{N}, \sqrt{\frac{C^*\ln \tp{N m} }{\max\set{1, N - 2kC^*}}} }}, & a_0 = a^*, \\
        O\tp{\sqrt{\frac{C^*\ln \tp{Nm}}{N}}}, & a_0 \neq a^*,
    \end{cases}
\]
which is
\begin{align*}
    SubOpt(N,C^*,k,\eps,\delta) = O\tp{\min\set{\sqrt{\frac{C^*\ln \tp{N m} }{N}} + \frac{ k \gamma C^*}{N}, \sqrt{\frac{C^*\ln \tp{N m} }{\max\set{1, N - 2kC^*}}} }}.
\end{align*}   
\end{proof}

\subsection{Lower Bounds for Single-Source Unlearning} \label{app:algo_single_lb}

\begin{theorem}[\Cref{thm:fixed_single_lb} restated] \label{thm:app_lb_fixed_single}
    Let $\eps \ge 0$, $\delta \leq \frac{1}{8\sqrt{e}}$. For any unlearning request $U$ from $a_0$, the lower bound of sub-optimality satisfies
    \[
        SubOpt(\vN,N_*,k,\eps,\delta) = \Omega\tp{e^{-\eps}\sqrt{\frac{1}{N(a_0)-k}}}.
    \]
\end{theorem}

\begin{proof}
To prove the lower bound, we could first relax the model by allowing that $\pi'$ has access to the entire dataset $D$ rather than $T(D)$. Since the point space $\+Z$ is bounded and providing additional information only helps. Moreover, because $\pi(D)$ is a deterministic function of $D$, it is without loss of generality to rewrite $\pi'(U,\pi(D),T(D))$ as $\pi'(U,D)$ and $\pi'(\emptyset,\pi(D\setminus U),T(D\setminus U))$ as $\pi'(\emptyset, D\setminus U)$. 

We construct two instances $I_1$ and $I_2$ to prove the lower bound, where the rewards are drawn from some Bernoulli distributions supported on $[0,1]$. Let $\vmu_{I_1}$ and $\vmu_{I_2}$ represent the vectors of arm mean rewards in $I_1$ and $I_2$. In particular,
\[
    \vmu_{I_1} = (\frac{1}{2}+\Delta, \frac{1}{2}), \quad \vmu_{I_2} = (\frac{1}{2}-\Delta, \frac{1}{2}).
\]
Here $0\leq \Delta \leq \frac{1}{5}$ and its value will be specified later. Let $D_1(\vN), D_2(\vN)$ denote the two datasets respectively, where $a_0=1$ and $\vN= (N(a_0), N-N(a_0))^\top$. In addition, the reward distribution underlying $D_1(\vN)$ is $\vmu_{I_1}$, while that underlying $D_2(\vN)$ is $\vmu_{I_2}$. Finally, let $U_1, U_2$ be the corresponding unlearning request from $D_1(\vN), D_2(\vN)$ respectively. 

Let $\E[D,U]{}$ represent the expectation conditioned on $D, U$. From the definition, for any algorithm $\pi,\pi'$, we have 
\begin{align}
    \E[D_1(\vN),U_1]{\mu(a^*) - \mu(\hat a')}  
    &= \Pr{\hat a' = a_2} \cdot \Delta \notag \\  
    &= \Pr{\pi'(U_1,D_1(\vN)) = a_2} \cdot \Delta  \notag  \\
    &\overset{(a)}{\ge} e^{-\eps}\tp{\Pr{\pi'(\emptyset, D_1(\vN)\setminus U_1) = a_2} -\delta} \cdot \Delta \notag  \\
    &= e^{-\eps} \Pr{\pi'(\emptyset, D_1(\vN)\setminus U_1) = a_2} \cdot \Delta - e^{-\eps}\delta  \Delta,
\end{align}
where $(a)$ is because $\pi$, $\pi'$ are $(\eps,\delta)$-UL. Similarly for $D_2, U_2$, we have
\begin{align}
    \E[D_2(\vN),U_2]{\mu(a^*) - \mu(\hat a')} \ge e^{-\eps} \Pr{\pi'(\emptyset, D_2(\vN)\setminus U_2) = a_1} \cdot \Delta - e^{-\eps}\delta  \Delta.
\end{align}
 
Next we generate $U_1$ by selecting $k$ points in the data points of $a_0$ from $D_1(\vN)$ uniformly and randomly, while $U_2$ is generated in the same way from $D_2(\vN)$. Therefore $U$ and the rewards of each point will be mutually independent. Denote $\vN'= (N(a_0)-k, N-N(a_0))^\top$ as the sample-count vector after unlearning request $U$. Then we apply the classic Le Cam’s method to give a lower bound of $SubOpt(\vN,N_*,k,\eps,\delta)$. Denote $\!{TV}(\cdot,\cdot)$ as the total variation distance between two distributions, $\!{KL}(\cdot \| \cdot)$ as the KL-divergence of two distributions. We have
\begin{align*}
    &\phantomeq SubOpt(\vN,N_*,k,\eps,\delta) \\
    &\ge \frac{1}{2} \tp{\E[D_1(\vN),U_1]{\mu(a^*) - \mu(\hat a')}  + \E[D_2(\vN),U_2]{\mu(a^*) - \mu(\hat a')}} \\
    &=\frac{e^{-\eps}\Delta}{2} \tp{\Pr{\pi'(\emptyset, D_1(\vN)\setminus U_1) = a_2} + \Pr{\pi'(\emptyset, D_2(\vN)\setminus U_2) = a_1}} - e^{-\eps}\delta\Delta  \\
    &\overset{(a)}{=} \frac{e^{-\eps}\Delta}{2} \tp{\Pr{\pi'(\emptyset, D_1(\vN')) = a_2} + \Pr{\pi'(\emptyset, D_2(\vN')) = a_1}} - e^{-\eps}\delta\Delta \\
    &\ge \frac{e^{-\eps}\Delta}{2} \tp{1 - \!{TV}\tp{\pi'(\emptyset, D_1(\vN')),\pi'(\emptyset, D_2(\vN'))}} - e^{-\eps}\delta\Delta \\
    &\overset{(b)}{\ge} \frac{e^{-\eps}\Delta}{4} \exp\tp{-\!{KL}\tp{\pi'(\emptyset, D_1(\vN')) \| \pi'(\emptyset, D_2(\vN'))}} - e^{-\eps}\delta\Delta  \\
    &\overset{(c)}{\ge} \frac{e^{-\eps}\Delta}{4} \exp\tp{-\!{KL}\tp{D_1(\vN') \| D_2(\vN')}} - e^{-\eps}\delta\Delta \\
    &\overset{(d)}{\ge} \frac{e^{-\eps}\Delta}{4} \exp\tp{-10\Delta^2 (N(a_0)-k)} - e^{-\eps}\delta\Delta,
\end{align*}
where $(a)$ is from $U$ and the reward being mutually independent, $(b)$ is from Bretagnolle–Huber inequality, $(c)$ is from the data processing inequality, $(d)$ is from direct calculation and the fact that $\ln \frac{1+2\Delta}{1-2\Delta} \leq 5\Delta$ when $0\leq \Delta \leq \frac{1}{5}$. By choosing $\Delta = \min\set{\frac{1}{5}, \frac{1}{2}\sqrt{\frac{1}{5(N(a_0)-k)}} }$, when $\delta \leq \frac{1}{8\sqrt{e}}$, we have  
\[
    SubOpt(\vN,N_*,k,\eps,\delta) \ge \frac{e^{-\eps}}{16}\sqrt{\frac{1}{5e(N(a_0)-k)}} = \Omega\tp{e^{-\eps}\sqrt{\frac{1}{N(a_0)-k}}}.
\]
\end{proof}

\begin{theorem}[\Cref{thm:dist_single_lb} restated]
Let $\eps \ge 0$, $\delta \leq \frac{1}{8\sqrt{e}}$. For any single-source unlearning request $U$, when $C^* \in [2,\infty)$, the lower bound of sub-optimality satisfies
\[
    SubOpt(N,C^*,k,\eps,\delta) = \Omega\tp{e^{-\eps}\sqrt{\frac{C^*}{N-k}}},
\]
while $C^*\in (1,2)$,
\[
    SubOpt(N,C^*,k,\eps,0) = \Omega\tp{\frac{2-C^*}{e^\eps}\exp\tp{-(N-k)(2-C^*)\ln\tp{\frac{2}{C^*-1}}} }.
\]
\end{theorem}

\begin{proof}
Following the proof of \Cref{thm:app_lb_fixed_single}, we relax the model and rewrite $\pi'(U,\pi(D),T(D))$ as $\pi'(U,D)$ and $\pi'(\emptyset,\pi(D\setminus U),T(D\setminus U))$ as $\pi'(\emptyset, D\setminus U)$. For $C^*\ge 2$, we construct two closely related instances $I_1$ and $I_2$ as in \Cref{thm:app_lb_fixed_single}, where $\vmu_{I_1}$ and $\vmu_{I_2}$ represent the vectors of arm mean rewards in $I_1$ and $I_2$. In particular,
\[
    \vmu_{I_1} = (\frac{1}{2}+\Delta, \frac{1}{2}), \quad \vmu_{I_2} = (\frac{1}{2}-\Delta, \frac{1}{2}).
\]
The only difference is the value of $\Delta$, which will be specified later. For the behavior policy $d$, let $d_{I_1}(1) = d_{I_2}(1) = \frac{1}{C^*}$, $d_{I_1}(2) = d_{I_2}(2) = 1 - \frac{1}{C^*}$. Since $C^* \ge 2$, we have $1-\frac{1}{C*} \ge \frac{1}{C^*}$ for $I_2$. Let $D_1(N,C^*)$ denote the dataset obtained by drawing $N$ i.i.d. samples from $d_{I_1} \otimes \vmu_{I_1}$ and $D_2(N,C^*)$ denote the dataset which draws $N$ i.i.d. samples from $d_{I_2} \otimes \vmu_{I_2}$. Finally, let $U_1, U_2$ be the unlearning request from $D_1(N,C^*), D_2(N,C^*)$ respectively.

From the definition, for any algorithm $\pi,\pi'$, we have 
\begin{align}
    \E[D_1(N,C^*),U_1]{\mu(a^*) - \mu(\hat a')}  
    &= \Pr{\hat a' = a_2} \cdot \Delta \notag \\  
    &= \Pr{\pi'(U_1, D_1(N,C^*)) = a_2} \cdot \Delta  \notag  \\
    &\overset{(a)}{\ge} e^{-\eps}\tp{\Pr{\pi'(\emptyset, D_1(N,C^*)\setminus U_1) = a_2} -\delta} \cdot \Delta \notag  \\
    &= e^{-\eps} \Pr{\pi'(\emptyset, D_1(N,C^*)\setminus U_1) = a_2} \cdot \Delta - e^{-\eps}\delta  \Delta,
\end{align}
where $(a)$ comes from $\pi$, $\pi'$ are $(\eps,\delta)$-UL. Similarly for $D_2, U_2$, we have
\begin{align}
    \E[D_2(N,C^*),U_2]{\mu(a^*) - \mu(\hat a')} \ge e^{-\eps} \Pr{\pi'(\emptyset, D_2(N,C^*)\setminus U_2) = a_1} \cdot \Delta - e^{-\eps}\delta  \Delta.
\end{align}
 
Next we generate $U_1$ by selecting $k$ points in $D_1$ uniformly and randomly, while $U_2$ is generated in the same way from $D_2$. We have the following lemma:
\begin{lemma} \label{lem:dist_equiv}
    $\forall i = 1,2$, $D_i(N-k, C^*)$ and $D_i(N,C^*)\setminus U_i$ follows the same distribution.
\end{lemma}

\begin{proof}
It is sufficient to prove the case of $i=1$, while $i=2$ follows the same. We need to prove that for any $N$ arms $\set{X_1,\cdots,X_N}$ sampled from $d_{I_1}$, after drawing a random set $S=\set{i_1 < i_2 < \cdots < i_{N-k}}$ from $[N]$ uniformly and randomly, the distribution of the remaining set $\set{X_{i_1},\cdots,X_{i_{N-k}}}$ is the same as drawing $N-k$ samples $\set{Y_1,\cdots,Y_{N-k}}$ from $d_{I_1}$. 

Fix arbitrary values $a_1,\cdots,a_{N-k} \in \mathcal{A}$. We need to compute
\[
    \Pr{X_{i_1} = a_1,\cdots,X_{i_{N-k}} = a_{N-k}}.
\]
By the law of total probability with respect to the random index set $S$, we have
\begin{align}
    \Pr{X_{i_1} = a_1,\dots,X_{i_{N-k}} = a_{N-k}} 
    &= \sum_{S\in [N]} \Pr{S}\cdot \Pr{X_{i_1} = a_1,\dots,X_{i_{N-k}} = a_{N-k}\mid S} \notag \\
    &\overset{(a)}{=} \sum_{S_0\in [N]} \frac{1}{\binom{N}{N-k}} \Pr{X_{i_1} = a_1,\dots,X_{i_{N-k}} = a_{N-k}\mid S_0} \label{line:cond_prob_fixed},
\end{align}
where $(a)$ is because $S$ is chosen uniformly from $[N]$ and independently of $\set{X_1,\dots,X_N}$.

For a fixed $S_0 = \{i_1<\dots<i_{N-k}\}$, the randomness of $X_{i_1},\cdots,X_{i_{N-k}}$ only comes from $d_{I_1}$, thus
\[
    \Pr{X_{i_1} = a_1,\dots,X_{i_{N-k}} = a_{N-k}\mid S_0} = \prod_{i=1}^{N-k} d(a_i).
\]
Bring this back to \Cref{line:cond_prob_fixed}, we obtain that
\begin{align*}
    \Pr{X_{i_1} = a_1,\dots,X_{i_{N-k}} = a_{N-k}} 
    =\sum_{S_0\in [N]} \frac{1}{\binom{N}{N-k}} \prod_{i=1}^{N-k} d(a_i) 
    =\prod_{i=1}^{N-k} d(a_i) 
    =\Pr{Y_1 = a_1,\dots,Y_{N-k} = a_{N-k}}, 
\end{align*}
which is exactly the joint mass function of $(N-k)$ i.i.d. samples from common law $d$.     
\end{proof}
Similarly we apply Le Cam’s method to give a lower bound of $SubOpt(N,C^*,k,\eps,\delta)$. We have
\begin{align*}
    SubOpt(N,C^*,k,\eps,\delta) 
    &\ge \frac{1}{2} \tp{\E[D_1(N,C^*),U_1]{\mu(a^*) - \mu(\hat a')}  + \E[D_2(N,C^*),U_2]{\mu(a^*) - \mu(\hat a')}} \\
    &=\frac{e^{-\eps}\Delta}{2} \tp{\Pr{\pi'(\emptyset, D_1(N,C^*) \setminus U_1) = a_2} + \Pr{\pi'(\emptyset, D_2(N,C^*) \setminus U_2) = a_1}} - e^{-\eps}\delta\Delta  \\
    &\overset{(a)}{=} \frac{e^{-\eps}\Delta}{2} \tp{\Pr{\pi'(\emptyset, D_1(N-k,C^*)) = a_2} + \Pr{\pi'(\emptyset, D_2(N-k,C^*)) = a_1}} - e^{-\eps}\delta\Delta \\
    &\ge \frac{e^{-\eps}\Delta}{2} \tp{1 - \!{TV}\tp{\pi'(D_1(N-k,C^*),\pi'(D_2(N-k,C^*)}} - e^{-\eps}\delta\Delta \\
    &\overset{(b)}{\ge} \frac{e^{-\eps}\Delta}{4} \exp\tp{-\!{KL}\tp{\pi'(D_1(N-k,C^*) \| \pi'(D_2(N-k,C^*)}} - e^{-\eps}\delta\Delta  \\
    &\overset{(c)}{\ge} \frac{e^{-\eps}\Delta}{4} \exp\tp{-\!{KL}\tp{D_1(N-k,C^*) \| D_2(N-k,C^*)}} - e^{-\eps}\delta\Delta \\
    &\overset{(d)}{\ge} \frac{e^{-\eps}\Delta}{4} \exp\tp{-\frac{4\Delta^2 (N-k)}{C^*}} - e^{-\eps}\delta\Delta,
\end{align*}
where $(a)$ is from \Cref{lem:dist_equiv}, $(b)$ is from Bretagnolle–Huber inequality, $(c)$ is from the data processing inequality, $(d)$ is from direct calculation and the fact that $\ln \frac{1+2\Delta}{1-2\Delta} \leq 5\Delta$ when $0\leq \Delta \leq \frac{1}{5}$. By choosing $\Delta = \min\set{\frac{1}{5}, \frac{1}{2}\sqrt{\frac{C^*}{5(N-k)}} }$, when $\delta < \frac{1}{8\sqrt{e}}$, we have  
\[
    SubOpt(N,C^*,k,\eps,\delta) \ge \frac{e^{-\eps}}{16}\sqrt{\frac{C^*}{5e(N-k)}} = \Omega\tp{e^{-\eps}\sqrt{\frac{C^*}{N-k}}}.
\]
When $C^* \in (1,2)$, the proof is similar to the case $C^* \in [2,\infty)$. We construct new instances that are different in both the reward distributions as well as the behavior policy. More specifically,
\[
    \vmu_{I_1} = (\frac{1}{2}+\Delta, \frac{1}{2}), \quad \vmu_{I_2} = (\frac{1}{2}, \frac{1}{2}+\Delta).
\]
Here we set $\Delta = \frac{2-C^*}{2}$. For the behavior policy $d$ on $\+A$, $d_{I_1}(1) = d_{I_2}(2) = \frac{1}{C^*}$, $d_{I_1}(2) = d_{I_2}(1) = 1 - \frac{1}{C^*}$. Therefore in both datasets $a^*$ satisfy $d(a^*) \ge \frac{1}{C^*}$. The notations $D_1(N,C^*), D_2(N,C^*), U_1, U_2$ are defined analogously. Moreover, we generate $U_1$ and $U_2$ in the same manner as in the case $C^*\ge 2$, from $D_1(N,C^*)$ and $D_2(N,C^*)$ respectively. And we could also establish \Cref{lem:dist_equiv} for $D_i(N-k,C^*)$ and $D_i(N,C^*)\setminus U_i$, $\forall i = 1,2$. Similarly we have 
\begin{align}
    SubOpt(N,C^*,k,\eps,0) 
    &\ge \frac{1}{2} \tp{\E[D_1(N,C^*),U_1]{\mu(a^*) - \mu(\hat a')}  + \E[D_2(N,C^*),U_2]{\mu(a^*) - \mu(\hat a')}} \notag \\
    &=\frac{e^{-\eps}\Delta}{2} \tp{\Pr{\pi'(\emptyset, D_1(N,C^*) \setminus U_1) = a_2} + \Pr{\pi'(\emptyset, D_2(N,C^*) \setminus U_2) = a_1}}  \notag \\
    &\ge \frac{e^{-\eps}\Delta}{2} \tp{1 - \!{TV}\tp{\pi'(D_1(N-k,C^*),\pi'(D_2(N-k,C^*)}}  \notag \\
    &\ge \frac{e^{-\eps}\Delta}{4} \exp\tp{-\!{KL}\tp{D_1(N-k,C^*) \| D_2(N-k,C^*)}}  \label{line:KLdiv_known}.
\end{align}
From direct calculation, we have
\begin{align*}
    \!{KL}\tp{D_1(N-k,C^*) \| D_2(N-k,C^*)} 
    &= (N-k)\cdot \!{KL}\tp{d_{I_1} \otimes \vmu_{I_1} \| d_{I_2} \otimes \vmu_{I_2}} \\
    &= (N-k) \tp{\!{KL}\tp{d_{I_1} \| d_{I_2}} + \sum_{a=1}^2 d_{I_1}(a)\cdot \!{KL}\tp{\vmu_{I_1}(a) \| \vmu_{I_2}(a)}} \\
    &= (N-k) \cdot \tp{ \tp{\frac{1+\Delta}{C^*}-\frac{1}{2}} \ln\tp{\frac{1+2\Delta}{C^*-1}} + \tp{\frac{1-\Delta}{C^*}-\frac{1}{2}} \ln \tp{\frac{1-2\Delta}{C^*-1}} } \\
    &= \frac{(N-k)(2-C^*)}{C^*} \ln\tp{\frac{2}{C^*-1}}. 
\end{align*}
Bring this result back to \Cref{line:KLdiv_known}, we obtain the conclusion
\begin{align*}
    SubOpt(N,C^*,k,\eps,0) 
    &\ge \frac{2-C^*}{8}\cdot e^{-\eps-\frac{(N-k)(2-C^*)}{C^*} \ln\tp{\frac{2}{C^*-1}}} \\
    &= \Omega\tp{(2-C^*)e^{-\eps-(N-k)(2-C^*)\ln\tp{\frac{2}{C^*-1}}} }.
\end{align*}
\end{proof}

\subsection{The Mixing Algorithm} \label{app:algo_fixed_single_mixing}
In this section, we study the mixing algorithm and investigate whether it can outperform the two base algorithms: Gaussian mechanism or rollback. For clarity of exposition, we present the algorithm under $\+M_f$ (we could obtain similar conclusion under $\+M_d$). When $a_0 \neq a^*$, rollback is always preferable, since it does not degrade the performance of the underlying learning algorithm. Accordingly, \Cref{alg:fixed_single_mixing} retains the initial check $\hat a = a_0$. 

Let $U'\subseteq U$ denote the subset of points that are first handled via rollback, and write $\abs{U'} = k'$. After rolling back $U'$, the algorithm adds Gaussian noise to the LCB vector with $\sigma = \frac{3(k-k')}{2(N(a_0) - k')}$. The parameter $k'$ interpolates between the two extremes: $k' = k$ recovers pure rollback, whereas $k'=0$ recovers the pure Gaussian-noise mechanism.

\begin{algorithm}[H]
\caption{Single-source unlearning mixing algorithm under $\+M_f$}
\label{alg:fixed_single_mixing}
{\bfseries Input:} Output of $\pi(D)$: $\hat a$, additional statistics $T(D)$: $\wh \vf(D)$, $\wh \vmu$, $\wh \vN$, unlearning request: $U$ from $a_0$, $\abs{U} = k$, a confidence level: $\tau \in (0,1)$, mixing parameter: $k' \in [0,k]$  
\begin{algorithmic}[1]
    \STATE Set $\wh \vf(D') \gets \wh \vf(D)$
    \IF{$\hat a \neq a_0$}
        \STATE $k' \gets k$
    \ENDIF
    \STATE Set $N'(a_0) \gets N(a_0) - k'$ and choose $k'$ data points in $U$ uniformly and randomly, denote them as $U'$
    \STATE Compute the new empirical mean reward  $\hat \mu'(a_0) \gets \frac{\hat\mu(a_0)\cdot N(a_0) - \sum_{(a_0,r_i) \in U'} r_i}{N'(a_0)}$
    \STATE Compute the new penalty $b'(a_0) \gets \sqrt{\frac{\ln \frac{ m}{\tau}}{2N'(a_0)}}$ 
    \STATE Set $\wh \vf_{(a_0)}(D') \gets \hat \mu'(a_0) - b'(a_0)$ 
    \STATE Set $\Delta_{\vf} \gets \frac{3(k-k')}{2(N(a_0) - k')}$, $\sigma \gets \Delta_{\vf}\gamma$
    \STATE Sample $\nu \in \bb R$ from $\+N(0,\sigma^2)$   
    \STATE Set $\wh \vf_{(a_0)}(D') \gets \wh \vf_{(a_0)}(D') + \nu$ \\
\end{algorithmic}
{\bfseries Output:} $\hat a' = \argmax_{a} \wh \vf_{(a)}(D')$ 
\end{algorithm}

\begin{algorithm}[H]
\caption{Single-source unlearning mixing algorithm under $\+M_f$ with input $\emptyset$}
{\bfseries Input:} Output of $\pi(D')$, additional statistics $T(D')$: $\pi(D) = \hat a$, $a_0$, $k$, $N(a_0)$, $\wh \vf(D')$, a confidence level: $\tau \in (0,1)$, mixing parameter: $k' \in [0,k]$    
\begin{algorithmic}[1]
    \IF{$\hat a \neq a_0$}
        \STATE $k' \gets k$
    \ENDIF
    \STATE Set $\Delta_{\vf} \gets \frac{3(k-k')}{2(N(a_0)-k')}$, $\sigma \gets \Delta_{\vf}\gamma$ 
    \STATE Sample $\nu \in \bb R$ from $\+N(0,\sigma^2)$   
    \STATE Set $\wh \vf_{(a_0)}(D') \gets \wh \vf_{(a_0)}(D') + \nu$ 
\end{algorithmic}
{\bfseries Output:} $\hat a' = \argmax_{a} \wh \vf_{(a)}(D')$ 
\end{algorithm}

\begin{theorem} 
    \Cref{alg:fixed_single_mixing} and \Cref{alg:LCB} are $(\eps,\delta)$-unlearning.
\end{theorem}

We first calculate the $\ell_2$ sensitivity of $\vf$. After the unlearning of $U'$, $\wh \vf_{(a_0)}(D')$ becomes $\hat \mu'(a_0) - b'(a_0)$ where $N'(a_0)= N(a_0)-k'$, $\hat \mu'(a_0) = \frac{\hat \mu(a_0)\cdot N(a_0) - \sum_{(a_0,r_i)\in U'} r_i}{N'(a_0)}$ and $b'(a_0) = \sqrt{\frac{\ln \frac{m}{\tau}}{2N'(a_0)}}$. For $a\neq a_0$, $\wh \vf_{(a)}(D') = \wh \vf_{(a)}(D)$. Now we calculate $\Delta_{\vf}$ when the dataset changes from $D\setminus U'$ to $D\setminus U = D'$:

\begin{align*}
    \Delta_{\vf} 
    &= \max_{D\setminus U', D'} \norm{f(D\setminus U')-\vf(D')}_2 \\
    &= \max_{D\setminus U', D'}
    \abs{\hat\mu'(a_0) - \frac{\hat\mu'(a_0)\cdot \tp{N(a_0)-k'} - \sum_{(a_0,r_i)\in U\setminus U'} r_i}{N(a_0)-k} + \sqrt{\ln \frac{m}{\tau}} \tp{\sqrt{\frac{1}{2(N(a_0)-k)}} - \sqrt{\frac{1}{2(N(a_0) - k')}}} } \\
    &\leq \max_{D\setminus U', D'}
    \abs{\frac{\tp{N(a_0) - k'}\sum_{(a_0,r_i)\in U\setminus U'} r_i - \tp{k-k'} \sum_{(a_0,r_i)\in D\setminus U'} r_i }{\tp{N(a_0)-k'} \tp{N(a_0)-k}} + \frac{(k-k') \sqrt{\ln \frac{m}{\tau}}}{2\tp{N(a_0)-k}\sqrt{2(N(a_0) - k')} } } \\
    &\overset{(a)}{=} \abs{\frac{k-k'}{N(a_0) - k'} +  \frac{(k-k') \sqrt{\ln \frac{m}{\tau}}}{2\tp{N(a_0)-k}\sqrt{2(N(a_0)-k')} }  } \overset{(b)}{\leq} \frac{3(k-k')}{2(N(a_0) - k')},
\end{align*}
where $(a)$ comes from choosing $r_i = 1, \forall (a_0,r_i)\in U\setminus U'$ and $r_j = 0, \forall (a_0,r_j)\in D\setminus U$, $(b)$ comes from $k \leq  N(a_0) - \sqrt{\frac{N(a_0)\ln \frac{m}{\tau}}{2}}$. Thus, by following the same proof as \Cref{thm:unlearn_fixed_single}, we conclude that \Cref{alg:fixed_single_mixing} and \Cref{alg:LCB} satisfy \((\eps,\delta)\)-unlearning.

\begin{theorem}[$\+M_f$] \label{thm:ub_fixed_single_mixing} 
Consider any offline dataset $D(\vN)$, any unlearning request $U$ from $a_0$. $N(a^*)\ge N_*, \abs{U} = k$. When $k \leq  N(a_0) - \sqrt{\frac{N(a_0) \ln \frac{m}{\tau}}{2}}$, by setting $\tau = \frac{1}{N}$, the arm $\hat a'$ returned by \Cref{alg:fixed_single_mixing} satisfies
\begin{align*}
    SubOpt(\vN,N_*,k,\eps,\delta) = O\tp{\max\set{\sqrt{\frac{\ln \tp{Nm}}{N(a_0) - k'}} +  \frac{(k-k') \gamma \sqrt{\ln \tp{Nm}}}{N(a_0) - k'}, \sqrt{\frac{\ln \tp{N m} }{N_*}} }}. 
\end{align*}    
\end{theorem}

\begin{proof}
We first consider the case where $a_0 = a^*$. Denote $\Eps_1$ as the event that $\abs{\hat \mu(a) - \mu(a)} \leq \sqrt{\frac{\ln \frac{ m}{\tau}}{2N(a)}}$ for every $a\in \+A$ and $\abs{\hat \mu'(a_0) - \mu(a_0)} \leq \sqrt{\frac{\ln \frac{m}{\tau}}{2(N(a_0) - k')}}$ for $a_0$ after the rollback of $k'$ points. Applying Hoeffding's inequality for a fixed arm $a$, we have 
\[
    \Pr{\abs{\hat \mu(a) - \mu(a)} \leq \sqrt{\frac{\ln \frac{ m}{\tau}}{2N(a)}}} \ge 1-  \frac{2\tau}{ m}.
\]
According to our assumption, $U$ is independent of the reward of $a_0$ and $U'$ is drawn from $U$ uniformly and randomly, thus $\hat \mu'(a_0)$ is unbiased regard to $\mu(a_0)$. We could also applying Hoeffding's inequality to bound the probability 
\[
    \Pr{\abs{\hat \mu'(a_0) - \mu(a_0)} \leq \sqrt{\frac{\ln \frac{m}{\tau}}{2(N(a_0) - k')}}} \ge 1- \frac{2\tau}{ m} \ge 1-\tau,
\]
which implies that $\Pr{\Eps_1} \ge 1- 3\tau$ using the union bound. Furthermore, regardless of whether $a_0 = \hat a$ happens or not, we could always let $\Eps_2$ be the event that
$\abs{\nu} \leq \sigma\sqrt{2\ln\frac{ m}{\tau}} = t$, according to the property of Gaussian $\+M_d$, we have 
\[
    \Pr{\abs{\nu} \leq t} \ge 1 - 2e^{-\frac{t^2}{2\sigma^2}} = 1- \frac{2\tau}{ m},
\]
then we obtain $\Pr{\Eps_2} \ge 1- \tau$. Therefore under the condition that $\Eps = \Eps_1 \cap \Eps_2$ happens, for every $a\in \+A$ and $\nu$, we have 
\[
    \abs{\hat\mu(a) - \mu(a)} \leq b(a), \quad \abs{\hat\mu'(a_0) - \mu(a_0)} \leq b'(a_0), \quad \abs{\nu} \leq \sigma\sqrt{2\ln\frac{ m}{\tau}}.  
\]
If the event $a_0 = \hat a$ happens, \Cref{alg:fixed_single_mixing} will first execute rollback of $k'$ points and then add Gaussian noise sampled from $\+N(0,\sigma^2)$. In view of the definition of $\hat a'$, for $\hat a' \neq a^*$, we have 
\[
    \hat\mu'(a^*) - b'(a^*) + \nu \leq \hat\mu(\hat a') - b(\hat a') 
    \Rightarrow \mu(a^*) \leq \mu(\hat a') + 2b'(a^*) - \nu.
\]
and this occurs only if $\nu < 0$. Then we have
\begin{align*}
    \E{\mu(a^*) - \mu(\hat a') \mid a_0 = \hat a} 
    &\leq \sqrt{\frac{2\ln \frac{ m}{\tau}}{N(a_0) - k'}} + \frac{\sigma}{\sqrt{2\pi}}.
\end{align*}
If the event $a_0 \neq \hat a$ happens, \Cref{alg:fixed_single_mixing} will execute rollback on $a_0 = a^*$. Before the rollback, we know that $\wh \vf_{(\hat a)}(D)$ is the largest among all $a\in \+A$. Since $\wh \vf_{(a)}(D') = \wh \vf_{(a)}(D)$ for every $a\neq a^*$, $\hat a'$ is either $\hat a$ or $a^*$ after the rollback. Then the expected sub-optimality will not be worse then the learning algorithm, and
\[
    \E{\mu(a^*) - \mu(\hat a) } \leq \sqrt{\frac{2\ln \frac{ m}{\tau}}{N(a^*)}} + 2\tau = \sqrt{\frac{2\ln \frac{ m}{\tau}}{N(a_0)}} + 2\tau.
\]
In conclusion, the expected sub-optimality satisfies:
\begin{align*}
    \E{\mu(a^*) - \mu(\hat a') } 
     &\leq \E{ \mu(a^*) - \mu(\hat a') \mid \Eps} +  
     \Pr{\Eps^c}\cdot \E{\mu(a^*) - \mu(\hat a') \mid \Eps^c} \\ 
     &= \E{\tp{\mu(a^*) - \mu(\hat a')}\1{a_0 = \hat a} \mid \Eps} +  \E{\tp{\mu(a^*) - \mu(\hat a')}\1{a_0 \neq \hat a} \mid \Eps} + \sum_{j=1}^2 \Pr{\Eps_j^c} \\ 
     &\leq \Pr{a_0 = \hat a \mid \Eps}\cdot \tp{ \sqrt{\frac{2\ln \frac{ m}{\tau}}{N(a_0) - k'}} + \frac{\sigma}{\sqrt{2\pi}} } + \Pr{a_0 \neq \hat a \mid \Eps}\cdot \tp{\sqrt{\frac{2\ln \frac{ m}{\tau}}{N(a_0)}} + \tau} + 4\tau \\
     &\leq \sqrt{\frac{2\ln \frac{ m}{\tau}}{N(a_0) - k'}} + \frac{\sigma}{\sqrt{2\pi}} + 4\tau \\ 
     &\leq \sqrt{\frac{2\ln \tp{Nm}}{N(a_0) - k'}} +  \frac{3(k-k') \gamma} {2\sqrt{2\pi}(N(a_0) - k')} + \frac{4}{N}. 
\end{align*}

The second case is $a_0 \neq a^*$. If the event $a_0 = \hat a$ occurs, then \Cref{alg:fixed_single_mixing} first rolls back $k'$ points and subsequently adds Gaussian noise to $\wh \vf_{(a_0)}(D)$. This operation does not affect $\hat \vf_{(a)}(D)$ for any $a\neq a_0$. If the noise realization satisfies $\nu \ge 0$, then $\hat a'$ remains $a_0$, and the performance is guaranteed by the learning algorithm. If $\nu \le 0$, the performance is still guaranteed, since all coordinates $ \vf_{(a)}(D)$ for $a\neq a_0$ remain unchanged. 

If instead the event $a_0\neq \hat a$ occurs, then \Cref{alg:fixed_single_mixing} executes rollback on arm $a_0$, and the expected sub-optimality is no worse than that of the learning algorithm. Therefore, when $a_0\neq a^*$, the expected sub-optimality under \Cref{alg:fixed_single_mixing} is no worse than under the learning algorithm, and thus
\[
    \E{\mu(a^*) - \mu(\hat a') } \leq  \E{\mu(a^*) - \mu(\hat a) } \leq \sqrt{\frac{2\ln \frac{ m}{\tau}}{N(a^*)}} + \tau \leq \sqrt{\frac{2\ln \frac{ m}{\tau}}{N_*}} + \tau.
\]
In conclusion, the upper bound satisfies
\begin{align*}
 \E{\mu(a^*) - \mu(\hat a')} 
     &=  
     \begin{cases}
        O\tp{\sqrt{\frac{\ln \tp{Nm}}{N(a_0) - k'}} +  \frac{(k-k') \gamma}{N(a_0) - k'}}, & a_0 = a^*, \\
        O\tp{\sqrt{\frac{\ln \tp{N m} }{N_*}} }, & a_0 \neq a^*.
    \end{cases}
\end{align*}
\end{proof}

\paragraph{Comparison}
We optimize the upper bound when $a_0 = a^*$:
\[
    \E{\mu(a^*) - \mu(\hat a')} \leq \sqrt{\frac{2\ln \tp{Nm}}{N(a_0) - k'}} +  \frac{3(k-k') \gamma} {2\sqrt{2\pi}(N(a_0) - k')} + \frac{4}{N}.
\]
Let $f(x) = 2\sqrt{\frac{\pi \ln \tp{Nm}}{N(a_0) - x}} +  \frac{3(k-x) \gamma}{2(N(a_0) - x)}, 0\leq x\leq k$. Then we calculate the derivative:
\[
    f'(x) = \frac{\sqrt{\pi\ln \tp{Nm} \tp{N(a_0)-x}} - 3\gamma(N(a_0)-k)}{(N(a_0)-x)^2}.
\]
Since the denominator is positive, the numerator is monotonically decreasing as $x$ increases, then the minimum point is obtained at $k' = 0$ or $k' = k$.

\subsection{Improved Results for $C^* \in (1,2)$} \label{app:algo_dist_single_le2}

\begin{algorithm}[H]
\caption{Single-source unlearning algorithm under $\+M_d$ ($C^*\in (1,2)$)}
\label{alg:dist_single_le2}
{\bfseries Input:} Output of $\hat a = \pi(D)$, additional statistics: $\wh \vN$, unlearning request $U$ from $a_0$, $\abs{U} = k$
\begin{algorithmic}[1]
    \STATE Set $\wh \vN' \gets \wh \vN$, $\wh \vN'_{(a_0)} \gets \wh \vN_{(a_0)} - k$ \\
\end{algorithmic}
{\bfseries Output:} $\hat a' = \argmax_{a} \wh \vN'_{(a)}$ 
\end{algorithm}

When the input of unlearned request is $\emptyset$, \Cref{alg:dist_single_le2} simply outputs the learned arm $\pi(D\setminus U)$, which is the arm with the most sample counts computed from $D\setminus U$. Then we prove the following theorem.

\begin{theorem} \label{thm:unlearn_dist_single_le2}
    \Cref{alg:dist_single_le2} and imitation learning are $(0,0)$-unlearning.
\end{theorem}

\begin{proof}
When the inputs are \(\emptyset,\,\pi(D\setminus U),\,T(D\setminus U)\), denote $\wh \vN_1'$ as the sample-count vector at the output step. Similarly, let $\wh \vN_2'$ denote the corresponding vector when the inputs are \(U,\,\pi(D),\,T(D)\). In either cases, the algorithm outputs the same vector, i.e. $\wh \vN_1' = \wh \vN_2' = \wh \vN'$, where $\wh \vN'$ is the sample-count vector of $D\setminus U$. Viewing $g(\wh \vN) = \argmax_{a}\wh \vN_{(a)}$ as a post-processing map, we obtain that \Cref{alg:dist_single_le2} and the imitation learning rule are $(0,0)$-unlearning.
\end{proof}

\begin{theorem}[\Cref{thm:dist_single_le2_ub} restated]
Consider any offline dataset $D\sim \+D^N$, where $d(a^*) \ge \frac{1}{C^*}$ for some $C^* \in (1,2)$, any unlearning request $U$ from $a_0$. $\abs{U} = k$. When $k \leq min\set{\frac{3N}{4}, \frac{(2-C^*)N}{C^*}}$, the arm $\hat a'$ returned by \Cref{alg:dist_single_le2} satisfies
\begin{align*}
    SubOpt(N,C^*,k,\eps,0) = O\tp{e^{-\frac{(N-k)}{2}\ln \frac{C^*}{8(C^*-1)} + \frac{(N+k)}{2}\ln \frac{2}{C^*} }}.
\end{align*}  
\end{theorem}

\begin{proof}
When $a_0 = a^*$, we bound $\E{\mu(a^*) - \mu(\hat a')}$ as follows,
\begin{align*}
    \E{\mu(a^*) - \mu(\hat a')} 
    &\leq \Pr{\hat a' \neq a^*} \\
    &\leq \Pr{\exists a \neq a^*, N(a) \ge N(a^*) - k}  \\
    &\leq \Pr{N - N(a^*) \ge N(a^*) - k} \\
    &= \Pr{N(a^*) \leq \frac{N+k}{2}}.
\end{align*}
Since $k \leq \frac{(2-C^*)N}{C^*}$, we have that $\frac{N+k}{2} \leq \frac{N}{C^*}$. Then applying Chernoff's bound for binomial random variables, we have 
\begin{align*}
    \Pr{N(a^*) \leq \frac{N+k}{2}} 
    &\leq \exp\tp{-N\cdot \!{KL}\tp{\Ber{\frac{N+k}{2N}} \, \Big\Vert \, \Ber{\frac{1}{C^*}}}} \\
    &= \exp\tp{-N\tp{\frac{N+k}{2N}\ln \frac{(N+k)C^*}{2N} + \frac{N-k}{2N}\ln \frac{(N-k)C^*}{2N(C^*-1)}} } \\
    &\overset{(a)}{\leq} \exp\tp{\frac{N+k}{2}\ln \frac{2}{C^*} - \frac{N-k}{2}\ln \frac{C^*}{8(C^*-1)}},
\end{align*}
where $(a)$ is from $k\leq \frac{3N}{4}$. When $C^* \rightarrow 1$, the second term $\frac{N-k}{2}\ln \tp{\frac{C^*}{8(C^*-1)}}$ dominates, and the resulting probability is on the order of $\tp{C^*-1}^{N-k}$, which matches the lower bound in \Cref{thm:dist_single_lb} in its dependence on $N,C^*,k$.

When $a_0 \neq a^*$, we know that 
\begin{align*}
    \E{\mu(a^*) - \mu(\hat a')} \leq \Pr{N - N(a^*) -k \ge N(a^*)} 
    \leq \Pr{N(a^*) \leq \frac{N-k}{2}} \leq \Pr{N(a^*) \leq \frac{N+k}{2}}.
\end{align*}
Therefore this bound is smaller than  the corresponding bound in the case $a_0 = a^*$. In conclusion, 
\begin{align*}
    SubOpt(N,C^*,k,\eps,0) = O\tp{e^{-\frac{(N-k)}{2}\ln \frac{C^*}{8(C^*-1)} + \frac{(N+k)}{2}\ln \frac{2}{C^*} }}.
\end{align*}  
\end{proof}

\subsection{Multi-Source Unlearning under $\+M_f$} \label{app:algo_fixed_multi}

\begin{algorithm}[H]
\caption{Multi-source unlearning algorithm under $\+M_f$}
 \label{alg:fixed_multi}
{\bfseries Input:} Output of $\pi(D)$: $\hat a$, additional statistics $T(D)$: $\wh \vf(D)$, $\wh \vmu$, $\wh \vN$, unlearning request: $U = \bigcup_{i=1}^\ell U_i$, $U_i$ is from $a_{u_i}$, $\abs{U_i} = k_i$, a confidence level: $\tau \in (0,1)$ 
\begin{algorithmic}[1]
    \STATE Set $\wh \vf(D') \gets \wh \vf(D)$
    \FOR{$i=1$ to $\ell$}
        \IF{($\hat a = a_{u_i}$)}
            \STATE $\gamma_0 \gets \frac{4}{3}\sqrt{\frac{\pi \ln \frac{m}{\tau} }{N_{\min}-k_{\max}}}$
            \IF{($\gamma < \gamma_0$)}
                \STATE Set $\Delta_{\vf} \gets \frac{3k_i}{2N(a_{u_i})}$, $\sigma \gets \Delta_{\vf}\gamma$
                \STATE Sample $\nu \in \bb R$ from $\+N(0,\sigma^2)$   
                \STATE Set $\wh \vf_{(a_{u_i})}(D') \gets \wh \vf_{(a_{u_i})}(D') + \nu$
            \ENDIF
        \ELSE
            \STATE Set $N'(a_{u_i}) \gets N(a_{u_i}) - k_i$ 
            \STATE Compute the new empirical mean reward  $\hat \mu'(a_{u_i}) \gets \frac{\hat\mu(a_{u_i})\cdot N(a_{u_i}) - \sum_{(a_{u_i},r_j) \in U_i} r_j}{N'(a_{u_i})}$
            \STATE Compute the new penalty $b'(a_{u_i}) \gets \sqrt{\frac{\ln \frac{ m}{\tau}}{2N'(a_{u_i})}}$ 
            \STATE Set $\wh \vf_{(a_{u_i})}(D') \gets \hat \mu'(a_{u_i}) - b'(a_{u_i})$ 
        \ENDIF
    \ENDFOR   
\end{algorithmic}
{\bfseries Output:} $\hat a' = \argmax_{a} \wh \vf_{(a)}(D')$ 
\end{algorithm}

\begin{algorithm}[H]
\caption{Multi-source unlearning algorithm under $\+M_f$ with input $\emptyset$}
{\bfseries Input:} Output of $\pi(D')$, additional statistics $T(D')$: $\pi(D) = \hat a$, $\set{a_{u_i}}_{i=1}^\ell$, $k_i$, $N(a_{u_i})$ for $i\in [\ell]$, $\wh \vf(D')$, a confidence level: $\tau \in (0,1)$  
\begin{algorithmic}[1]
    \FOR{$i=1$ to $\ell$}
        \IF{($\hat a = a_{u_i}$)}
            \STATE $\gamma_0' \gets \frac{4}{3}\sqrt{\frac{\pi \ln \frac{m}{\tau} }{N_{\min}-k_{\max}}}$
            \IF{($\gamma < \gamma_0$)}
                \STATE Set $\Delta_{\vf} \gets \frac{3k_i}{2N(a_{u_i})}$, $\sigma \gets \Delta_{\vf}\gamma$
                \STATE Sample $\nu \in \bb R$ from $\+N(0,\sigma^2)$   
                \STATE Set $\wh \vf_{(a_{u_i})}(D') \gets \wh \vf_{(a_{u_i})}(D') + \nu$
            \ENDIF
        \ENDIF
    \ENDFOR    
\end{algorithmic}
{\bfseries Output:} $\hat a' = \argmax_{a} \wh \vf_{(a)}(D')$ 
\end{algorithm}

We let $a_{u_1}$ be the arm with the largest $\wh \vf_{(a)}(D)$ among $\set{a_{u_i}}_{i=1}^\ell$ without loss of generality. Therefore if $\hat a \in \set{a_{u_i}}_{i=1}^\ell$, then $\hat a = a_{u_1}$. Moreover, let $\hat \mu'(a) = \hat \mu(a), b'(a) = b(a)$ if arm $a$ does not rollback. 

For every $a_{u_i}$, the way to achieve $(\eps,\delta)$-UL is the same as \Cref{alg:single}. Since \Cref{alg:fixed_multi} will add gaussian noise to at most one arm $a_{u_1}$ and execute rollback on other arms, by applying \Cref{thm:unlearn_fixed_single}, we could obtain that \Cref{alg:fixed_multi} and \Cref{alg:LCB} are $(\eps,\delta)$-UL. Then we give an upper bound of sub-optimality $SubOpt(\vN,N_*,\set{k_i}_{i=1}^\ell,\eps,\delta)$.

\begin{theorem}[\Cref{thm:fixed_multi_ub} restated]
Consider any offline dataset $D(\vN)$, any unlearning request $U = \bigcup_{i=1}^{\ell} U_i$ where $U_i$ is selected from the data points of $a_{u_i}$. $N(a^*) \ge N_*$, $\abs{U_i} = k_i, \forall i\in[\ell]$. When $k_i \leq  N(a_{u_i}) - \sqrt{\frac{N(a_{u_i})\ln\frac{m}{\tau}}{2}}, \forall i\in [\ell]$, $k_{\max} < N_{\min}$, by setting $\tau = \frac{1}{N}$, the arm $\hat a'$ returned by \Cref{alg:fixed_multi} satisfies when $\gamma < \gamma_0' = \frac{4}{3}\sqrt{\frac{\pi \ln \tp{Nm} }{N_{\min} - k_{\max}}}$,
\begin{align*}
    SubOpt(\vN,N_*,\set{k_i}_{i=1}^\ell,\eps,\delta)  = O\tp{\max\set{\sqrt{\frac{\ln \tp{Nm}}{N_{\min}}} +  \frac{k_{\max} \cdot \gamma}{N_{\min}}, \sqrt{\frac{\ln \tp{N m} }{N_*}} }}. 
\end{align*}
while when $\gamma \ge \gamma_0'$,
\begin{align*}
    SubOpt(\vN,N_*,\set{k_i}_{i=1}^\ell,\eps,\delta)  = O\tp{\max\set{\sqrt{\frac{\ln \tp{N m} }{N_{\min} - k_{\max}} }, \sqrt{\frac{\ln \tp{N m} }{N_*}} }}. 
\end{align*} 
\end{theorem}

\begin{proof}
When $\gamma < \gamma_0'$, we first consider the case where $a^* \in \set{a_{u_i}}_{i=1}^\ell$. Denote $\Eps_1$ as the event that $\abs{\hat \mu(a) - \mu(a)} \leq \sqrt{\frac{\ln \frac{ m}{\tau}}{2N(a)}}$ for every $a\in \+A$ on dataset $D$ and $\abs{\hat \mu'(a_{u_i}) - \mu(a_{u_i})} \leq \sqrt{\frac{\ln \frac{ m}{\tau}}{2(N(a_{u_i} - k_i)}}$ for every $a_{u_i} \in \set{a_{u_i}}_{i=1}^\ell$ on $D'$. Applying Hoeffding's inequality for a fixed arm $a\in \+A$ and $a_{u_i}\in \set{a_{u_i}}_{i=1}^\ell$, we have 
\[
    \Pr{\abs{\hat \mu(a) - \mu(a)} \leq \sqrt{\frac{\ln \frac{ m}{\tau}}{2N(a)}}} \ge 1 - \frac{2\tau}{ m}, \quad \Pr{\abs{\hat \mu'(a_{u_i}) - \mu(a_{u_i})} \leq \sqrt{\frac{\ln \frac{ m}{\tau}}{2(N(a_{u_i}) - k_i)}}} \ge 1 - \frac{2\tau}{ m}
\]
which implies that $\Pr{\Eps_1} \ge 1- 4\tau$ using the union bound. Furthermore, regardless of whether $\hat a \in \set{a_{u_i}}_{i=1}^\ell$ occurs or not, we could always let $\Eps_2$ be the event that
$\abs{\nu} \leq \sigma\sqrt{2\ln\frac{ m}{\tau}} = t$, according to the property of Gaussian $\+M_d$, we have 
\[
    \Pr{\abs{\nu} \leq t} \ge 1 - 2e^{-\frac{t^2}{2\sigma^2}} = 1- \frac{2\tau}{ m},
\]
then we obtain $\Pr{\Eps_2} \ge 1- \tau$ since $m \ge 2$. Therefore under the condition that $\Eps = \Eps_1 \cap \Eps_2$ occurs, for every $a\in \+A$, $a_{u_i}\in \set{a_{u_i}}_{i=1}^\ell$ and $\nu$, we have 
\[
    \abs{\hat\mu(a) - \mu(a)} \leq b(a), \quad \abs{\hat \mu'(a_{u_i}) - \mu(a_{u_i})} \leq \sqrt{\frac{\ln \frac{ m}{\tau}}{2(N(a_{u_i}) - k_i)}}, \quad \abs{\nu} \leq \sigma\sqrt{2\ln\frac{ m}{\tau}}.  
\]

If the event $\hat a \in \set{a_{u_i}}_{i=1}^\ell$ (i.e. $\hat a = a_{u_1}$) occurs, \Cref{alg:fixed_multi} will add Gaussian noise sampled from $\+N(0,\sigma^2)$ to $\wh \vf_{(a_{u_1})}(D)$ and execute rollback on $a_{u_i}, 2\leq i \leq \ell$. There are also two branches of whether $\hat a = a^*$ or not. If $\hat a = a^*$, in view of the definition of $\hat a'$, for $\hat a' \neq \hat a$, we have 
\[  
    \mu(\hat a') \ge \wh \vf_{(\hat a')}(D') \ge \wh \vf_{(a^*)}(D)+ \nu \ge \mu(a^*) - 2b(a^*) + \nu.
\]
Else $\hat a \neq a^*$, for $\hat a' \neq \hat a$, we have 
\[
    \mu(\hat a') \ge \wh \vf_{(\hat a')}(D') \ge \wh \vf_{(\hat a)}(D) + \nu \ge \wh \vf_{(a^*)}(D) + \nu   \ge \mu(a^*) - 2b(a^*) + \nu.
\]
otherwise $\hat a' = \hat a$ and the sub-optimality is bounded by the learning algorithm. And in both branches $\hat a' \neq \hat a$ occurs only if $\nu < 0$. Similarly we could prove that 
\begin{align*}
    \E{\mu(a^*) - \mu(\hat a') } 
    &\leq \sqrt{\frac{2\ln \frac{ m}{\tau}}{N(a^*)}} + \frac{\sigma}{\sqrt{2\pi}}.
\end{align*}

If the event $\hat a \notin \set{a_{u_i}}_{i=1}^\ell$ occurs, \Cref{alg:fixed_multi} will execute rollback on all $a_{u_i}$. Before the rollback, we know that $\wh \vf_{(\hat a)}(D) > \wh \vf_{(a^*)}(D) \ge \mu(a^*) - 2b(a^*)$ and $\wh \vf_{(\hat a)}(D)$ remains the same after the rollback. For any $\hat a' \neq \hat a$, we know that 
\[
    \mu(\hat a') \ge \wh \vf_{(\hat a')}(D') \ge \wh \vf_{(\hat a)}(D) \ge \wh \vf_{(a^*)}(D) \ge \mu(a^*) - 2b(a^*).
\] 
In conclusion, when $a^* = a_{u_j}$ for some $j\in [\ell]$, the expected sub-optimality satisfies:
\begin{align*}
    \E{\mu(a^*) - \mu(\hat a') } 
     &\leq \E{ \mu(a^*) - \mu(\hat a') \mid \Eps} +  
     \Pr{\Eps^c}\cdot \E{\mu(a^*) - \mu(\hat a') \mid \Eps^c} \\ 
     &= \E{\tp{\mu(a^*) - \mu(\hat a')}\1{\hat a \in \set{a_{u_i}}_{i=1}^\ell} \mid \Eps} +  \E{\tp{\mu(a^*) - \mu(\hat a')}\1{\hat a\notin \set{a_{u_i}}_{i=1}^\ell} \mid \Eps} + \sum_{j=1}^2 \Pr{\Eps_j^c} \\ 
     &\leq \Pr{\hat a \in \set{a_{u_i}}_{i=1}^\ell \mid \Eps}\cdot \tp{\sqrt{\frac{2 \ln \frac{ m}{\tau}}{N(a^*)}} + \frac{\sigma}{\sqrt{2\pi}} } + \Pr{\hat a\notin \set{a_{u_i}}_{i=1}^\ell \mid \Eps}\cdot \tp{ \sqrt{\frac{2\ln \frac{ m}{\tau}}{N(a^*)}} + \tau } + 5\tau \\
     &\leq \sqrt{\frac{2\ln \frac{ m}{\tau}}{N(a^*)}} + \frac{\sigma}{\sqrt{2\pi}} + 5\tau \\ 
     &\leq \sqrt{\frac{2\ln \tp{Nm}}{N_{\min}}} +  \frac{3k_{\max} \gamma}{2\sqrt{2\pi}N_{\min}} + \frac{5}{N}. 
\end{align*}
The second case is $a^* \notin \set{a_{u_i}}_{i=1}^\ell$. Regardless of whether the event $\hat a \in \set{a_{u_i}}_{i=1}^\ell$ occurs or not, $\wh \vf_{(a^*)}(D)$ will remain the same. Thus the expected sub-optimality will not be worse than the learning algorithm, therefore 
\begin{align*}
 \E{\mu(a^*) - \mu(\hat a') } 
     &\leq  \E{\mu(a^*) - \mu(\hat a) } \leq \sqrt{\frac{2\ln \tp{Nm}}{N_*}} + \frac{2}{N}.  
\end{align*}
In conclusion, when $\gamma < \gamma_0'$, the upper bound satisfies
\begin{align*}
\E{\mu(a^*) - \mu(\hat a')} 
     &=  
     \begin{cases}
        O\tp{\sqrt{\frac{\ln \tp{Nm}}{N_{\min}}} +  \frac{k_{\max} \gamma}{N_{\min}} }, & a^* \in \set{a_{u_i}}_{i=1}^\ell, \\
        O\tp{\sqrt{\frac{\ln \tp{N m} }{N_*}} }, & a^* \notin \set{a_{u_i}}_{i=1}^\ell.
    \end{cases}
\end{align*}

When $\gamma \ge \gamma_0'$, \Cref{alg:fixed_multi} executes rollback on all $a \in \set{a_{u_i}}_{i=1}^\ell$. Additionally when $a^* \notin \set{a_{u_i}}_{i=1}^\ell$, we could similarly prove that the expected sub-optimality will not be worse than the learning algorithm, thus 
\begin{align*}
 \E{\mu(a^*) - \mu(\hat a') } 
     &\leq  \E{\mu(a^*) - \mu(\hat a) } \leq \sqrt{\frac{2\ln \tp{Nm}}{N(a^*)}} + \frac{2}{N}.  
\end{align*}
When $a^*\in \set{a_{u_i}}_{i=1}^\ell$, the upper bound satisfies
\begin{align*}
 \E{\mu(a^*) - \mu(\hat a')} 
     &\leq \E{ \mu(a^*) - \mu(\hat a') \mid \Eps} +  
     \Pr{\Eps^c}\cdot \E{\mu(a^*) - \mu(\hat a') \mid \Eps^c} \\ 
     &\leq \sqrt{\frac{2\ln \tp{Nm}}{N_{\min} - k_{\max} }} + \frac{5}{N}. 
\end{align*}
In conclusion, when $\gamma \ge \gamma_0'$, the upper bound satisfies
\begin{align*}
    \E{\mu(a^*) - \mu(\hat a')} 
     &=  
     \begin{cases}
        O\tp{\sqrt{\frac{\ln \tp{N m} }{N_{\min} - k_{\max}}} }, & a^* \in \set{a_{u_i}}_{i=1}^\ell,  \\
        O\tp{\sqrt{\frac{\ln \tp{N m} }{N_*}} }, & a^* \notin \set{a_{u_i}}_{i=1}^\ell.
    \end{cases}
\end{align*}

\end{proof}

\paragraph{How to obtain $\gamma_0'$}
Similarly, the threshold $\gamma_0'$ is optimized in the case $a^*\in \set{a_{u_i}}_{i=1}^\ell$, since the upper bound is fixed to be $O\tp{\sqrt{\frac{\ln \tp{N m} }{N(a^*)}}}$ when $a^*\notin \set{a_{u_i}}_{i=1}^\ell$. Let 
\[
    \sqrt{\frac{2\ln \frac{m}{\tau} }{N_{\min} - k_{\max} }} = \sqrt{\frac{2\ln \frac{m}{\tau}}{N_{\min}}} +  \frac{3k_{\max} \gamma}{2\sqrt{2\pi}N_{\min}},
\]
and we could obtain that $\gamma = \frac{4\sqrt{\pi N_{\min} \ln \frac{m}{\tau} }}{3\sqrt{N_{\min}-k_{\max}}\tp{\sqrt{N_{\min}-k_{\max}} + \sqrt{N_{\min}}} } \approx \frac{4}{3}\sqrt{\frac{\pi \ln \frac{m}{\tau} }{N_{\min}-k_{\max}}} = \gamma_0'$.

\section{Experiments}
\label{append:exp}
The experiments are performed on a Dell Inspiration 16 PLUS 7620 laptop equipped with a 12th Generation Intel(R) Core(TM) i7-12700H CPU running at 2.30 GHz, alongside 16 GB of RAM. The operating system utilized is Windows 11, and the Python version employed is 3.13.5.  

\subsection{Additional Experimental Details of \Cref{sec:exp_fixed}}\label{append:exp_fixed}

We generate synthetic offline bandit data using a round-robin behavior policy over 5 arms with Bernoulli rewards $\vmu = (0.10, 0.08, 0.06, 0.04, 0.02)$. For each experimental configuration, we generate 200 independent trajectories and evaluate the algorithms under a prefix-sharing protocol. Unless otherwise stated, trajectories are partitioned into blocks with an effective block size of $B = 100$.

\paragraph{Deletion protocols.}
For experiments that vary the total sample size $N$ or the parameter $\gamma$, deletions always target the first 5 blocks of a designated arm $a_0$. For experiments that vary the deletion size $k$, we fix $N = 3000$ and delete all samples of arm $a_0$ within a sequence of consecutive blocks, with the number of blocks ranging from 1 to 25. Under the round-robin behavior policy, each block contains approximately $B/5 = 20$ samples per arm. As a result, deleting between 1 and 25 consecutive blocks of arm $a_0$ corresponds to deletion sizes ranging from $k = 20$ to $k = 500$, which explains the range of $k$ considered in \Cref{fig:alg2.pdf}. In both experimental settings, each prefix is subjected to 5 independent deletion requests.

We consider both the hard deletion case, where $a_0 = a^*$ is the optimal arm, and the easy deletion case, where $a_0$ is the second-best suboptimal arm. All results are averaged over 10 independent runs per configuration.

\paragraph{Reporting protocol.}
Due to multiple sources of randomness---including stochastic rewards, the behavior policy, and randomized baselines---the variance across runs is highly heterogeneous across methods and parameter regimes. Plotting error bars in a single figure would therefore significantly clutter the visualization and obscure the main trends. For clarity, we only report mean performance across runs.

\paragraph{Implementation details.}
Although we describes blocks of size $B = 100$ for ease of exposition here, in the actual implementation we generate blocks of size $B/2 = 50$. Two consecutive blocks are treated as a single logical block when reporting results, so that the effective block size remains $B = 100$. This design choice is motivated by the mixing baseline: for each logical block, we split the samples into two disjoint subsets, using one half for the rollback step and the other half for the Gaussian mechanism. Generating blocks of size $B/2$ allows this split to be implemented cleanly at the block level without introducing additional randomness. Importantly, this implementation detail does not affect the total sample size, the prefix-sharing protocol, or the effective deletion sizes reported in the main text. All reported results are therefore consistent with the description in \Cref{sec:exp_fixed}.

\subsection{Additional Experimental Details of \Cref{sec:exp_distribution}}
To study the $C^* \in [2,+\infty)$ case, we study the performance of our methods under a stochastic behavior policy that selects the optimal arm with probability $1/C^*, C^*=5$ and distributes the remaining probability uniformly over suboptimal arms, inducing controlled imbalance in the data. Prefixes of length $N\in\{1000,\dots,5000\}$ are used. For experiments varying $N$ or $\gamma$, deletions target a designated arm with $k=80$ samples per prefix. For experiments varying the deletion size $k$ (ranging from 20 to 500 at fixed $N=3000$), $k$ consecutive samples are removed. In both cases, each prefix undergoes 5 independent deletion requests. The configurations are also averaged over $10$ runs for each configuration. This setup isolates the effect of data imbalance on unlearning performance.

\subsection{Experiments of \Cref{alg:fixed_single_mixing}}
In the experiments below, we focus exclusively on the hard deletion case
where \(a^* = a_0\). This choice is motivated by the design of \Cref{alg:fixed_single_mixing}, in which the mixing parameter is set to \(k' = k\) whenever \(\hat a \neq a_0\). As a result, for easy deletion cases (\(a^* \neq a_0\)), different values of \(\eta\) do not produce meaningful differences across algorithms, and including these cases would not provide additional insights.

We adopt the same experimental setup as described below to evaluate the performance of \Cref{alg:fixed_single_mixing}, except for a minor implementation difference in the fixed-sample model. Specifically, in the fixed-sample model, each logical block of size \(B=100\) is further partitioned into 4 equal-sized sub-blocks of size $B/4=25$. This construction allows us to implement different mixing ratios \(\eta \in \{0, 0.25, 0.5, 0.75, 1\}\), where \(\eta = k'/k\) controls the fraction of samples allocated to the rollback step, and the remaining samples are processed using the Gaussian mechanism. In particular, \(\eta=0\) corresponds to the pure Gaussian mechanism, while \(\eta=1\) corresponds to the pure rollback procedure. Intermediate values of \(\eta\) are realized by assigning the corresponding number of sub-blocks to rollback. Apart from this block-level construction in the fixed-sample model, all other aspects of the experimental protocol—including data generation, prefix-sharing, deletion strategies, and evaluation metrics—remain identical to those described in the main text. In particular, the experiments under the distribution model use exactly the same configuration as in \Cref{sec:exp}. \Cref{fig:alg2_mixing.pdf} shows that optimal performance is achieved at the two extremes \(\eta=0\) and \(\eta=1\) at almost all time, which validates the design choices behind \Cref{alg:single} and \Cref{alg:dist_single_le2}. We only need to consider the trade-off between Gaussian mechanism and rollback. 

\begin{figure}[H]
    \centering
    \begin{subfigure}{\linewidth}
        \centering
        \includegraphics[width=\linewidth]{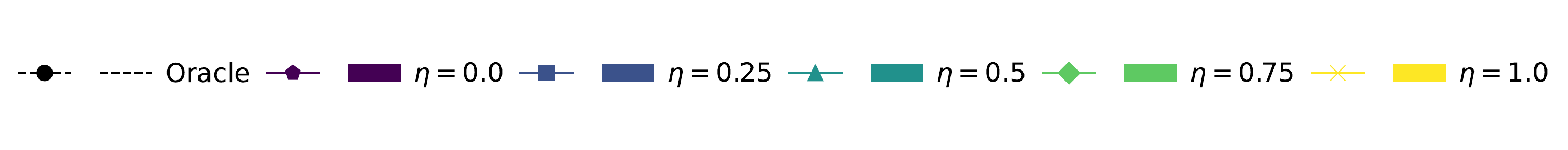}
    \end{subfigure}
    \begin{subfigure}{0.33\linewidth}
        \centering
        \includegraphics[width=\linewidth]{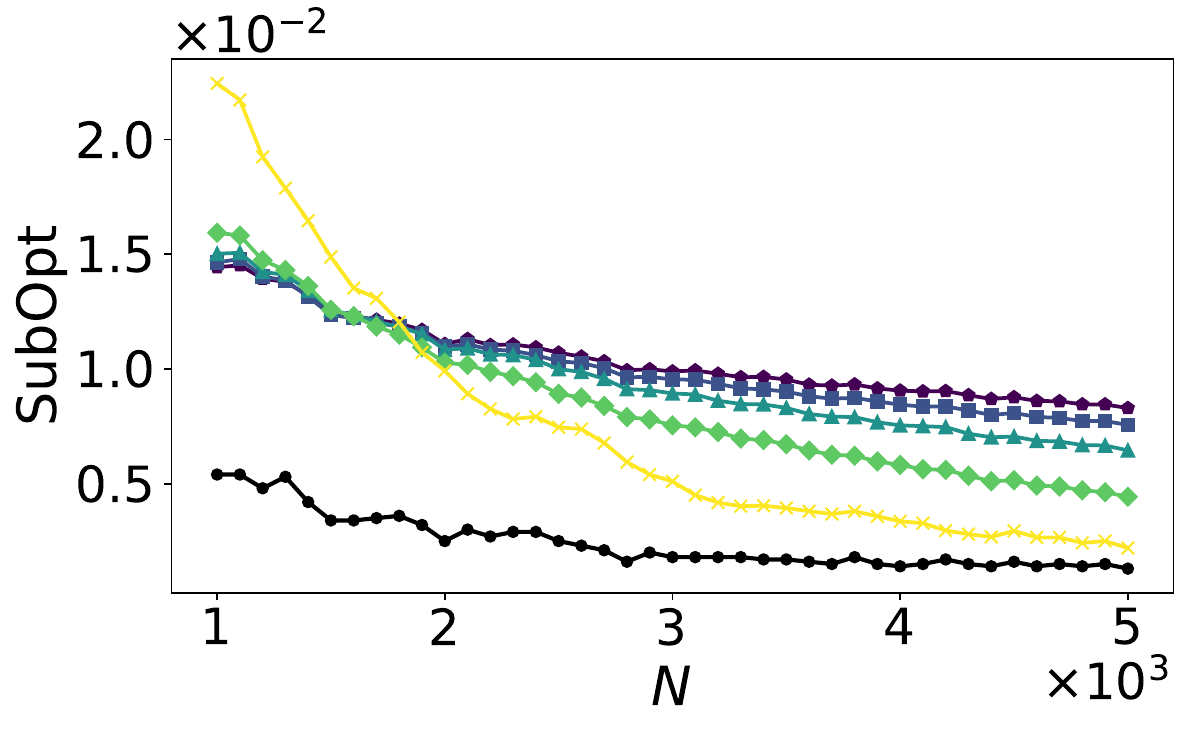}
        \caption{$\+M_f$, hard case}
    \end{subfigure}
    \begin{subfigure}{0.33\linewidth}
        \centering
        \includegraphics[width=\linewidth]{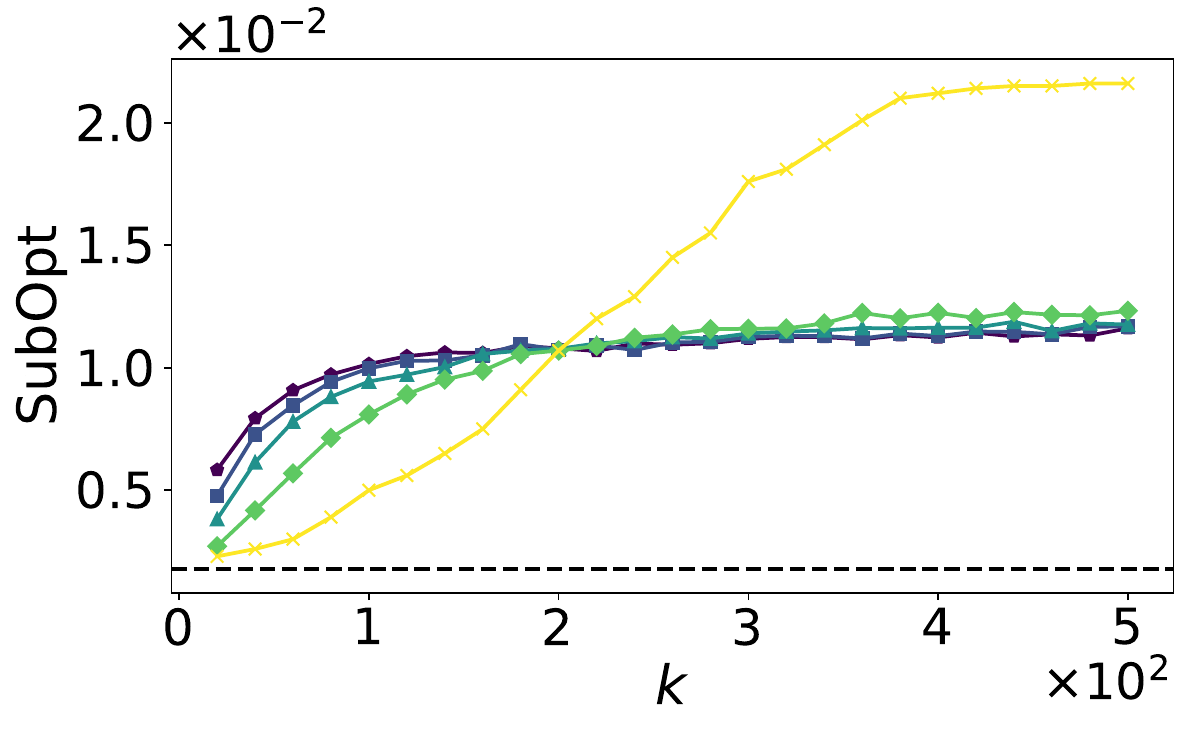}
        \caption{$\+M_f$, hard case}
    \end{subfigure}
    \hfill
    \begin{subfigure}{0.33\linewidth}
        \centering
        \includegraphics[width=\linewidth]{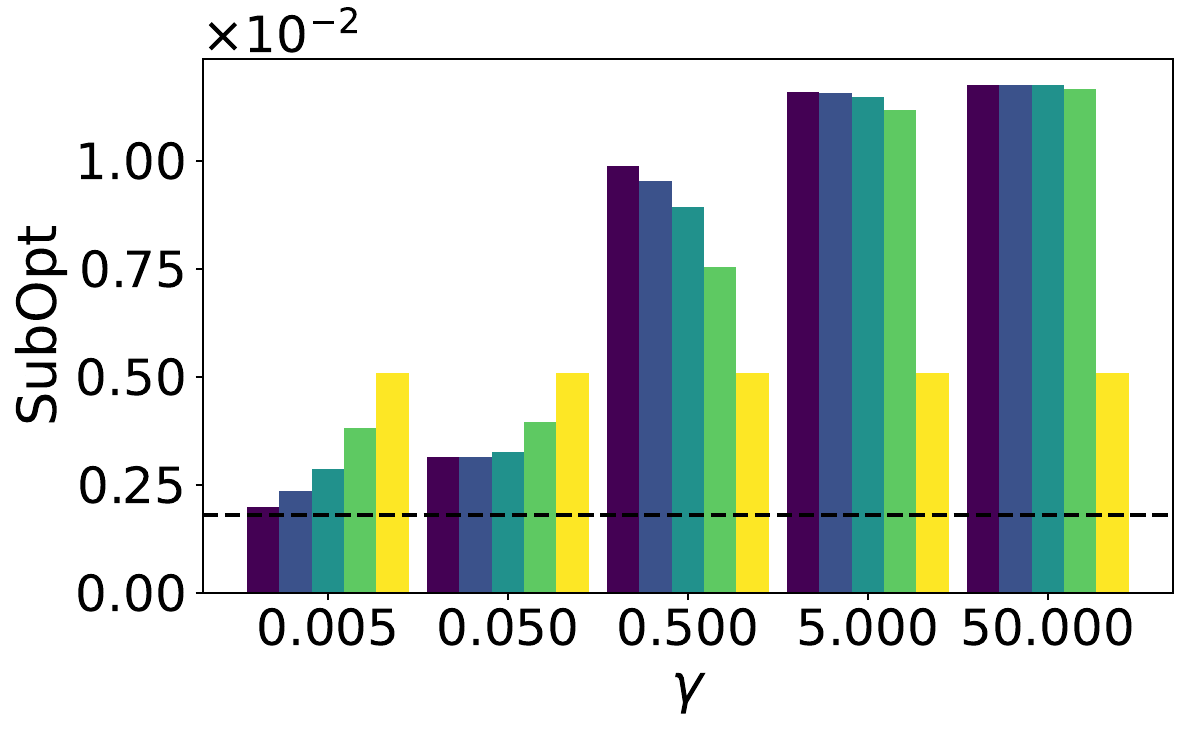}
        \caption{$\+M_f$, hard case}
    \end{subfigure}
    \begin{subfigure}{0.33\linewidth}
        \centering
        \includegraphics[width=\linewidth]{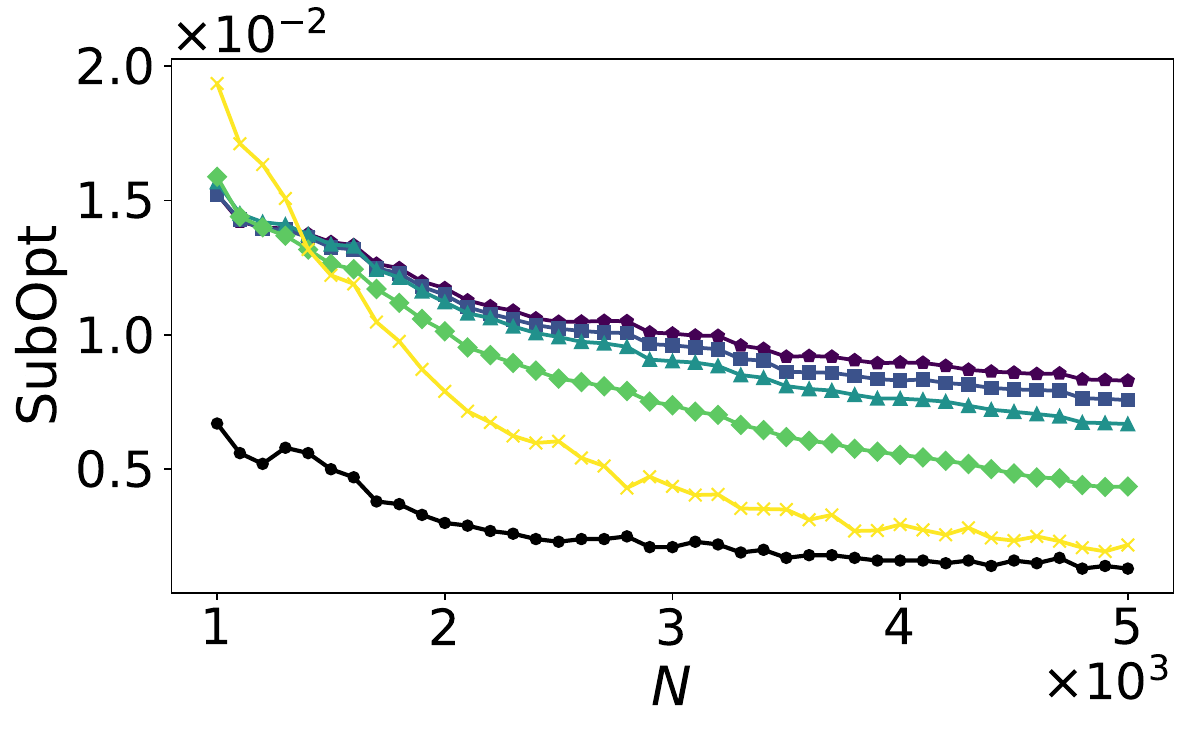}
        \caption{$\+M_d$, hard case}
    \end{subfigure}
    \begin{subfigure}{0.33\linewidth}
        \centering
        \includegraphics[width=\linewidth]{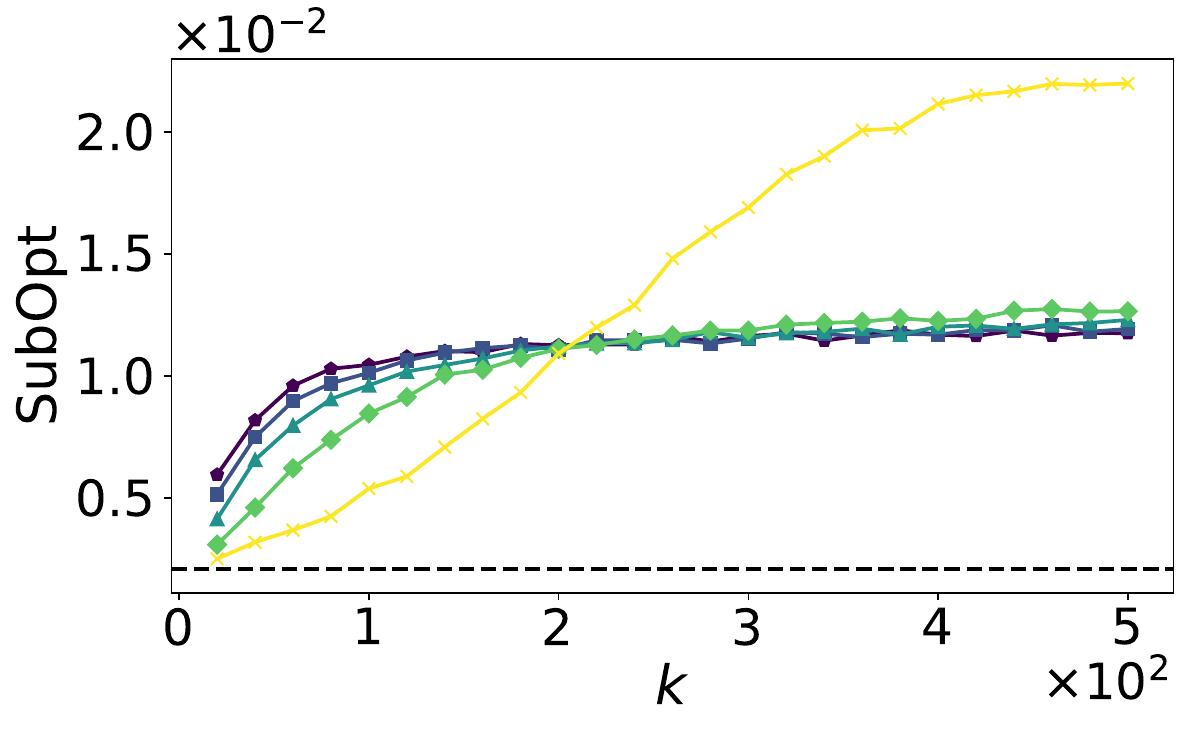}
        \caption{$\+M_d$, hard case}
    \end{subfigure}
    \hfill
    \begin{subfigure}{0.33\linewidth}
        \centering
        \includegraphics[width=\linewidth]{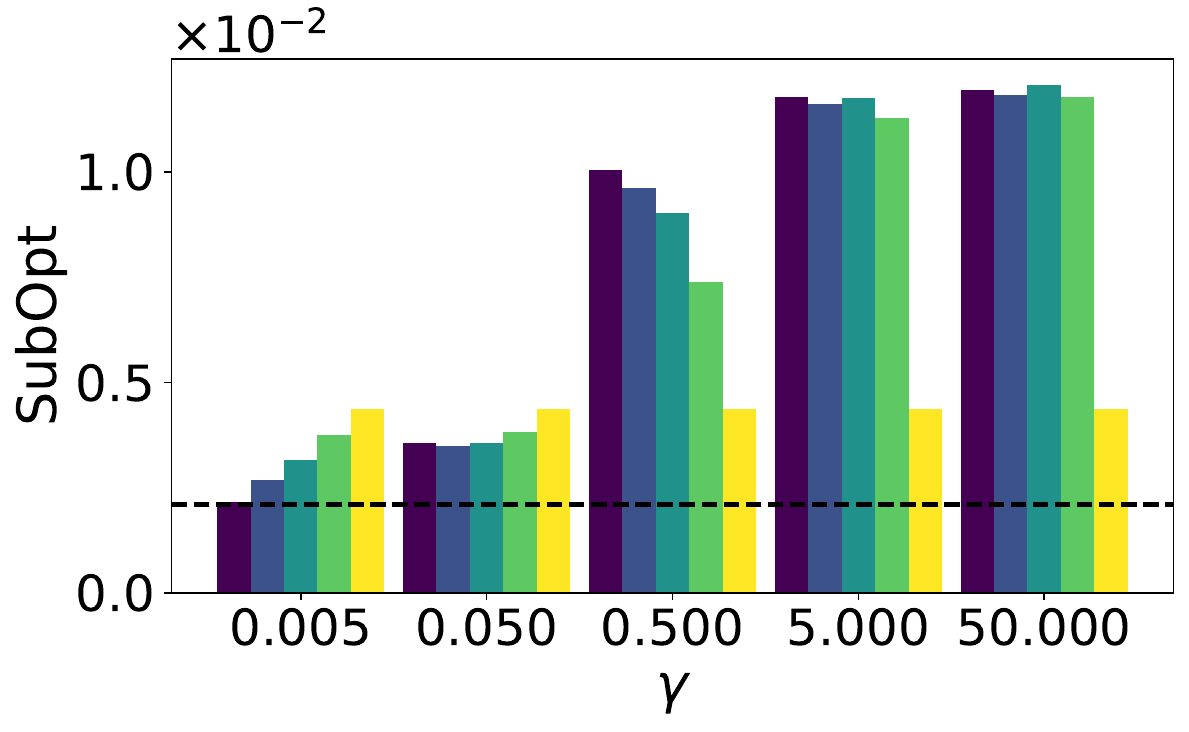}
        \caption{$\+M_d$, hard case}
    \end{subfigure}
    \caption{Comparison of \Cref{alg:fixed_single_mixing} with baselines under $\+M_f$ (uniform) and $\+M_d$ ($C^*=5$) when fixing $\gamma=0.5$ and $k=80$, fixing $\gamma=0.5$ and $N=3000$, and fixing $k=80$ and $N=3000$.}
    \label{fig:alg2_mixing.pdf}
\end{figure}




\subsection{Experiments of \Cref{alg:fixed_multi}}
The experimental setup follows the same protocol as described in \Cref{append:exp_fixed}, except that we focus exclusively on the fixed-sample setting, in accordance with the theoretical analysis of the multi-source unlearning algorithm. For the 5-arm fixed-sample model with ordering \(0 > 1 > 2 > 3 > 4\), the hard case corresponds to deleting $k$ samples each from arms 0 and 1, while the easy case corresponds to deleting $k$ samples each from arms 2 and 3. All other aspects of the experimental protocol—including block construction, prefix-sharing, and evaluation over multiple independent runs—remain identical to those described in \Cref{append:exp_fixed}.

\Cref{fig:multi.pdf} shows that \Cref{alg:fixed_multi} consistently exhibits robust and competitive performance across all configurations. In the hard case, where deletions affect the top arms, \Cref{alg:fixed_multi} remains consistently competitive and avoids the sharp degradation observed for the Gaussian mechanism under large $\gamma$ or large $k$. Quantitatively, when $N=2000$ and $\gamma=0.05$, it achieves 8\% ($k=40$) and 9\% ($k=80$) reduction in sub-optimality relative to the Gaussian mechanism, and  17\% ($k=40$) and 12\% ($k=80$) reduction relative to rollback when deletions are severe. As $N$ increases, the performance gap narrows and \Cref{alg:fixed_multi} remains perform better than Gaussian mechanism and rollback, demonstrating stable behavior even under adversarial deletion patterns. In the easy case, the algorithm achieves sub-optimality that is nearly identical to the Oracle and rollback baselines across all settings, while significantly outperforming the Gaussian mechanism. For example, when $N=2000$ and $\gamma=0.5$, \Cref{alg:fixed_multi} reduces sub-optimality by more than 80\% relative to the Gaussian mechanism, and remains either the best-performing method or within 1\% of it. This behavior persists as $N$ increases to $4000$, indicating that the algorithm scales favorably with sample size while maintaining stability. Overall, these results highlight a consistent trend: while baseline methods can be optimal in specific regimes, their performance varies substantially across configurations. In contrast, \Cref{alg:fixed_multi} maintains low sub-optimality across both hard and easy cases, offering a robust trade-off between accuracy and unlearning effectiveness in the multi-source fixed-sample setting.

\begin{figure}[H]
    \centering
    \begin{subfigure}{\linewidth}
        \centering
        \includegraphics[width=\linewidth]{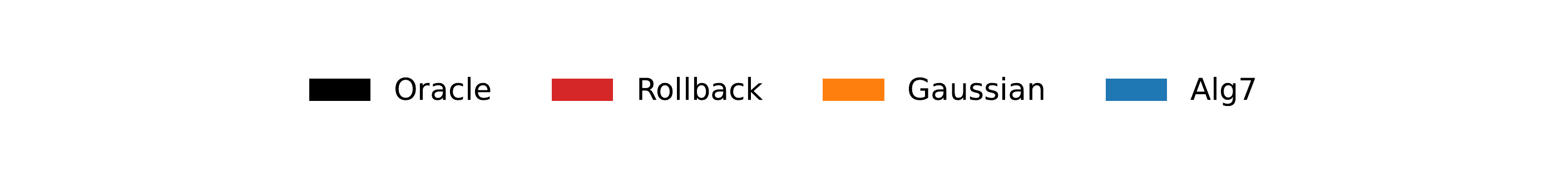}
    \end{subfigure}
    \begin{subfigure}{0.48\linewidth}
        \centering
        \includegraphics[width=\linewidth]{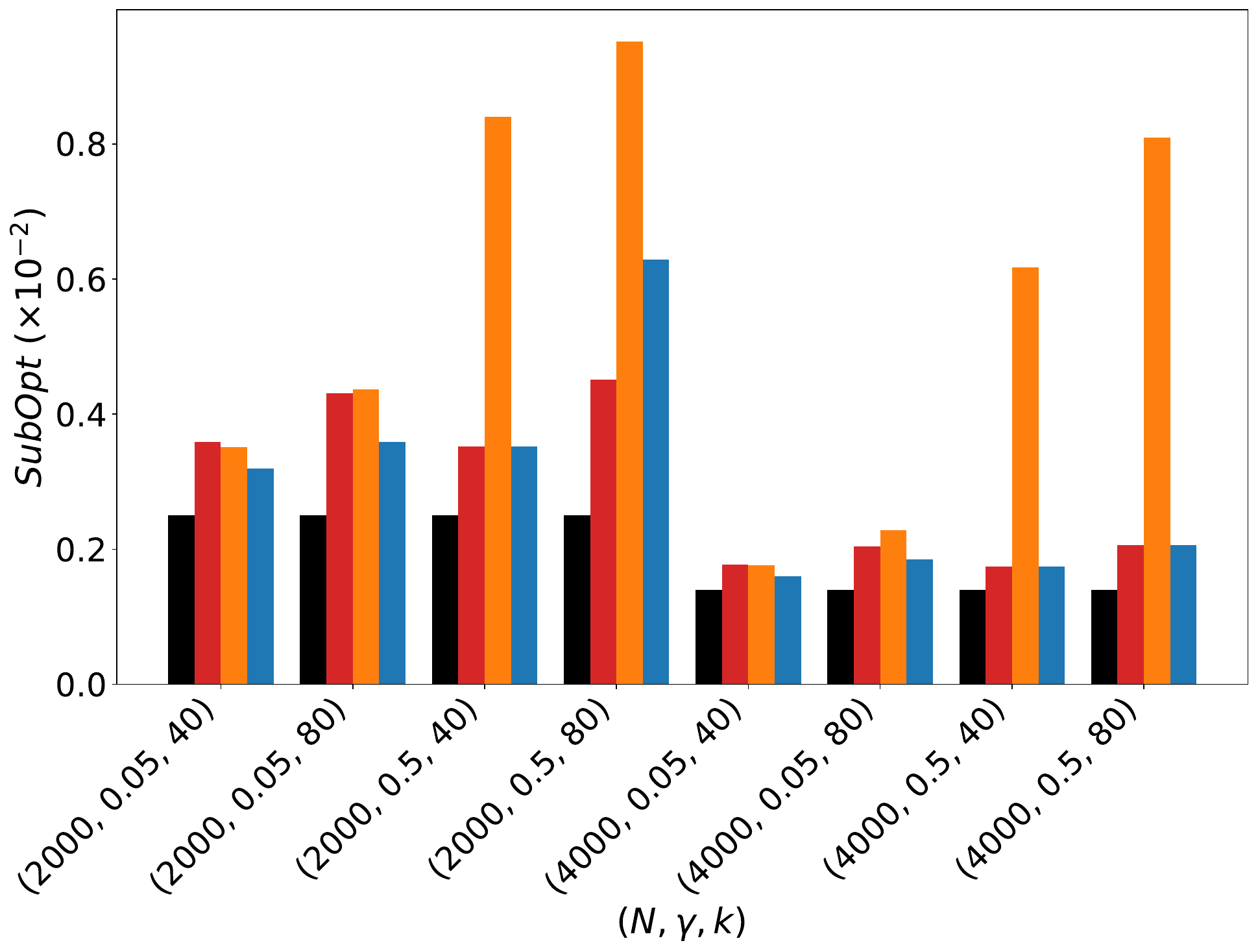}
        \caption{$\+M_f$, hard case}
    \end{subfigure}
    \begin{subfigure}{0.48\linewidth}
        \centering
        \includegraphics[width=\linewidth]{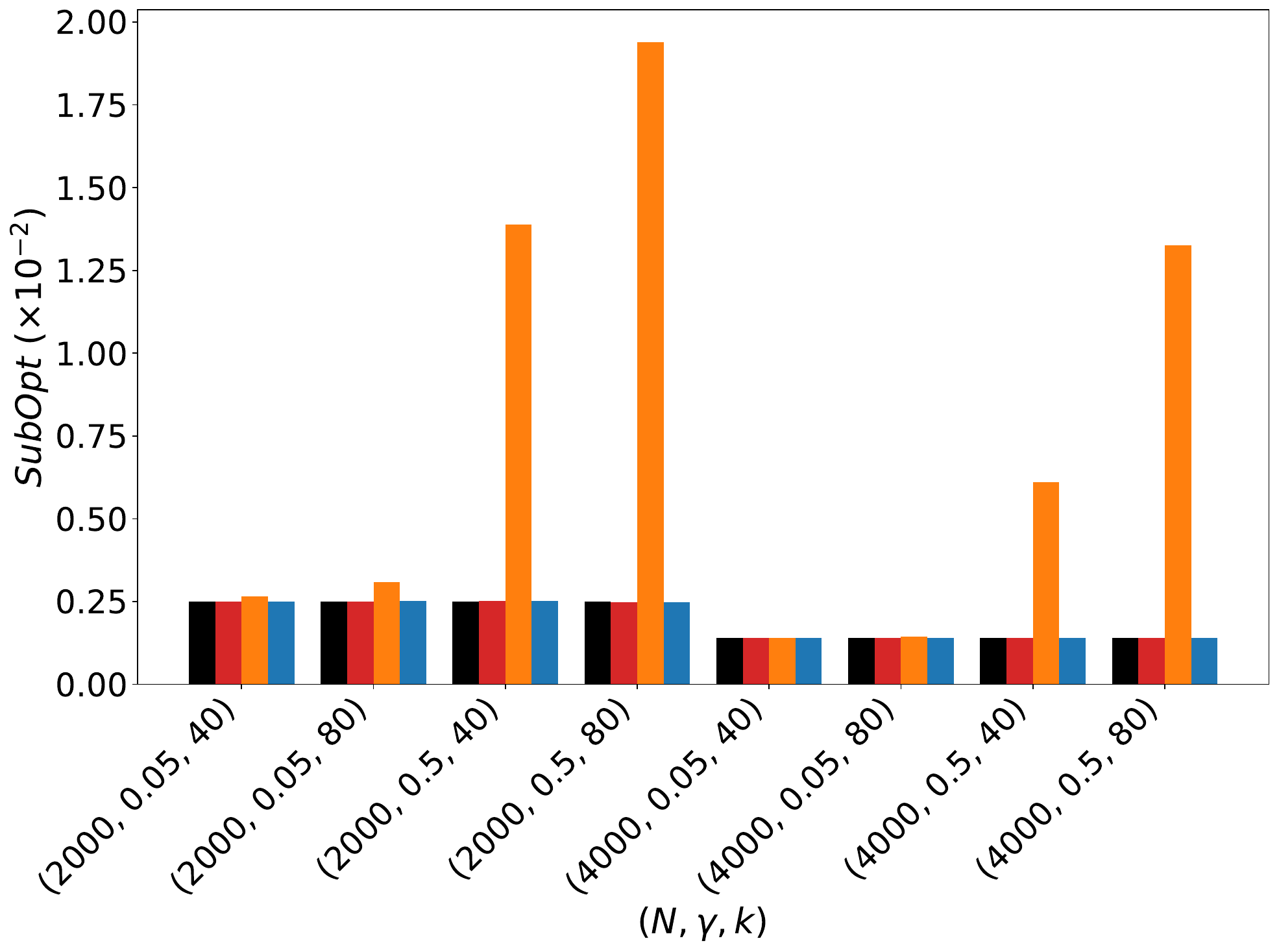}
        \caption{$\+M_f$, easy case}
    \end{subfigure}
    \caption{Comparison of \Cref{alg:fixed_multi} with baselines under $\+M_f$ (uniform).}
    \label{fig:multi.pdf}
\end{figure}

\end{document}